\documentclass[10pt,twocolumn,letterpaper,abstract=true]{scrartcl}

\usepackage{appendix}

\usepackage[utf8]{inputenc}             
\usepackage[T1]{fontenc}               	
\usepackage[english]{babel}                   	

\usepackage{amsmath}					
\usepackage{amssymb}					
\usepackage{amsthm}
\usepackage{amsbsy} 

\usepackage{lmodern}
\usepackage{fourier}

\usepackage{mathtools}
\usepackage{listings}					
\usepackage[dvipsnames]{xcolor}			

\usepackage{tabularx}
\usepackage{enumitem}					
\usepackage{tikz}
\usepackage{tikz-cd}
\usetikzlibrary{calc}

\usepackage[style = numeric-comp, sorting = none]{biblatex}

\definecolor{iccvblue}{rgb}{0.21,0.49,0.74}
\usepackage[breaklinks,colorlinks,allcolors=iccvblue]{hyperref}

\makeatletter
\@namedef{ver@everyshi.sty}{}
\makeatother

\usepackage{tikz}
\usepackage{mathtools}
\usepackage{array}
\usepackage{tikz-cd}
\newcolumntype{C}{>{$}c<{$}} 

\newcommand{\PP}{\mathbb{P}}
\newcommand{\Gr}{\mathrm{Gr}}
\newcommand{\GrLine}{\Gr(1, \PP^3)}
\newcommand{\SO}{\mathrm{SO}}

\newcommand{\rs}{RS$_1$\mbox{ }}

\newcommand{\RR}{\mathbb R}

\newcommand{\FF}{\mathbb F}
\newcommand{\PL}{\xi}

\newcommand{\R}{\mathbb{R}}

\newcommand{\arr}[2]{\begin{array}{#1} #2\end{array}}
\newcommand{\mat}[2]{\left[\!\!\arr{#1}{#2}\!\!\right]}

\newcommand{\xx}[1]{\left[#1\right]_{\times}}
\newcommand{\Macaulay}{\texttt{Macaulay2}}

\newcommand{\myqed}{}

\newcommand\blfootnote[1]{%
  \begingroup
  \renewcommand\thefootnote{}\footnote{#1}%
  \addtocounter{footnote}{-1}%
  \endgroup
}
\bibliography{local,Albl-CVPR-2020,Albl-ICCV-2019}

\setkomafont{author}{\large}

\title{Order-One Rolling Shutter Cameras}

\author{Marvin Anas Hahn\\
Trinity College Dublin\\
17 Westland Row, Dublin
Ireland\\
{\small hahnma@maths.tcd.ie}
\and
Kathlén Kohn\\
KTH Royal Institute of Technology\\
Lindstedtsvägen 25, 10044 Stockholm, Sweden\\
{\small kathlen@kth.se}
\and
Orlando Marigliano\\
University of Genoa\\
Via Dodecaneso 35, 16146 Genoa,
Italy\\
{\small orlando.marigliano@edu.unige.it}
\and
Tomas Pajdla\\
CIIRC, Czech Technical University in Prague\\
Czech Republic\\
{\small pajdla@cvut.cz}
}

\usepackage{amsthm}
\usepackage[capitalize]{cleveref}
\crefname{section}{Sec.}{Secs.}
\Crefname{section}{Section}{Sections}
\Crefname{table}{Table}{Tables}
\crefname{table}{Tab.}{Tabs.}
\newtheorem{theorem}{Theorem}
\newtheorem{lemma}[theorem]{Lemma}
\newtheorem{corollary}[theorem]{Corollary}
\newtheorem{proposition}[theorem]{Proposition}

\theoremstyle{remark}
\newtheorem{remark}[theorem]{Remark}
\theoremstyle{definition}
\newtheorem{definition}[theorem]{Definition}
\newtheorem{example}[theorem]{Example}
\newtheorem{problem}[theorem]{Problem}

\begin{document}

\twocolumn[{%
\renewcommand\twocolumn[1][]{#1}%
\maketitle
\centering
\begin{tabular}{cccc} 
        \fbox{\begin{minipage}[b][3.7cm][c]{2cm} \centering
        \includegraphics[width=2.0cm]{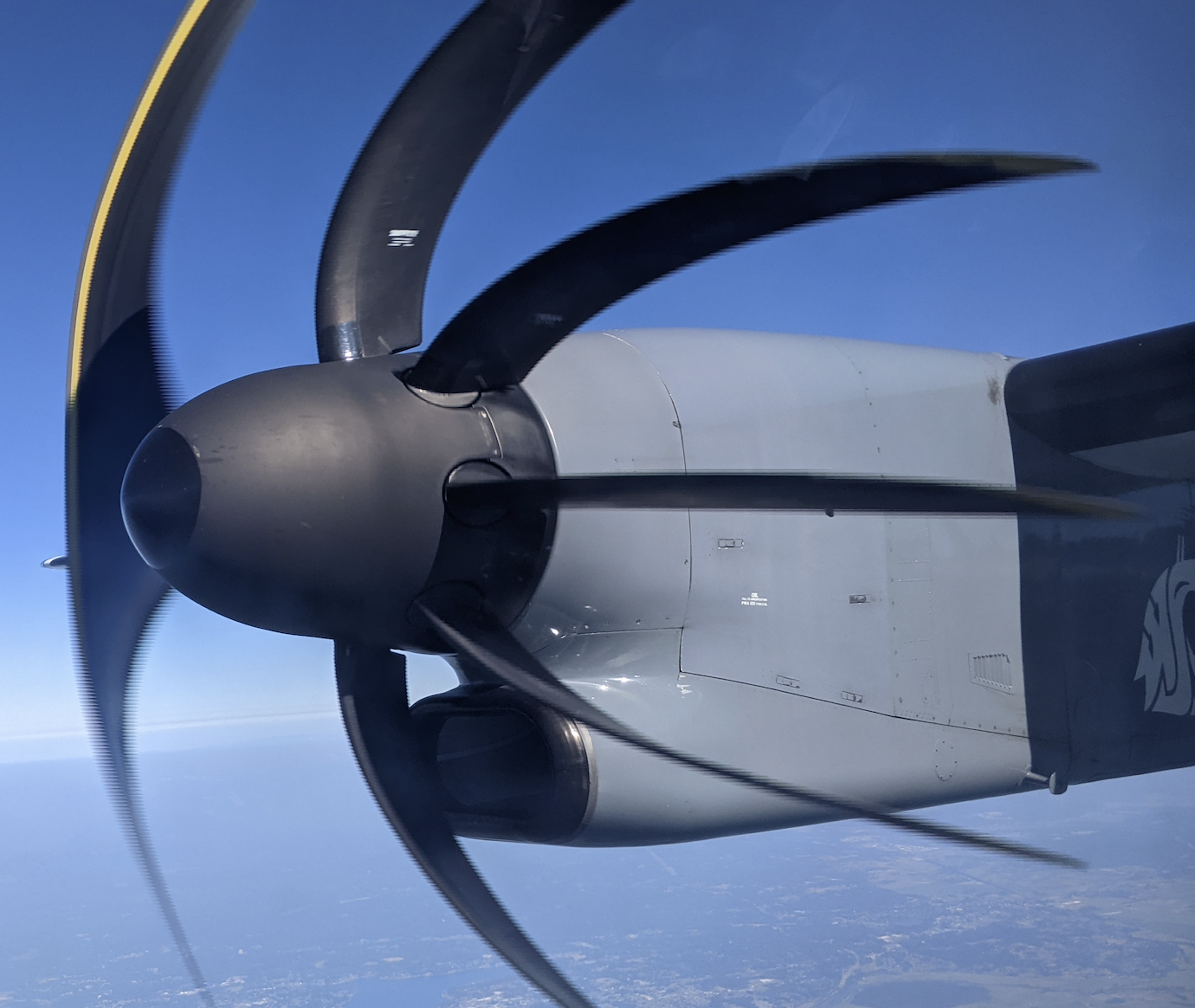}\\
        \includegraphics[width=1.9cm]{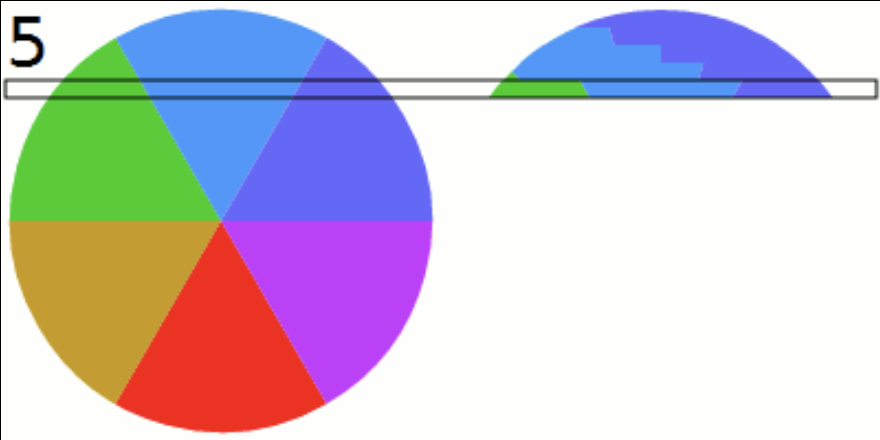}\\[1mm]
        \includegraphics[width=1.9cm]{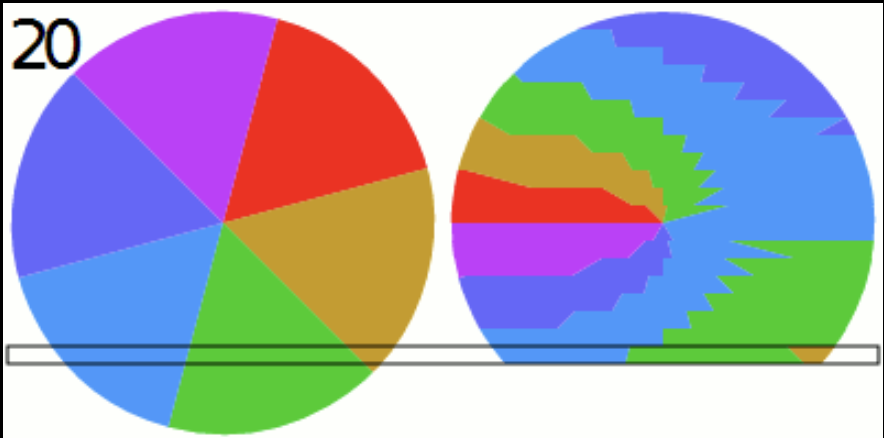}
        \end{minipage}}&
        \fbox{\includegraphics[height=3.7cm]{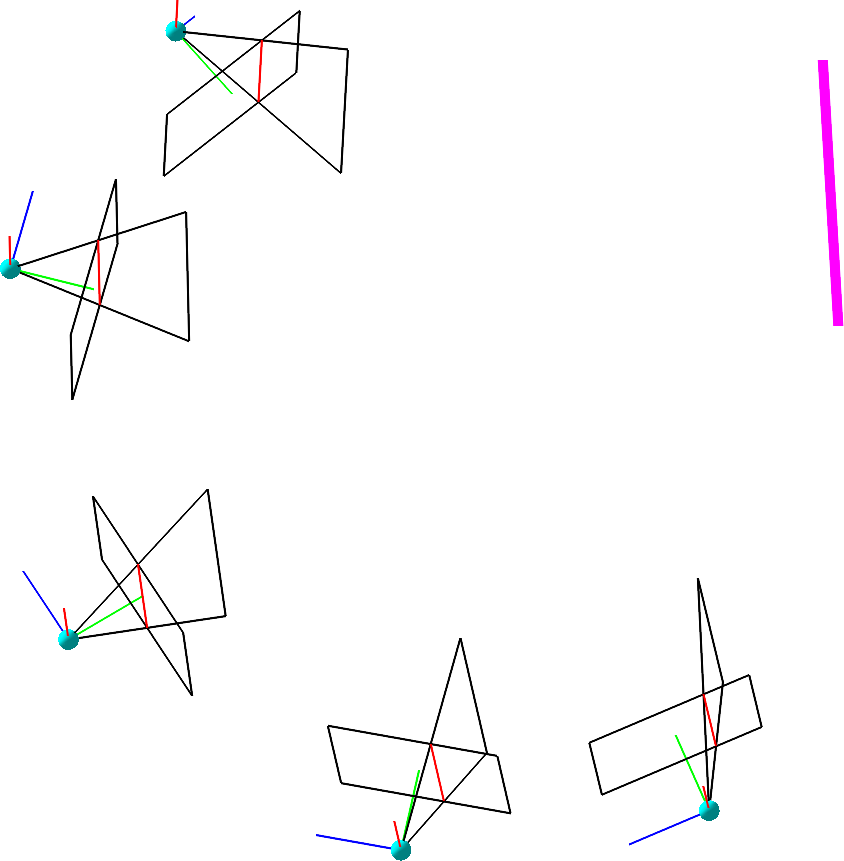}}&
        \fbox{\begin{minipage}[b][3.7cm][c]{2.2cm}
               \includegraphics[height=2.2cm]{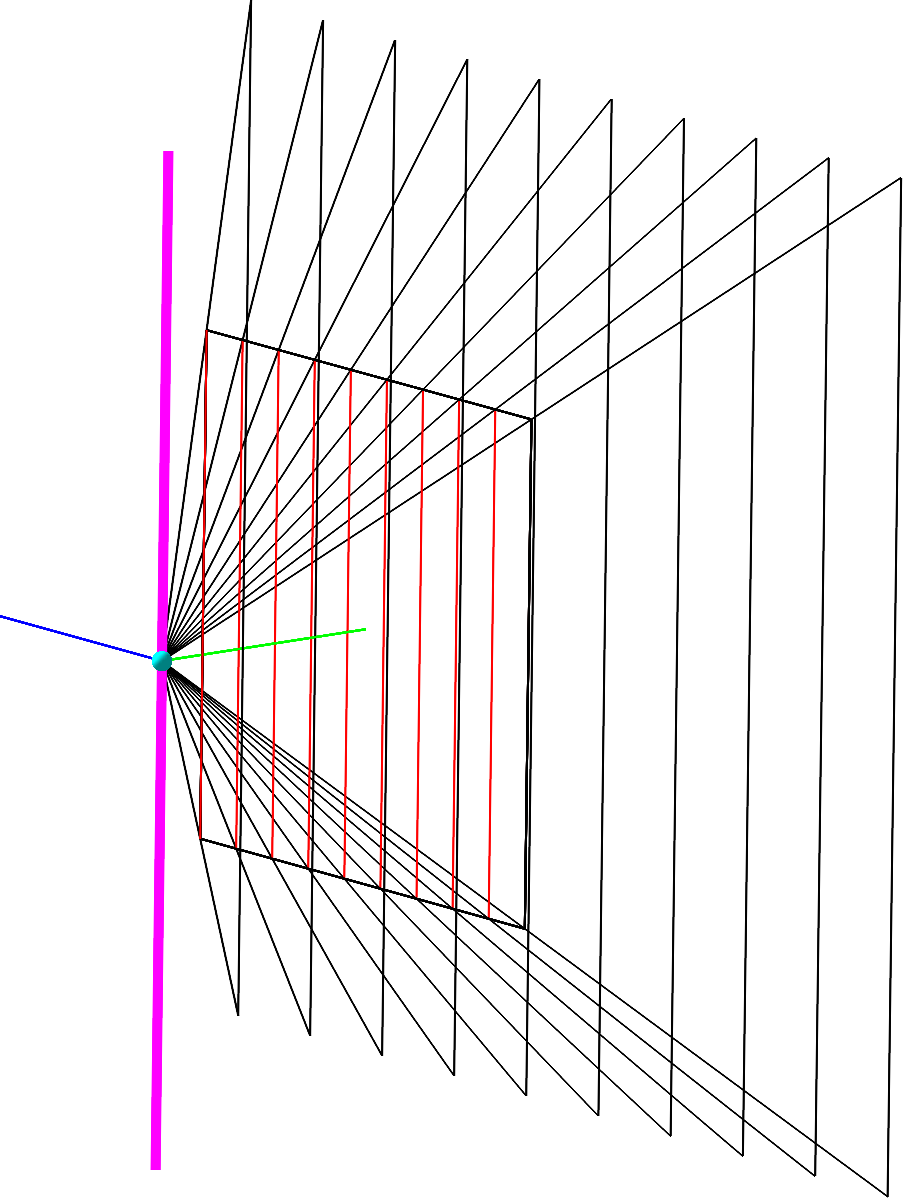}
              \end{minipage}}&
        \fbox{\includegraphics[height=3.7cm]{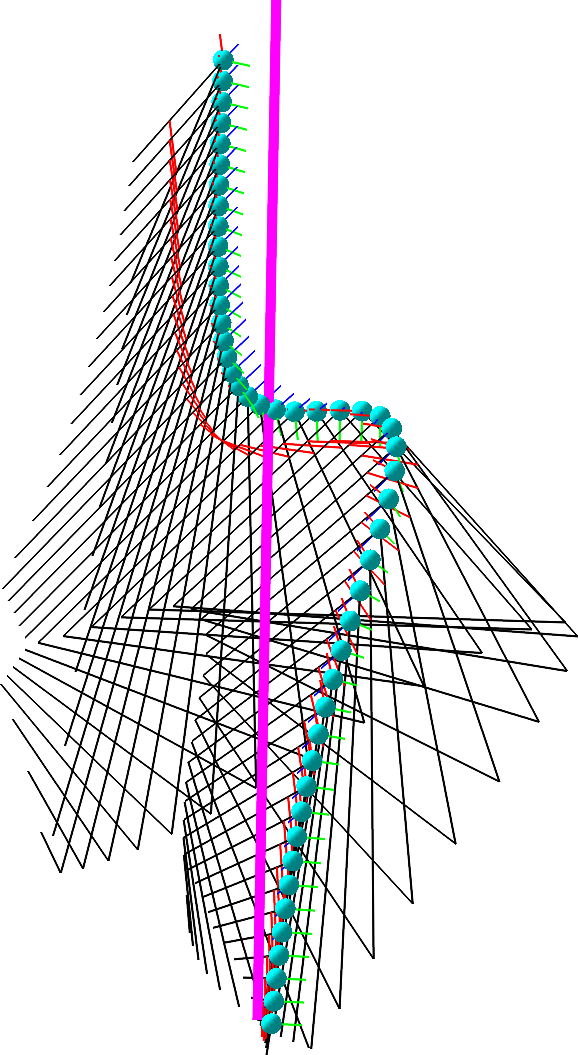}}\\
        (a) & (b) & (c) & (d)
\end{tabular}
\captionof{figure}{(a) General rolling shutter cameras see points in space multiple times ({\small en.wikipedia.org/wiki/Rolling\_shutter}). (b) Order-one rolling shutter cameras see points in space exactly once. Their rolling planes intersect in a line. Examples include (c) perspective cameras and (d) some Straight-Cayley cameras moving on a twisted cubic. \vspace{1em}}
\label{fig:teaser}
}]
\begin{abstract} \vspace{-8mm}

Rolling shutter (RS) cameras \blfootnote{
\tiny{Authors are listed alphabetically. 
KK and OM were supported by Göran Gustafssons Stiftelse and Wallenberg AI, Autonomous Systems and Software Program (WASP) funded by Knut and Alice Wallenberg Foundation.
OM was supported by the EU (Project 101061315–MIAS–HORIZON-MSCA-2021-PF-01) and
TP  by  OPJAK CZ.02.01.01/00/22\_008/0004590 Roboprox.}}
dominate consumer and smartphone markets. Several methods for computing the absolute pose of RS cameras have appeared in the last 20 years, but the relative pose problem has not been fully solved yet. We provide a unified theory for the important class of order-one rolling shutter (RS$_1$) cameras. These cameras generalize the perspective projection to RS cameras, projecting a generic space point to exactly one image point via a rational map. We introduce a new back-projection RS camera model, characterize RS$_1$ cameras, 
construct explicit parameterizations of such cameras, and determine the image of a space line. We classify all minimal problems for solving the relative camera pose problem with linear RS$_1$ cameras and discover new practical cases. 
Finally, we show how the theory can be used to explain RS models previously used for absolute pose computation.
\end{abstract}    

\section{Introduction}
Rolling shutter (RS) cameras~\cite{Meingast-OMNIVIS-2005} dominate consumer and smartphone markets thanks to affordability, enhanced resolution, and rapid frame rates. Unlike global shutter cameras (GS), RS cameras capture images sequentially line-by-line, causing image distortions if the camera moves during capture~(\cref{fig:teaser}a). The distorted RS images do not match the geometry of GS cameras. Thus, for non-negligible movements, the developed multi-view geometry for GS cameras cannot be applied. So currently, accurate multi-view geometry from moving cameras requires GS cameras. Since moving cameras are omnipresent (e.g., every modern car is equipped
with cameras and cameras are often mounted on drones), 
new theory and algorithms must be developed for RS cameras. 
Many partial results appeared in the last 20 years~\cite{ait-aider_simultaneous_2006,ait-aider_structure_2009,BatogGP10,hedborg_rolling_2012,magerand_global_2012,oth_rolling_2013,saurer_rolling_2013,saurer_minimal_2015,albl_r6p_2015,albl_rolling_2016,albl_degeneracies_2016,dai2016rolling,ito_self-calibration-based_2017,zhuang_rolling-shutter-aware_2017,albl_linear_2018,lao_robust_2018,purkait_ackermann_2018,vasu_occlusion-aware_2018,Albl-PAMI-2019,Zhuang_2019_CVPR,albl2020two,DBLP:journals/ijautcomp/FanDH23,DBLP:journals/corr/abs-2112-06170,DBLP:journals/ijcv/LaoAB21,DBLP:conf/iccv/FanD21,DBLP:conf/cvpr/ZhongZS21,DBLP:journals/ivc/FanDZW22,DBLP:conf/cvpr/BaiSB22,DBLP:conf/cvpr/LiaoQXZL23,DBLP:journals/pami/FanDL23,DBLP:conf/wacv/McGriffMAD24,DBLP:journals/corr/abs-2408-05409,DBLP:conf/eccv/NiuCZLJZ24}. Here, we provide a unified theory and
extend multi-view GS geometry to the important class of moving order-one RS cameras. 

General RS cameras project a point in space into many image points (\cref{fig:teaser}a).  Perspective cameras project a point in space by a linear rational map to exactly one image point.
Thus, it is natural to study a generalization of perspective cameras: RS cameras that project a point in space to exactly one image point via a rational map. We call these \emph{Order-one Rolling Shutter (RS$_1$)} cameras.
RS$_1$ are relatively simple but can still explain some common scenarios. For instance, they can be used  when RS images are taken on vehicles moving at constant speed on a straight line parallel to the image plane, which is common for cars, trains, planes, etc. 
Moreover, every flatbed scanner is an RS$_1$ camera, and  satellite imaging is often done with push-broom scanners~\cite{DBLP:conf/eccv/HartleyG94}, which are very close to RS$_1$ cameras.

\subsection{Contribution and main results}

We present a systematic study of RS and \rs cameras.

In~\cref{sec:RS-Models}, we introduce a new back-projection model of RS cameras that provides explicit parameterizations of RS camera-rays via a map $\Lambda$ \eqref{eq:LambdaMap} and of RS camera rolling planes via a map $\Sigma$ \eqref{eq:SigmaMap}. Our model connects the geometry of rays in space with the image projection maps of RS cameras.
The map $\Lambda$ assigns to every image point the ray in space that projects to the point.
In general, this map has no inverse for RS cameras that see space points several~times.

\rs cameras are precisely those where such an inverse picture-taking map $\Phi=\Lambda^{-1}$ exists and is rational.
We analyze these maps in \cref{sec:Order-One}.
We show (\cref{thm:rs1cameras}) that 
all rolling planes of an
\rs camera
intersect in a space line $K$. Furthermore,
the rolling planes map $\Sigma$
of such a camera is birational
and the camera center moves on a curve $\mathcal{C}$ that either equals $K$ or 
intersects $K$ in $\deg(\mathcal{C})-1$ points.  

In~\cref{sec:Building-RS-cameras}, we construct explicit parameterizations of \rs cameras. Explicit parameterizations open a way to identify camera parameters from image measurements.
We give the dimensions of several \rs parameter spaces needed to characterize minimal problems (\cref{sec:Min-Prob}). We analyze  the 
special cases of (1) constant rotation (\cref{ssec:const-rot}) and (2) pure translation with a constant speed (\cref{sec:linearCameras}), which were studied in~\cite{dai2016rolling} under the name \emph{linear RS cameras}. 

\rs cameras give rise to picture-taking maps $\Phi$. 
In~\cref{sec:Image-of-a-line}, we give the degree of $\Phi$ for all parameterizations of \rs cameras from~\cref{sec:Building-RS-cameras} (see \cref{thm:degreePhi}), and show that the image of a line in space is a rational image curve of degree $\deg(\Phi)$ that passes through a special point at infinity $\deg(\Phi)-1$ many times. 
This means that the image of a space line contains precisely one further point at infinity.
We use that point
to simplify the camera relative pose minimal problems (see 
\cref{sec:Min-Prob}). 

In~\cref{sec:Min-Prob}, we present {\em all minimal problems} for computing the relative poses of linear \rs cameras from correspondences between multiple images of points and lines (with potential incidences) under complete visibility assumptions~\cite{DBLP:conf/cvpr/HrubyDLP22}. We show that there are exactly {\em 31 minimal problems for 2, 3, 4, and 5 cameras} (\cref{fig:minmtwo}). All these minimal problems are new. We also show that no minimal problems exist for a single camera and more than 5 cameras. For every minimal problem, we compute the number of solutions (degree). There are several practical minimal problems for two cameras: (i) Three problems with small degrees (28, 48, 60) and a small number of image features (e.g., 7 or 9 points, or 3 points + 2 lines) are suitable for constructing efficient symbolic-numeric minimal solvers~\cite{kukelova_automatic_2008,DBLP:conf/cvpr/LarssonOAWKP18,Martyushev_2022_CVPR}. (ii) Two important problems with 7 and 9 points have moderate degrees (140, 364) and thus are suitable for solving by optimized homotopy continuation~\cite{DBLP:journals/pami/FabbriDFRPTWHGKLP23,DBLP:conf/cvpr/HrubyDLP22,breiding2018homotopycontinuationjl}. Similarly, there is a practical problem for three cameras with degree 160. Minimal problems for more than three cameras are much harder and impractical unless they could be decomposed into simple problems~\cite{DBLP:conf/cvpr/HrubyKDOPPL23}. This is an open problem for the future. 

\cref{sec:Straight-Cayley} shows when a practical ``Straight-Cayley'' (SC) RS camera model~\cite{Albl-PAMI-2019} produces  \rs cameras.
The SC model is important since it leads to tractable minimal RS camera absolute pose problems~\cite{Albl-PAMI-2019}.  We provide explicit general constraints on the CS model and concrete examples. This demonstrates how the theory for \rs cameras developed in this paper can be used to understand existing practical RS camera models. 

\subsection{The most relevant previous work}
\cite{dai2016rolling} formulates the relative pose problems for two RS cameras for several RS models,
but order-one cameras were not considered and camera order was not investigated. We show that general linear RS cameras have order two and that they have order one exactly when the motion line $\mathcal{C}$ is parallel to the camera projection plane. The uniform RS camera model of~\cite{dai2016rolling} uses the Rodrigues parameterization~\cite{hartley_multiple_2003} of rotation, which is not algebraic. Hence, \cite{dai2016rolling} replaces rotation matrices by their linearization to arrive at an approximate algebraic model. We use the Cayley rotation parameterization~\eqref{eq:Cayley}, which is algebraic, and we show when it produces \rs cameras. \cite{dai2016rolling}  observes that there is an 11-point minimal relative pose problem for two order-two linear RS cameras but does not present the solver. Instead, a rather impractical 20-point linear solver is suggested. We consider an \rs model and use all the algebraic constraints to get a solvable 9-point minimal relative \rs camera problem. We also classify all minimal problems for this model for arbitrarily many cameras (\cref{fig:minmtwo}) and identify several practical RS camera relative pose problems.

\cite{albl2015r6p,Albl-PAMI-2019} developed efficient absolute pose minimal problems with the Straight-Cayley model. We show that the Straight-Cayley model is not only efficient but also very general and accounts for a large class of RS cameras, ranging from perspective cameras to many order-one cameras suitable for practical applications. 

\cite{PONCE201762}  studies linear congruences to model ray arrangements of generalized cameras~\cite{DBLP:journals/ijcv/SeitzK02,DBLP:journals/ijcv/Pajdla02,DBLP:conf/cvpr/Ponce09}. It characterizes order-one congruences and together with its follow-up paper \cite{TragerSCHP16} introduces ``photographic camera'' projection maps that are rational.  They study special two-slit cameras but do not relate the congruences and maps to RS cameras. 
We extend \cite{PONCE201762,TragerSCHP16} to real problems arising with RS cameras. 
\subsection{Concepts used in the main paper}
We work with cameras that take pictures of points in projective $3$-space $\PP^3$ and produce points in the projective image plane $\PP^2$.
We often identify planes $\Sigma$ in $\PP^3$ with points $\Sigma^\vee$ in the dual space $(\PP^3)^\ast$.
The Grassmannian $\GrLine$ is the set of lines in $\PP^3$. 
For a line $L \in \GrLine$, we write $L^\vee \subset (\PP^3)^\ast$ for its dual line.
The span of  projective subspaces $X$ and $Y$ is denoted  by  $X\vee Y$. 
An \emph{algebraic variety} is a solution set of polynomial equations.
The \emph{Zariski closure} of a set is the smallest algebraic variety containing the set.
The \emph{degree} of an algebraic curve $\mathcal{C}$ in $\PP^3$ is the number of complex points in its intersection with a generic~plane. We indicate rational functions, that are possibly not defined everywhere, via dashed arrows $\dashrightarrow$. 
A \emph{birational map} is a rational function that is bijective onto its image almost everywhere (i.e., outside of a proper subvariety of the domain).
Additional concepts, definitions, lemmas, and the proofs are included in SM.
\section{Rolling shutter camera model}\label{sec:RS-Models}

An \emph{RS camera} is defined by moving a perspective camera with center $C$ and projection plane $\Pi$ in the space $\PP^3$ while scanning the projection plane $\Pi$ along a pencil $\mathcal{R}$ of (parallel) lines.  The lines in this pencil are called \emph{rolling image lines}. They capture the geometry of the rolling shutter effect. In applications, $C$, $\Pi$ and lines $r \in \mathcal{R}$ are functions of time. However, typically, $\mathcal{R}$ is in one-to-one correspondence with a time interval. Thus, we use $r$ to parameterize the camera motion and write $C(r)$ and $\Pi(r)$.

\begin{figure}[htb]
    \centering
\includegraphics[width=0.78\linewidth]{Figs/RSoverview.pdf}    \caption{Overview of RS notation.}
    \label{fig:RSoverview}
\end{figure}

\vspace*{-2mm}
Each rolling line $r \in \mathcal{R}$ generates a rolling plane $\Sigma(r)$ that is the preimage of $r$ for the perspective camera with center $C(r)$ and projection plane $\Pi(r)$. The points of $\Sigma(r) \subset \PP^3$ are projected onto the rolling line $r \subset \PP^2$ along the pencil $\mathcal{L}(r)$ of lines in $\Sigma(r)$ that pass through the center $C(r)$.
The union $\mathcal{L} = \bigcup_{r \in \mathcal{R}} \mathcal{L}(r)$ of all pencils $\mathcal{L}(r)$ forms the set of rays of the RS camera; see \cref{fig:RSoverview}.

To specify an RS camera model, we need to parameterize the pencil $\mathcal{R}$ of  rolling lines in $\Pi$ and couple it with a parameterization of the motion of the perspective camera in $\PP^3$. We set the coordinate system in the  plane~$\PP^2$ so that the rolling lines are parallel to the $y$-axis.  The rolling-lines pencil $\mathcal{R}$ is then parameterized by the bijective morphism 
\begin{align} \label{eq:rho} 
    \rho: \PP^1 &\to \mathcal{R}, \\[-1mm]
    (v:t) &\mapsto (0:1:0) \vee (v:0:t) \equiv (-t:0:v) \in (\PP^2)^*. \nonumber
\end{align}

Considering the calibrated scenario~\cite{HZ-2003}, the camera position and  orientation are defined on the affine chart $\mathbb{R}^1$ where $t \neq 0$. 
They are described by $C(\tfrac{v}{t}) \in \R^3$ and $R(\tfrac{v}{t}) \in \SO(3)$.
This gives the corresponding projection matrix $P(\tfrac{v}{t}) := R(\tfrac{v}{t})\,[ I_3 \mid -C(\tfrac{v}{t})] \in \mathbb{R}^{3 \times 4}$ that represents a linear map $\PP^3 \dashrightarrow \PP^2$. 
Now, we can take the preimages of the rolling lines in $\mathcal{R}$ and obtain rolling planes in $\PP^3$:
\begin{align} \label{eq:SigmaMap}
        \Sigma^\vee \!: \mathbb{R}^1 \to (\PP^3)^*, 
    \tfrac{v}{t} \mapsto \Sigma(\tfrac{v}{t})^\vee = (1:0:-\tfrac{v}{t}) \!\cdot\! P(\tfrac{v}{t}). 
\end{align} 
Now the set $\mathcal{L}$ of camera rays is determined: it is the union of the pencils $\mathcal{L}(\tfrac{v}{t}) := \{ L \in \GrLine \mid C(\tfrac{v}{t}) \in L \subset \Sigma(\tfrac{v}{t}) \}$.
A natural global parametrization of this union of pencils is given by taking preimages of image points on $r = \rho(\tfrac{v}{t}:1)$ under the projection matrix $P(\tfrac{v}{t})$.
To  parameterize the points on $r$, we intersect $r$ with another conveniently chosen set of lines, reflecting pixels on images. For that, we make the standard choice, using lines parallel to the $x$-axis:
\begin{align}
    \nu: \PP^1 &\to (\PP^2)^*, \\[-1mm]
    (u:s) &\mapsto (1:0:0) \vee (0:u:s) \equiv (0:-s:u) \in (\PP^2)^*. \nonumber
\end{align}
An image point on the rolling line $r = \rho(v:t)$ is obtained by intersecting $r$ with the line $\nu(u:s)$, which is captured by the birational map 
\begin{align}\label{small-phi}
    \varphi:  \PP^1 \!\times\! \PP^1 \dashrightarrow \PP^2, 
    ((v:t), (u:s)) \mapsto (s v: u t : s t). 
\end{align}
This map is not defined at the point $((1:0),(1:0))$, where both lines $\rho(1:0)$ and $\nu(1:0)$  equal  the line at infinity. 
We obtain all camera rays in $\mathcal{L}$ by taking the preimages in $\PP^3$ of the image points with affine coordinates $(\tfrac{u}{s},\tfrac{v}{t})$ under the projection matrix $P(\tfrac{v}{t})$ (see~\cref{ssec:Lambda} for a derivation):
\begin{align} \label{eq:LambdaMap}
    \Lambda: \mathbb{R}^1 \times \mathbb{R}^1 &\to \GrLine, \\[-2mm]
    (\tfrac{v}{t},\tfrac{u}{s}) &\mapsto  \mat{cc}{I_3\\-C(\tfrac{v}{t})^\top} 
    \xx{R(\tfrac{v}{t})^\top\mat{r}{\tfrac{v}{t} \\ \tfrac{u}{s} \\ 1}} \!\!\!\!
    [I_3 \mid -C(\tfrac{v}{t})].
    \nonumber
\end{align}
Here, given a vector $V \in \mathbb{R}^3$, we write $[V]_\times$ for the skew-symmetric $3 \times 3$ matrix that represents the linear map that takes the cross product with $V$.
Then, $\Lambda(\tfrac{v}{t},\tfrac{u}{s})$ is a skew-symmetric $4 \times 4$ matrix whose entries are the dual Plücker coordinates of the camera ray that the camera maps to the image point $(\tfrac{v}{t},\tfrac{u}{s})$, i.e., 
the $(i,j)$-th entry of $\Lambda(\tfrac{v}{t},\tfrac{u}{s})$ is the determinant of the submatrix with columns $(i,j)$ of the $2 \times 4$ matrix whose two rows are $(0:1:-\tfrac{u}{s}) \cdot P(\tfrac{v}{t})$ and $\Sigma(\tfrac{v}{t})^\vee$.

The set $\mathcal{L}$ of all camera rays captures most of the essential geometry of the camera, while its parametrization $\Lambda$ describes the concrete imaging process onto an actual picture plane \cite{ponce2017congruences}. 
Note that we used the term `camera rays' for lines through the camera center, and not for actual rays in the sense of half-lines.
This a good model for RS cameras with a view of less than 180 degrees. For more general modeling, we'd have to consider half-lines with orientation.

\begin{remark}
    \label{rem:groupAction}
$G := \{ [\begin{smallmatrix}
        R & t \\ 0 & \alpha 
    \end{smallmatrix} ]  \mid  
    R \in \mathrm{SO}(3), t \in \R^3, \alpha \in \R \setminus \{ 0 \} 
    \}$ is the 
scaled special Euclidean group on $\RR^3$.
It acts on the space of RS cameras. 
For a camera given by the projection-matrix map $P: \mathbb{R}^1 \to \R^{3 \times 4}, \tfrac{v}{t} \mapsto P(\tfrac{v}{t})$ and a group element $g \in G$, the action is defined via
\begin{align}
    g . P := \left(
    \mathbb{R}^1 \to \R^{3 \times 4}, \; \tfrac{v}{t} \mapsto P(\tfrac{v}{t}) \cdot g
    \right).
\end{align}
When acting simultaneously on cameras and the world points $X \in \PP^3$ via $g.X := g^{-1} \cdot X$, the imaging process stays invariant.
Indeed, writing $g . \Lambda$ for the map \eqref{eq:LambdaMap} associated with the transformed projection $g.P = g.(R[I_3|-C])$, 
we have that the line $(g.\Lambda)(\tfrac{v}{t},\tfrac{u}{s})$ is the image of the line $\Lambda(\tfrac{v}{t},\tfrac{u}{s})$ under the $G$-action on $\PP^3$.
This means that 3D reconstruction using our camera model is only possible up to a proper rigid motion and a non-zero scale, when we do not fix any RS camera's scale (e.g., by knowing  \mbox{ray distances).}
\end{remark}


\section{Order-one cameras}\label{sec:Order-One}
In this article, we analyze RS cameras that see a generic (i.e., sufficiently random) point in $3$-space exactly once, i.e.,
 a generic space point appears exactly once on the image plane $\mathbb{P}^2$.
For that to happen, two conditions need to be satisfied:
1) For a generic  point in $\PP^3$, there has to be a unique camera ray in $\mathcal{L}$ passing through it.
2) The camera ray has to correspond to a unique point on the image plane via the map $\Lambda$.
We  classify all RS cameras satisfying those conditions, where we additionally impose the   imaging process to be algebraic, i.e., we require $\Lambda$ to be a rational map.

The rationality of the map $\Lambda$ implies that both the rolling planes map $\Sigma$ and the center-movement map $C: \mathbb{R}^1 \to \mathbb{R}^3$ are rational (\cref{lem:LambdaRationalImplications}).
In that case, the \emph{center locus} $\mathcal{C}$, i.e., the Zariski closure in $\PP^3$ of the image of $C$, is either a point or a rational curve.
Moreover,  the Zariski closure $\overline{\mathcal{L}}$ of the set $\mathcal{L}$ of camera rays inside the Grassmannian $\GrLine$ is a surface (or a one-dimensional pencil in the degenerate case when the center locus $\mathcal{C}$ is a single point and all rolling planes are equal); see~\cref{union-dimension-two}.
Such a surface in $\GrLine$ is classically called a \emph{line congruence} \cite{jessop2001treatise}.
An important invariance of such a congruence is its order. The \emph{order} is the number of lines on the congruence that pass through a generic space point.
Hence, we are interested in RS cameras whose associated congruence $\overline{\mathcal{L}}$ has order one.

\begin{definition}
    We say that a RS camera has \emph{order one} if its associated congruence $\overline{\mathcal{L}}$ has order one and its parametrization $\Lambda$ is birational; in other words, if it projects a generic point in space to exactly one image point via a rational map. We shortly write \emph{\rs camera} for a RS camera of order one.
\end{definition}

For \rs cameras, the image projection is a rational function $\Phi: \PP^3 \dashrightarrow \PP^2$, which we can explicitly describe as follows:
The congruence $\overline{\mathcal{L}}$ has order one if and only if 
there is a map $\Gamma: \PP^3 \dashrightarrow \overline{\mathcal{L}}$ from (a Zariski dense subset of) $\PP^3$ to the camera rays that assigns to a visible point $X$ the unique ray that sees it.
The map $\Lambda$ is birational if and only if it is rational and (almost everywhere) invertible, 
which means that the camera ray $\Gamma(X)$ corresponds to the unique image point $\Lambda^{-1}(\Gamma(X))$.
The inverse map $\Lambda^{-1}$ is rational as well. Thus,
using the embedding $\mathbb{R}^1 \times \mathbb{R}^1 \subset \PP^1 \times \PP^1$, 
taking a picture of a generic space point $X \in \PP^3$ 
is the rational map 
\begin{align}
    \label{eq:PhiMap}
    \begin{split}
        \Phi: \PP^3 \dashrightarrow \PP^2, \quad 
    X  \longmapsto \varphi(\Lambda^{-1}(\Gamma(X))).
    \end{split}
\end{align}
\begin{example}
    A familiar \rs camera is the static pinhole camera. Its congruence $\overline{\mathcal{L}}$ consists of all lines passing through the fixed camera center $C$. 
The map $\Gamma$ assigns to each space point $X \neq C$  the camera ray spanned by $C$ and $X$. 
That ray intersects the static projection plane $\Pi \cong \PP^2$ in the unique point $\Phi(X)$. 
\end{example}

\begin{theorem}\label{thm:rs1cameras}
    Consider a RS camera whose congruence-parametrization map $\Lambda$ is rational. 
    The camera has order one if and only if
    the intersection of all rolling planes $\Sigma(\tfrac{v}{t})$ is a line $K$, the rolling planes map $\Sigma$ is birational, and its center locus $\mathcal{C}$ is one of the following:
    \\  I.~ $\mathcal{C}$ is a rational curve \& meets $K$ in $\deg \mathcal{C}-1$ many points,
    \\ II.~  or $\mathcal{C}=K$,  
    \\III. or $\mathcal{C}$ is a point on $K$. 
\end{theorem}

\begin{remark}    \label{rem:typeISigmaDeterminesC}
    In type I, the points in the intersection $\mathcal{C} \cap K$ are counted with multiplicity.
    Also,  the maps $\Sigma$ and $C$ determine each other:
    Whenever $C(\tfrac{v}{t}) \notin K$, we have $\Sigma(\tfrac{v}{t}) = K \vee C(\tfrac{v}{t})$.
    Conversely, every rolling plane $\Sigma(\tfrac{v}{t})$ meets $\mathcal{C}$ in $\deg \mathcal{C}$ many points (counted with multiplicity), out of which all but one lie on the line $K$. The remaining point is $C(\tfrac{v}{t})$. 
    In particular, since $\Sigma$ is birational, so is $C$.
\end{remark}

\section{Building 
\texorpdfstring{\rs}{RS1}cameras}
\label{sec:Building-RS-cameras}

\Cref{thm:rs1cameras} is constructive, meaning that we can use it to build -- in theory --  all \rs cameras.
In the following, we describe the spaces of all \rs cameras of types I, II, and III (see Sec. \ref{sec:paramSpaceOverview}).
We start with type I.
For that, we consider
\begin{align}
\begin{split}
    \mathcal{H}_d := \{ (\mathcal{C}, K) \mid \mathcal{C} \!\subset\! \PP^3 \text{ rational curve of degree }d, \\[-1mm]
    \, K \!\subset\! \PP^3 \text{ line},
    \, \#(K \cap \mathcal{C}) = d-1 \},\end{split}
\end{align}
where $\#$ counts with multiplicities.
We can explicitly pick elements in this parameter space as follows: Choose a line $K$. Rotate and translate until $K$ becomes the $z$-axis $K' := (0\!:\!0\!:\!1\!:\!0)\vee(0\!:\!0\!:\!0\!:\!1)$. Every curve $\mathcal{C}'$ with $(\mathcal{C}',K') \in \mathcal{H}_d$ is parametrized by
$\PP^1 \dashrightarrow \PP^3, (v:t) \mapsto (vf(v:t) : tf(v:t) : g(v:t)  : h(v:t))$, where $f,g,h$ are homogeneous polynomials of degree $d-1,d,d$ \cite[Eqn.~(18)]{ponce2017congruences}.
Reverse the translation and rotation to obtain~$(\mathcal{C},K) \in \mathcal{H}_d$.

Picking elements in $\mathcal{H}_d$ allows us to parametrize all \rs cameras of type~I.
For that, let $H^{\infty}$ be the plane at infinity $(0\!\!:\!\!0\!\!:\!\!0\!\!:\!\!1)^\vee\subseteq \PP^3$. We denote the intersection of a variety $\mathcal V\subseteq \PP^3$ with $H^\infty$ by $\mathcal V^\infty$.
For a map $\Sigma^\vee \!\!:\! \PP^1\!\!\dashrightarrow\!\! (\PP^3)^*$, we define $\Sigma^\vee_{\infty}(v:t)\in (H^{\infty})^*$ as the projection of $\Sigma^\vee(v:t)$ to $(H^\infty)^*$. In the primal $\PP^3$, this is the line in the intersection of the planes $\Sigma(v:t)$ and $H^\infty$.
For vectors $A,B$, we write $A \cdot B$ for the bilinear form $\sum_i A_iB_i$.
\begin{proposition} \label{prop:typeIparameters}
    The \rs cameras of type I are in 4-to-1 correspondence with the parameter space
    \vspace{-1ex}
\small
\begin{align} \begin{split}
    &\mathcal{P}_{I,d,\delta} := \{ (K,\mathcal{C},\Sigma^\vee_\infty,\lambda) \mid   (\mathcal{C},K) \in \mathcal{H}_d, 
      \\[-1mm]& \quad \mathcal{C}^\infty \neq \mathcal{C}, \, K^\infty \neq K,
      \, \lambda: \mathbb{P}^1 \dashrightarrow K, \, \deg (\lambda) = \delta,
       \\[-1mm]
     & \quad \Sigma^\vee_\infty: \mathbb{P}^1  \dashrightarrow  (K^\infty)^\vee, (v:t) \mapsto Av+Bt \text{ \ for some } \\[-1mm]& \quad A,B \text{ with }  
      A \cdot B =0 \text{ and } A \cdot A = B \cdot B \}. \end{split}
\end{align}
\normalsize 

\vspace{-1ex}\noindent
The dimension of this  space is $\dim \mathcal{P}_{I,d,\delta} =  
3d + 2 \delta + 7.$
\end{proposition}
Each element $(K,\mathcal{C}, \Sigma^\vee_\infty,\lambda)$ in $\mathcal{P}_{I,d,\delta}$ corresponds to four \rs cameras as follows (see \cref{fig:paramRS1}): 
The rolling planes map $\Sigma$ of the cameras can be read off from the the map $\Sigma^\vee_\infty$ since each rolling plane $\Sigma(v:t)$ is the span of the line $\Sigma_\infty(v:t)$ with $K$.
By~\cref{rem:typeISigmaDeterminesC}, the map $\Sigma$ determines uniquely the parametrization $C$ of the curve $\mathcal{C}$, i.e., the movement of the camera center. 
The camera rotation map $R: \mathbb{R} \to \mathrm{SO}(3)$ is fixed as follows: 
For $x \in \mathbb{R}$, $R(x)$ has three degrees of freedom.
The first two are accounted for by the rolling plane $\Sigma(x)$. 
Here, we have to choose an orientation / sign of the normal vector of the plane $\Sigma(x)$, since this is not encoded in the projective map $\Sigma$.
Finally, the map $\lambda: \PP^1 \dashrightarrow K$ chooses the unique point $\lambda(x)$ on $K$ that the projection matrix $P(x)$ maps to $(0:1:0)$, the intersection point of all rolling lines.
Thus, $\lambda(x)$ fixes the third degree of freedom of $R(x)$, but its sign gives us again two choices. 
In summary, the two choices of orientation, that on $\Sigma$ and that on $\lambda$, give us 4 rotation maps $R: \mathbb{R} \to \mathrm{SO}(3)$.

\begin{remark}\label{rem:cameraPerCurve}
    For every line, conic, or non-planar rational curve $\mathcal{C}$ of degree at most five, there is a \rs camera moving along $\mathcal{C}$.
    For a generic rational curve of degree at least six, there is no such camera.
\end{remark}

\begin{proposition} \label{prop:typeIIparameters}
    The \rs cameras of type II are in 4-to-1 correspondence with \vspace{-1ex}
 \small
 \begin{align}
 \begin{split}
    &\mathcal{P}_{II,d,\delta} :=
    \{ (K,C,\Sigma^\vee_\infty,h,p) \mid 
    K \in \GrLine, \\[-1mm] & \quad K^\infty \neq K, \, C: \PP^1 \dashrightarrow K, \, \deg(C) = d,
       \\[-1mm] 
      &\quad \Sigma^\vee_\infty: \mathbb{P}^1  \dashrightarrow  (K^\infty)^\vee, (v:t) \mapsto Av+Bt \text{ for some } \\[-1mm] &\quad A,B \text{ with }  A \!\cdot\! B =0 \text{ \& } A \!\cdot\! A = B \!\cdot\! B, \\[-1mm] &\quad (h,p) \in \PP(\R[v,t]_\delta \times  \R[v,t]_{\delta+d+1}) 
      \}.
 \end{split}
\end{align}
\normalsize
 \rs cameras of type III are the special case when $d=0$ and the image of the constant map $C$ is not  at infinity.
Moreover, 
$    \dim \mathcal{P}_{II,d,\delta} = 3d +  2\delta + 8.
$
\end{proposition} 
An element $(K,C, \Sigma^\vee_\infty,h,p)$ in $\mathcal{P}_{II,d,\delta}$ gives rise to 4 \rs cameras as follows:
$C$ describes the movement of the camera center. As above, the map $\Sigma_\infty^\vee$ determines the rolling planes map $\Sigma$, which fixes (up to orientation) 2 degrees of freedom of each rotation $R(x)$ for $x \in \mathbb{R}$.
To fix the 3rd degree, we assume (by rotating and translating) that $K$ is the $z$-axis.
Then, $C = (0:0:C_3:C_0)$, where $C_3,C_0$ are homogeneous of degree $d$, and $\Sigma^\vee = (\Sigma_1:\Sigma_2:0:0)$ with $\Sigma_1,\Sigma_2$ homogeneous linear. 
The polynomials $h,p$ define a map $\xi: \PP^1 \dashrightarrow \GrLine$ in Plücker coordinates: 
$(0\!:\!-h \Sigma_2 C_3\!:\! -h \Sigma_2 C_0\!:\! h \Sigma_1 C_3\!:\! h \Sigma_1 C_0\!:\! p)$\footnote{
The Plücker coordinates $(p_{12}\!:\!p_{13}\!:\!p_{10}\!:\!p_{23}\!:\!p_{20}\!:\!p_{30})$ of the line spanned by $(a_1\!:\!a_2\!:\!a_3\!:\!a_0)$ and $(b_1\!:\!b_2\!:\!b_3\!:\!b_0)$  are 
$p_{ij} = a_ib_j - a_jb_i$.}.
This map satisfies $C(v:t) \in \xi(v:t) \subseteq \Sigma(v:t)$ for all $(v:t) \in \PP^1$ (see \Cref{lem:cameraRayMapII}).
It chooses the unique camera ray $\xi(x)$ that the projection matrix $P(x)$ maps to $(0\!:\!1\!:\!0)$. 
So, up to sign, $\xi(x)$ fixes $R(x)$.
The two choices of orientation, that on $\Sigma$ and that on $\xi$, give us 4 rotation maps $R: \mathbb{R} \to \mathrm{SO}(3)$.

\subsection{Constant rotation}\label{ssec:const-rot}

In the special case that cameras do not rotate, we can more easily check whether the order is 1. 
In fact, the center-movement map $C$ is rational iff $\Lambda$ is rational (see \eqref{eq:LambdaMap} and~\cref{lem:LambdaRationalImplications}), and the condition that all rolling planes meet in a line implies all other conditions in~\cref{thm:rs1cameras}.

\begin{proposition}\label{prop:rs1NoRotation}
    A RS camera with constant rotation and rational center movement has order one if and only if the intersection of all its rolling planes is a line.
\end{proposition}

We denote the three rows of the constant rotation matrix $R \in \mathrm{SO}(3)$ by $w_1,w_2,w_3$.
Then, we can write the rolling planes map \eqref{eq:SigmaMap} as
$\Sigma: (v:t) \mapsto (w_2:0) \vee (vw_1+tw_3:0) \vee C(v:t)$.
In particular, all rolling planes go through the point $(w_2:0)$.
When the intersection of the rolling planes is a line $K$, 
then $(w_2:0)$ is the unique point of intersection of $K$ with the plane at infinity.
(Otherwise, if $K$ were contained in that plane, then all rolling planes would be equal to that plane.)
In particular, the line $K$ determines the second row of $R$ (up to sign) and the remaining rows of $R$ determine the rolling planes map as
\begin{align}
\label{eq:SigmaViaR}
    \Sigma: (v:t) \mapsto K \vee (vw_1+tw_3:0).
\end{align}

\noindent
So, for constant rotation, the spaces of \rs cameras 
are
\begin{align} \begin{split}
    &\mathcal{P}_{I,d} := \{ (R,K,\mathcal{C})  \mid
     R \in \mathrm{SO}(3) \text{ with 2nd row } w_2,\\[-1mm] & \quad 
    K \in \GrLine, \, K^\infty = (w_2:0), 
    (\mathcal{C},K) \in \mathcal{H}_d
    \}, 
\end{split} 
\end{align} \vspace{-4mm}
\small\begin{align} 
    &\mathcal{P}_{II,d} \!:=\! \{ (R,K, C)  \mid 
    R \in \mathrm{SO}(3) \text{ with 2nd row } w_2,\\[-1mm]
    & \quad K \in \GrLine, \, K^\infty \!=\! (w_2\!:\!0),
    C : \PP^1 \!\dashrightarrow\! K, \, \deg(C) = d
    \} 
    \nonumber
\end{align} \normalsize
and $\mathcal{P}_{II,0}$ (with $\mathrm{im}(C) \neq K^\infty$) is the space of static pinhole cameras; note 
    $\dim \mathcal{P}_{I,d} = 3d+6$ and $\dim \mathcal{P}_{II,d} =  2d+6$.

\subsection{Moving along a line with constant speed}
\label{sec:constantSpeed}

In many applications, where the camera moves along a line, it moves with approximately constant speed. 
Projectively, this means that the parameterization $C$ of the line $\mathcal{C}$ is birational with $C(1:0) =  \mathcal{C}^\infty$.
In the case of \rs cameras of type I, this means that $\Sigma(1:0) = K \vee \mathcal{C}^\infty$ is already determined by $K$ and $\mathcal{C}$, and cannot be freely chosen.
Thus, the space of such \rs cameras is
\small
\begin{align} \begin{split}
    &\mathcal{P}^\mathrm{cs}_{I,1,\delta} := \{
         (K,\mathcal{C},\Sigma^\vee_\infty,\lambda)
 \mid 
      K,\mathcal{C} \in \GrLine, \\[-1mm]& \quad \mathcal{C}^\infty \neq \mathcal{C}, \, K^\infty \neq K, \,  K \cap \mathcal{C} = \emptyset,  \\[-1mm]& \quad
       \Sigma^\vee_\infty: \mathbb{P}^1  \dashrightarrow  (K^\infty)^\vee, (v:t) \mapsto Av+Bt, \text{ where }  \\[-1mm]& \quad A \!=\! (K^\infty \!\vee\! \mathcal{C}^\infty)^\vee, \, A \cdot B =0, A \cdot A = B \cdot B, \\[-1mm] & \quad
      \lambda: \PP^1 \dashrightarrow K, \, \deg (\lambda) = \delta 
    \}   
\end{split}
\end{align}
\normalsize

\begin{figure}
    \centering
    \includegraphics[width=0.6\linewidth]{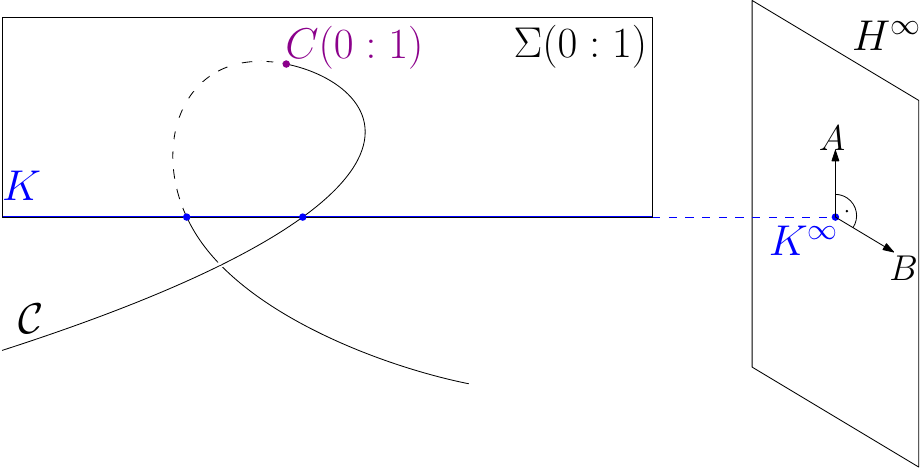}
    \caption{Illustration of $\mathcal{P}_{1,3,\delta}$: E.g., the rolling plane $\Sigma(0:1)$ meets the infinity plane $H^\infty$ at a line with normal vector~$B$.}
    \label{fig:paramRS1}
\end{figure}

The dimension of this space is one less than the space $\mathcal{P}_{I,1,\delta}$ without the constant-speed assumption,
 i.e., $ \dim \mathcal{P}^\mathrm{cs}_{I,1,\delta} = 2 \delta + 9$, as the birational map $\Sigma^\vee_\infty$ is already prescribed.
 In fact, over $\mathbb{R}$, there are two such maps since the coefficient vector $B$ can be scaled by $-1$.
Similarly, the space of constant-speed \rs cameras of type II is
\small 
\begin{align} \begin{split}
 & \mathcal{P}^\mathrm{cs}_{II,1,\delta} \! := \!\{
         (K, C,  \Sigma^\vee_\infty, h, p)
   \mid
         K \!\in\! \GrLine,  K^\infty \!\neq\! K, \\[-1mm]&\quad C: \PP^1 \dashrightarrow K \text{ birational},  C(1\!:\!0) \!=\! K^\infty, \\[-1mm]&\quad 
         \Sigma^\infty: \mathbb{P}^1  \dashrightarrow  (K^\infty)^\vee, (v:t) \mapsto Av+Bt \text{ for some } \\[-1mm]&\quad  A,B \text{ with }  A \cdot B =0, A \cdot A = B \cdot B, \\[-1mm]&\quad 
      (h,p) \in \PP(\R[v,t]_\delta \times  \R[v,t]_{\delta+2})
    \}    
\end{split}
\end{align}
\normalsize
and has dimension $\dim \mathcal{P}^\mathrm{cs}_{II,1,\delta} = 2\delta+10$.

\subsection{Linear \texorpdfstring{\rs}{RS1}cameras}
\label{sec:linearCameras}
As in \cite{dai2016rolling}, we call a RS camera that does not rotate and moves along a line with constant speed  a \emph{linear RS camera}.
We describe such cameras of order one.
Since, in type~I, constant speed means $\Sigma(1:0) = K \vee \mathcal{C}^\infty$,  \eqref{eq:SigmaViaR} implies that the point $\mathcal{C}^\infty$ has to lie on the line in the infinity plane that is spanned by the first two 2 of the fixed rotation matrix~$R$: \small
\begin{align}
    \begin{split}
    &\mathcal{P}_{I,1}^\mathrm{cs} := \{ (R, K, \mathcal{C})  \mid
         R = \left[\begin{smallmatrix}
             w_1 \\ w_2 \\ w_3
         \end{smallmatrix}\right] \in \mathrm{SO}(3), \\[-1mm] &\quad
         K \in \GrLine, \, K^\infty = (w_2 : 0), \,
         \mathcal{C} \in \GrLine, \\[-1mm] &\quad K \cap \mathcal{C} = \emptyset, \,  \mathcal{C}^\infty \in  (w_1 : 0) \vee (w_2 : 0)  
    \}.
    \end{split}
\end{align}\normalsize
Since the rotation is constant, the embedded projection plane $\Pi(v:t)$ of the camera is only affected by parallel translation. 
It stays always parallel to the plane spanned by $w_1$, $w_2$ and the origin in $\R^3$.
Projectively, this means that $\Pi(v:t)^\infty = (w_1 : 0) \vee (w_2 : 0)$ for all $(v:t)\in \PP^1$.
Hence, the last condition in the definition of $\mathcal{P}_{I,1}^{\mathrm{cs}}$  means that the line $\mathcal{C}$ has to be parallel to the projection plane $\Pi$.
In particular, the projection plane does not change at all over time.

For linear \rs cameras of type II, we have analogously \small
\begin{align}\begin{split}
    &\mathcal{P}_{II,1}^\mathrm{cs} \!:= \{ (R, K, C)  \mid
     R \!\in\! \mathrm{SO}(3) \text{ with 2nd row } w_2,\\[-1mm] &\quad
    K \in \GrLine, \, K^\infty = (w_2:0), \\[-1mm] &\quad
    C: \PP^1 \dashrightarrow K  \text{ birational}, \,  C(1:0) = K^\infty
    \}.    
\end{split}
\end{align} \normalsize
The dimensions of these spaces are
$\dim \mathcal{P}_{I,1}^\mathrm{cs} = 3 + 2 + 3 = 8$
and $\dim \mathcal{P}_{II,1}^\mathrm{cs} = 3 + 2 + 2 = 7$.
By the following proposition, we can put $\mathcal{P}_{I,1}^\mathrm{cs}$ and $\mathcal{P}_{II,1}^\mathrm{cs}$ into a joint parameter space.
\begin{proposition}\label{prop:constandSpeedConstantRotation}
A linear RS camera (i.e., that moves with constant speed along a line $\mathcal{C} \subseteq \PP^3$ and does not rotate) has order one 
if and only if
the line $\mathcal{C}$ is parallel to the projection plane $\Pi$.
If $\mathcal{C}$ is parallel to the rolling-shutter lines on $\Pi$, then the \rs camera is of type II. 
Otherwise, it is of type I.
\end{proposition}
The joint parameter space is \small
\begin{align} \label{eq:paramSpaceLinearRScam} 
\begin{split}
    &\mathcal{P}^\mathrm{cs}_1
    := \{
(R,\mathcal{C},C) \mid
     R = \left[\begin{smallmatrix}
             w_1 \\ w_2 \\ w_3
         \end{smallmatrix}\right] \in \mathrm{SO}(3), \\[-1mm] &\quad
         \mathcal{C} \in \GrLine, \, \mathcal{C} \neq \mathcal{C}^\infty \in  (w_1 : 0) \vee (w_2 : 0), \\[-1mm] &\quad
         C: \PP^1 \dashrightarrow \mathcal{C} \text{ birational}, \, C(1:0) = \mathcal{C}^\infty
    \}.    
\end{split}
\end{align} \normalsize
The parameters with $\mathcal{C}^\infty = (w_2:0)$ correspond to type-II cameras; the others to type-I cameras.

\begin{remark}
    \label{rem:constantSpeedOrder2}
    A linear RS camera  whose center moves on a line that is \emph{not} parallel to the projection plane has order two.
\end{remark}
\section{The image of a line }\label{sec:Image-of-a-line}

\rs cameras view 3D points exactly once. So their image of a 3D line is an irreducible curve. However, that curve is typically not a line. The degree of that image curve is the degree of the picture-taking map $\Phi$~in~\eqref{eq:PhiMap}.

\begin{theorem}    \label{thm:degreePhi}
    The degree of $\Phi$ for a general \rs camera in the parameter spaces described above is\\ \small\begingroup
\setlength{\tabcolsep}{2pt}
\begin{tabular}{c|c|c|c|c|c|c|c}
        $\mathcal{P}_{I,d,\delta}$ & 
        $\mathcal{P}_{I,1,\delta}^\mathrm{cs}$ &
        $\mathcal{P}_{II,d,\delta}$ &
        $\mathcal{P}_{II,1,\delta}^\mathrm{cs}$  
        & $\mathcal{P}_{I,d}$ & 
        $\mathcal{P}_{I,1}^\mathrm{cs}$ & 
        $\mathcal{P}_{II,d}$ & 
        $\mathcal{P}_{II,1}^\mathrm{cs}$ 
        \\ \hline
        $2d\!+\!\delta\!+\!1$ & 
        $\delta+3$ & 
        $2d\!+\!\delta\!+\!2$ & 
        $\delta+4$ & $d+1$ & 
        $2$ & 
        $d+1$ & 
        $2$
    \end{tabular} \normalsize
\endgroup
\end{theorem}

The image curve of a line under an \rs camera is not an arbitrary rational curve of the degree as prescribed in~\cref{thm:degreePhi}. In fact, they have a single singularity at the image point where all rolling lines meet.

\begin{proposition} \label{prop:highMultiplicityPoint}
Consider an \rs camera.
For a general line $L$, its image $\Phi(L)$ is a curve of degree $\mathrm{deg}(\Phi)$ with multiplicity \mbox{$\deg (\Phi)-1$} at the point $(0:1:0)$.
\end{proposition}

\begin{figure*}
\centering
\begin{minipage}{0.68\textwidth}
\scalebox{0.35}{
\tikzset{every picture/.style={line width=0.75pt}} 
}
\caption{Illustration of all minimal problems as lines and points in $\mathbb{P}^3$. Each problem is encoded by 9 integers: The number $m$ of cameras, followed by its combinatorial signature (see \cref{def:balanced} ff.) with $p_\infty$ at the end. 
Points on dashed lines are known to be collinear in $\PP^3$, but the image conics are not observed.
Lower bounds for the degrees
are shown below the sketches, see~\Cref{deg-comp}.
``$\ast$'': the maximum from different computational runs is shown (the actual number is close to the number shown). ``$+$'': interrupted runs (the actual number is higher than the number shown).}
\label{fig:minmtwo}
\end{minipage}
\hspace*{10pt}
\begin{minipage}{0.28\textwidth}
\centering
\begin{tabular}{c}
\fbox{(a)~\includegraphics[height=1.9cm,]{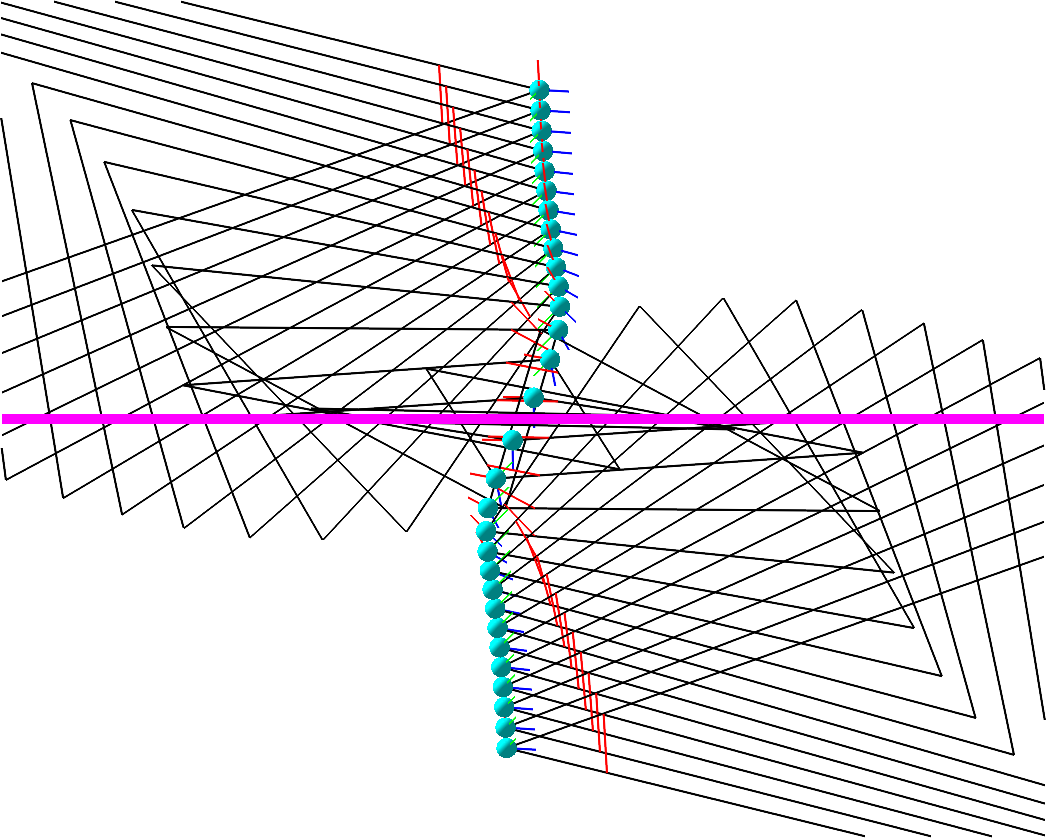}
\includegraphics[height=1.9cm,width=1.5cm]{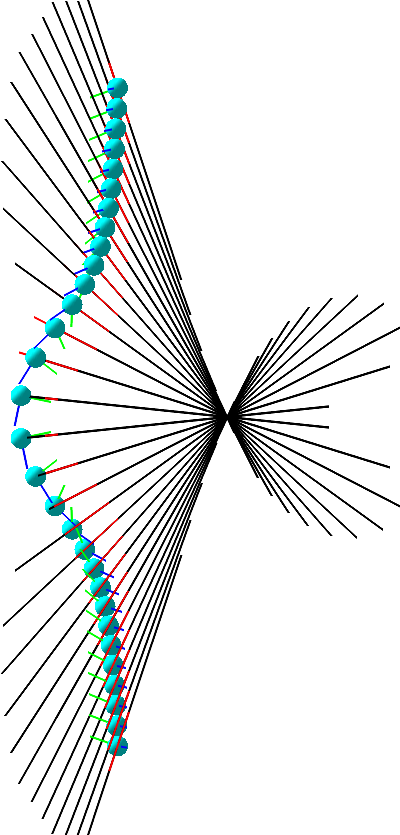}}\\
\fbox{(b)~\includegraphics[height=1.9cm,width=2.15cm]{Figs/simpl-model-figs/22-2.png}
\includegraphics[height=1.9cm,width=1.7cm]{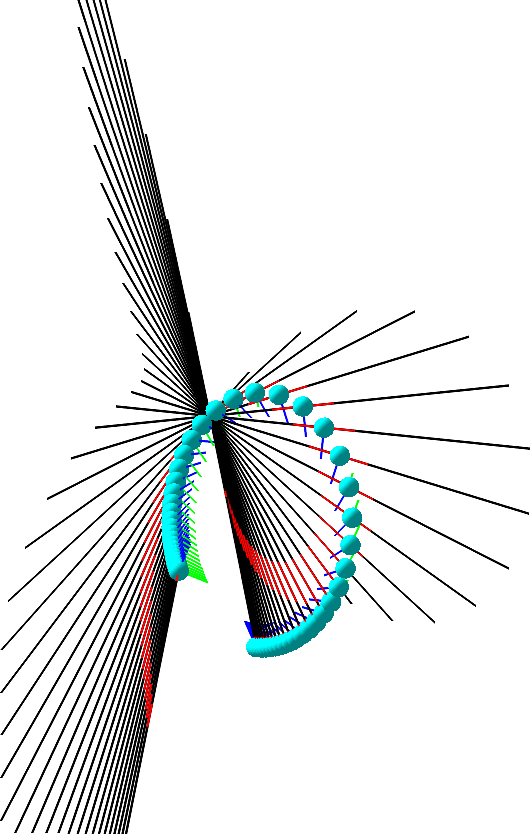}}\\
\fbox{(c)~\includegraphics[height=1.9cm,width=2.07cm]{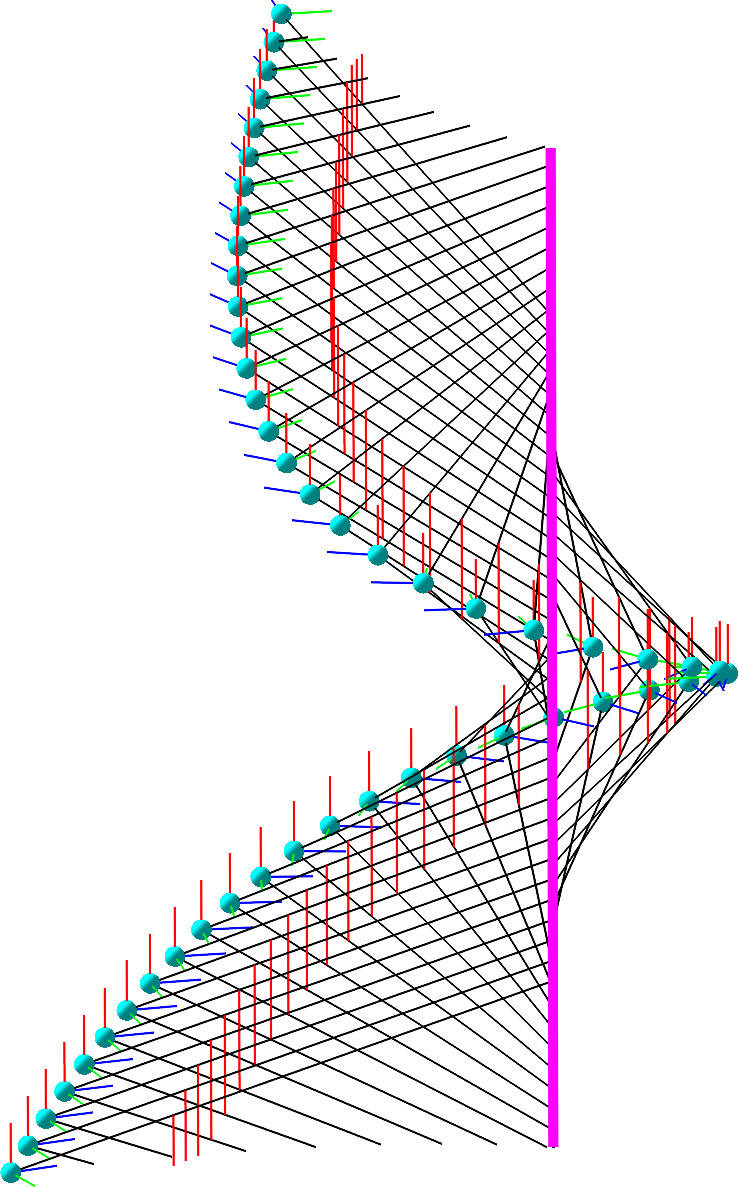}
\includegraphics[height=1.9cm]{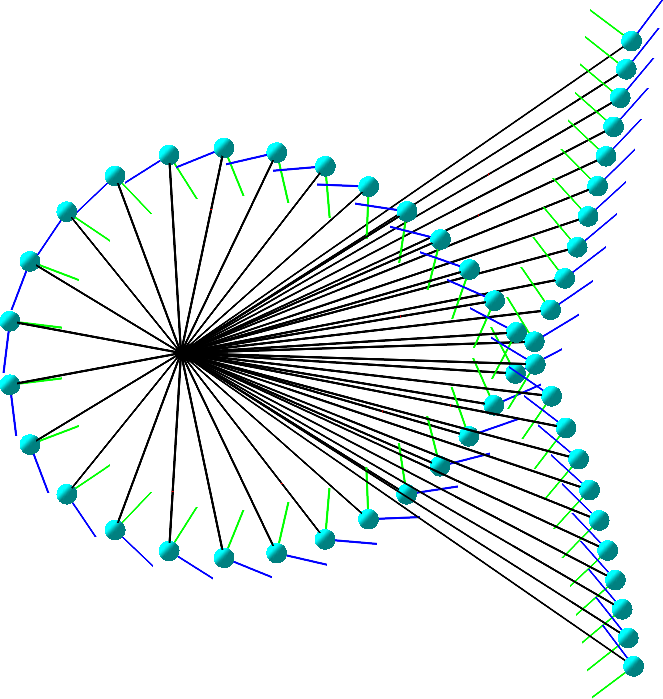}}
\end{tabular}
\caption{
(a)  \cref{ex:twistedCubicCameras1}, (b)  \cref{ex:twistedCubicCameras2}, (c)  \cref{ex:twistedCubicCameras3}. The camera center ${\textcolor{cyan}C}$ (cyan) moves along a twisted cubic curve $\mathcal{C}$. The rolling planes $\Sigma$ (black) intersect in a line ${\textcolor{magenta} K}$ (magenta). See~\cref{sec:details-twistedCubicCameras} for more details.}
\label{fig:twistedCubicCameras} 
\vspace*{1em}
\end{minipage}
\end{figure*}

\section{Minimal problems of linear \texorpdfstring{\rs}{RS1}cameras}\label{sec:Min-Prob}

 The linear \rs cameras are classified in \Cref{prop:constandSpeedConstantRotation}.
This section classifies the minimal problems of structure-from-motion (SfM) from linear \rs cameras that observe points, lines, and their incidences. 
In this setting, SfM is the following 3D reconstruction problem:

\begin{problem}\label{reconstruction-problem}
We have pictures $Y_1,\dotsc,Y_m$ of a finite set $X$ of points and lines in space.
The points and lines in $X$ satisfy some prescribed incidences.
Each picture was taken by a linear \rs camera.
Find the set $X$ and the camera parameters that produced the pictures.
\end{problem}

Let $\mathcal P^m$ be the set of all $m$-tuples of camera parameters.
We can take $\mathcal{P}^m = (\mathcal{P}_1^{\mathrm{cs}})^m$, where the latter parameter space is defined in \eqref{eq:paramSpaceLinearRScam}.
We consider point-line arrangements in $3$-space consisting of $p$ points and $\ell$ lines whose incidences are encoded by an index set $\mathcal{I} \subset [p ] \times [ \ell ]$.
That $(i,j) \in \mathcal{I}$ means that the $i$-th point is contained in the $j$-th line.
We model intersecting lines by requiring their  intersection  point to be one of the $p$ points.
We  write $\mathfrak X = \mathfrak X(p,\ell,\mathcal{I})$ for the variety of all $p$-tuples of points and $\ell$-tuples of lines in space that satisfy the incidences \mbox{prescribed by $\mathcal{I}$.}

By \Cref{thm:degreePhi}, a general linear \rs camera maps a general line in space to an image conic. 
The next lemma explains that the line at infinity on the image plane intersects each image conic at two points: one is $(0:1:0)$, the other depends only on the camera parameters.

\begin{lemma}\label{lemma-p-infty-short}
Let $(R,\mathcal{C},C)\in \mathcal{P}_1^{cs}$ be of type I and  $(a:b:0)$ be such that $\mathcal{C}^\infty = (a:b:0) \cdot R$.
Consider the associated picture-taking map $\Phi\colon\mathbb{P}^3\dashrightarrow\mathbb{P}^2$ and a line $L\subset\mathbb{P}^3$. Then:
\begin{itemize}
	\item either $\overline{\Phi(L)}$ is a conic through 
 $(0:1:0)$ and $(a:b:0)$,
	\item or $\overline{\Phi(L)}$ is the line through 
 $(0:1:0)$ and $(a:b:0)$.
\end{itemize}
\end{lemma}

Hence, the camera maps a general point-line arrangement in $\mathfrak{X}$ to an arrangement in the image plane consisting of $p$ points and $\ell$ conics that satisfy the incidences $\mathcal{I}$ and such that all conics pass through the same two points at infinity, one of them being $(0:1:0)$. 
We write $\mathcal Y(p,\ell,\mathcal{I})$ for the variety of all such planar point-conic arrangements. 

Note that, as soon as such an arrangement contains at least one conic, the point $(a:b:0)$ from \Cref{lemma-p-infty-short} is known. It can be obtained from intersecting the conic with the line at infinity.
If however only image points have been observed, then the point $(a:b:0)$ is a priori unknown.
But we might have an oracle that has additional knowledge of camera parameters and provides us with that special point.
Thus, we want to allow the possible knowledge of the point $(a:b:0)$ for each involved camera, without necessarily observing any image conics.  
We set the boolean value
 $p_\infty\in\{1,0\}$ according to whether we assume knowledge of the point $(a:b:0)$ or not. 
 We write $\mathcal Y = \mathcal Y(p,\ell,\mathcal{I},p_\infty)$ for the variety of all planar point-conic arrangements in $\mathcal Y(p,\ell,\mathcal{I})$ plus the point $(a:b:0)$ if $p_\infty =1$. 

\Cref{reconstruction-problem} asks to compute the preimage of such a planar arrangement $y \in \mathcal{Y}$ under the rational joint-camera map $\mathcal P^m\times \mathfrak X\dashrightarrow \mathcal Y$.
The scaled special Euclidean group $G$ from \Cref{rem:groupAction} acts on the preimages of that map, so we rather consider the quotient on the domain and let $\Phi^{(m)} = \Phi^{(m)}(p,\ell,\mathcal{I})$ be the map
\begin{align}\label{eq:jcm}
    \Phi^{(m)}:(\mathcal P^m\times \mathfrak X) / G \dashrightarrow \mathcal Y.
\end{align}
\begin{definition} \label{def:minimal}
    The Reconstruction \cref{reconstruction-problem} is \emph{minimal} if its solution set is non-empty and finite for generic input pictures. 
    In that case, the number of complex solutions given a generic input is the \emph{degree} of the minimal problem.
\end{definition}
 \begin{theorem} \label{thm:minimalMain}
     There are exactly 31 minimal problems for SfM with linear \rs cameras fully observing point-line arrangements, where either all or no given views know the special point $(a\!:\!b\!:\!0)$. They are shown in \cref{fig:minmtwo}.
\end{theorem}

\section{Straight-Cayley cameras}\label{sec:Straight-Cayley}
In~\cite{Albl-PAMI-2019}, the Straight-Cayley RS model with a constant speed translation in the camera coordinate system $\beta_u$ and Cayley rotation parameterization was used to set up a tractable minimal problem of RS absolute pose camera computation:
\begin{eqnarray}
\lambda [u \,  v \, 1]^\top = R\left((u-u_0)\,O_\delta\right)\,X_\delta + T_{\beta_u} + (u-u_0)\,V_{\beta_u} \nonumber \\
R(c)=\frac{1}{d} 
\left[ \begin{smallmatrix}
1+c_1^2-c_2^2-c_3^2&2\,(c_1 c_2- c_3)&2\,(c_1 c_3+ c_2)\\
2\,(c_1 c_2+ c_3)&1-c_1^2+c_2^2-c_3^2&2\,(c_2 c_3- c_1)\\
2\,(c_1 c_3- c_2)&2\,(c_2 c_3+ c_1)&1-c_1^2-c_2^2+c_3^2
\end{smallmatrix} \right]
\label{eq:Cayley}\\ 
d= 1+c_1^2+c_2^2+c_3^2  \nonumber
\end{eqnarray}
Here, we have depth $\lambda \in \R$, image coordinates $u,v \in \R$, the offset $u_0 \in \R$, a rotation axis $O_\delta \in \R^3$ in the world coordinate system $\delta$, a 3D point $X_\delta \in \R^3$, a camera center $T_{\beta_u} \in \R^3$ for $u = u_0$ in the cameras coordinate system $\beta_u$, and a translation direction vector $V_{\beta_u} \in \R^3$. The translation velocity is given by $\|V_\delta\| = \|V_{\beta_u}\|$. The rotation angle $\theta$ of $R\left((u-u_0)\,O_\delta\right)$ around the axis $o$ is determined by $(u-u_0)\|O_\delta\| = \tan(\theta/2)$. Thus, for small angles, $\theta \simeq (u-u_0)\|O_\delta\|/2$, and $\|O_\delta\|/2$ approximates the angular velocity of the rotation.  
We now show how the camera center moves in the world coordinate system $\delta$. We can write
\begin{align}
& \lambda [u\, v\, 1]^\top =  R\!\left((u-u_0)\,O_\delta\right) \cdot \\ & \quad \left(X_\delta - R\!\left((u-u_0)\,O_\delta\right)^\top \left(-T_{\beta_u} - (u-u_0)\,V_{\beta_u}\right)\right),
   \nonumber
\end{align}
 so the camera center in the world coordinate system $\delta$ is
\begin{equation}
    C_\delta (u) = - R\!\left((u-u_0)O_\delta\right)^\top \left(T_{\beta_u} + (u-u_0)V_{\beta_u}\right).
    \label{eq:C-curve}
\end{equation}
\begin{proposition}\label{prop:twistedCubicGeneral}
    For generic choices of the parameters $O_\delta, T_{\beta_u}, V_{\beta_u}$, the curve $\mathcal{C}$ parametrized by $C_\delta$ is a twisted cubic curve
    and the RS camera has order four.
\end{proposition}

We are interested in understanding when such a camera has order one and falls into the setting of this paper. Recall that a necessary condition for order one is that all rolling planes intersect in a line.

\begin{theorem} \label{thm:twistedCubicOrder1}
    All rolling planes intersect in a line if and only if the parameters $O_\delta = (o_1,o_2,o_3), T_{\beta_u} = (t_1,t_2,t_3), V_{\beta_u} = (v_1,v_2,v_3)$ satisfy one of the following:
    \begin{enumerate}
        \item $v_3 = 0$, $o_3=0$, and $o_2 = -1$; or
        \item $v_3=0$, $o_1=0$, and $o_2^2+o_3^2+o_2 =0$; or
        \item $v_3=0$, $v_1=t_3$, $t_1=0$, $o_1 =0$, and $o_3=0$.
    \end{enumerate}
    In each case, the camera-center curve $\mathcal{C}$ is generically still a twisted cubic curve.
    In the first two cases, the RS camera has order one.
    In the third case, the RS camera has generically order three, and its order is one if and only if the parameters also satisfy the conditions in either 1. or 2.
\end{theorem}

\noindent
Thus, the first 2 cases of Thm. \ref{thm:twistedCubicOrder1} describe all \rs cameras.
\begin{proposition}
\label{prop:degreePhiTwistedCubic}
    For both cases of \rs cameras, the picture-taking map $\Phi$
    is generically of degree four: it maps generic lines in space to quartic image curves.
\end{proposition}
\begin{example}[$O_\delta \!=\! (1,-1,0),  T_{\beta_u} \!=\! (0,0,1),  V_{\beta_u} \!=\! (0,1,0)$] \label{ex:twistedCubicCameras1}
This is an example of case 1 in~\cref{thm:twistedCubicOrder1}.
The camera center moves on the twisted cubic curve defined~by
$X_1X_3-X_2X_3+X_2X_0=2X_2^2+X_3^2+X_3X_0=2X_1X_2+X_3^2-X_0^2=0.$
The line $K$, which is the intersection of all rolling planes, is defined by $X_1 = 0$ and $X_3=X_0$. It intersects the twisted cubic $\mathcal{C}$ at two complex conjugated points. 
The picture-taking map $\Phi: \PP^3 \dashrightarrow \PP^2$ sends $(X_1:X_2:X_3:X_0)$ to
{\tiny
\begin{align*}
\left(\hspace*{-3ex}\begin{array}{c}
-2X_1^2X_2X_3-X_1X_3^3+2X_1^2X_2X_0+X_1X_3^2X_0+X_1X_3X_0^2-X_1X_0^3\\
-\!2\!X_1^3X_3 \!\!-\!\! 2\!X_1X_3^3 \!\!+\!\! X_2X_3^3 \!+\! 3\!X_1X_3^2X_0 \!\!-\!\! 3\!X_2X_3^2X_0 \!\!+\!\! 3\!X_2X_3X_0^2 \!\!-\!\! X_1X_0^3 \!\!-\!\! X_2X_0^3\\
2X_1X_2X_3^2+X_3^4-4X_1X_2X_3X_0-2X_3^3X_0+2X_1X_2X_0^2+2X_3X_0^3-X_0^4
\end{array}\hspace*{-3ex}\right).
\end{align*}}
\end{example}
\begin{example}[$O_\delta \!=\! (0,-\tfrac{1}{2},\tfrac{1}{2}),  T_{\beta_u} \!=\! (0,0,1),  V_{\beta_u} \!=\! (0,1,0)$] \label{ex:twistedCubicCameras2}
This is an example of case 2 in~\cref{thm:twistedCubicOrder1}.
The camera center moves on the twisted cubic curve defined~by 
$X_1X_2 \!-\! X_1X_3 \!-\! X_1X_0 \!+\! 2X_3X_0 \!+\! 2X_0^2=
X_1^2 \!+\! 2X_2X_3 \!+\! 2X_2X_0 \!-\! 2X_3X_0 \!-\! 2X_0^2=
X_2^2-X_2X_3-X_1X_0-2X_2X_0+X_3X_0+X_0^2=0.$
    \noindent
     The line $K$,  the intersection of all rolling planes, is defined by $X_1 = 0$,  $X_2=X_0$. It meets the twisted cubic $\mathcal{C}$ at the points $(0:0:1:0)$ and $(0:1:-1:1)$. 
    The picture-taking map $\Phi: \PP^3 \dashrightarrow \PP^2$ sends $(X_1:X_2:X_3:1)$~to
     {\tiny
     \begin{align*}
    \left(\hspace*{-2ex} \begin{array}{l}
        {-X_1^3X_2\!-\!2X_1X_2^2X_3\!+\!X_1^3\!-\!2X_1X_2^2\!+\!4X_1X_2X_3\!+\!4X_1X_2\!-\!2X_1X_3\!-\!2X_1}\\
%
        2X_1^2X_2^2\!+\!2X_2^4\!-\!X_1^2X_2X_3\!-\!X_1^3\!-\!4X_1^2X_2\!-\!2X_1X_2^2\!- \\ 
        \hspace*{24mm} - \!6X_2^3\!+\!X_1^2X_3\!+ \!2X_1^2\! +\!4X_1X_2\! +\!6X_2^2\!-\!2X_1\!-\!2X_2\\
        X_1^2X_2^2\!\!+\!\!2\!X_2^3X_3\!\!-\!\!2\!X_1^2X_2\!\!+\!\!2\!X_2^3\!\!-\!\!6\!X_2^2X_3\!\!+\!\!X_1^2\!\!-\!\!6\!X_2^2\!\!+\!\!6\!X_2X_3\!\!+\!\!6\!X_2\!\!-\!\!2\!X_3\!\!-\!\!2
    \end{array}\!\!\!\right).
    \end{align*}}
%
\end{example}
\begin{example}[$O_\delta \!=\! (0,1,0),  T_{\beta_u} \!=\! (0,0,1), V_{\beta_u} \!=\! (1,1,0)$] \label{ex:twistedCubicCameras3}
This is an example of case 3 in~\cref{thm:twistedCubicOrder1}.
    The camera center moves on the twisted cubic curve defined~by
    $X_2X_3\!+\!X_1X_0\!-\!2X_2X_0=X_1X_2\!+\!X_2^2\!-\!X_3X_0\!-\!X_0^2=X_1^2\!-\!X_2^2\!+\!X_3^2\!-\!X_0^2=0.$
    The line $K$, which is the intersection of all rolling planes, is defined by $X_1 = 0$ and $X_3=0$.
    It does not meet the twisted cubic curve $\mathcal{C}$.
\end{example}
\section{Conclusion}
We provided a new model of RS cameras and characterized \rs cameras whose picture-taking process is encoded in a rational map. We described parameter spaces of \rs cameras and how images of lines taken by such cameras look. We classified all point-line minimal problems for linear \rs cameras and discovered new problems with few solutions and image features. In future work, we plan to implement and test the practicality of those new minimal solvers. 
For the found minimal problems with higher degrees, we plan to investigate whether they decompose into smaller problems via monodromy groups \cite{duff2022galois}. We further plan to classify the minimal relative pose problems for the Straight-Cayley RS model to exhibit whether this model is not only practical for absolute, but also relative pose. Finally, we will analyze higher-order cameras and the affect of the \mbox{order on relative pose problems.}

\printbibliography

\newpage 
\clearpage
\appendix

 This is the Supplementary Material for the paper {\em ``M. A. Hahn, K. Kohn,  O. Marigliano, T. Pajdla. Order-One Rolling Shutter Cameras. CVPR 2025.''}

We provide additional notations, definitions, concepts, technical lemmas, and proofs of all results in the main paper. The code we used for computations is included as a pair of ancillary files. See \texttt{readme.txt} for how to use the code.

\section{Derivations and Proofs}

We start by explaining some of the basic mathematical objects we use throughout the article. Our starting point is the projective three-space $\mathbb P^3$, whose points are determined by their homogeneous coordinates $(x_1:x_2:x_3:x_0)$.
\medskip

\noindent\textbf{Dual space.} The \emph{dual space} $(\mathbb P^3)^*$ is the set of planes in $\mathbb P^3$. It is isomorphic to $\mathbb P^3$ by identifying a plane $\Sigma = \{x\in\mathbb P^3 \mid \sum_i c_i x_i = 0\}$ with the coefficients $\Sigma^\vee = (c_1:c_2:c_3:c_0)$ of its defining linear equation. If $L\subseteq \mathbb P^3$ is a line, then the set of all planes that contain $L$ form a line in $(\mathbb P^3)^*$ which is called the \emph{dual line} of $L$ and denoted by $L^\vee$.
\medskip

\noindent\textbf{Grassmannian, pencil of lines.} The set of lines in $\PP^3$ is called the Grassmannian $\mathrm{Gr}(1,\PP^3)$. Each of its elements can be represented by a matrix
\[
\begin{pmatrix}
c_1 & c_2 & c_3 & c_0 \\
d_1 & d_2 & d_3 & d_0
\end{pmatrix}
\]
whose rows contain the coefficients of the two linear equations that define the line. Two such matrices represent the same element of $\mathrm{Gr}(1,\PP^3)$ if there is a $2\times 2$ invertible matrix that takes one to the other.

Elements of the Grassmannian are uniquely determined by their \emph{dual Plücker coordinates} $(p_{12}:p_{13}:p_{10}:p_{23}:p_{20}:p_{30})$, which are computed as the six $2\times 2$ minors $p_{ij} = c_id_j - c_jd_i$ of the above matrix.
This exhibits the Grassmannian as a subset of $\PP^5$. It is in fact an \emph{algebraic variety} since it is the solution set of the polynomial equation $p_{12}p_{30}-p_{13}p_{20}+p_{10}p_{23}=0$ over $\PP^5$.

Analogously, one can uniquely represent a line in $\mathbb{P}^3$ via its \emph{primal Plücker coordinates}, which are the $2\times 2$ minors of a $2 \times 4$ matrix whose rows span the line.

A \emph{pencil} of lines is a one-dimensional family of lines passing through a common point and contained in a common plane. If the common point is at infinity, the lines of the pencil are parallel to each other. A pencil of lines in $\PP^3$ can be seen as a line (i.e., a curve of degree one) inside $\Gr(1,\mathbb P^3)$. 
\medskip

\noindent\textbf{Algebraic varieties, Zariski closure, degree.} An \emph{algebraic variety} is the set of solutions of a set of polynomial equations over some field. In our paper, we mostly consider varieties $V$ inside some projective space $\PP^n$ over the real or the complex numbers. 
These take the form $V(E)$ for a set $E$ of polynomials in several variables. When working in projective space, we must require these polynomials to be homogeneous. 
\\ \indent
For a subset $X$  of $\PP^n$, we define its Zariski closure as the smallest variety containing $X$. \\ \indent
Let $V$ be a variety of dimension $d$ in $\PP^n$. We define the degree of $X$ as the number of complex intersection points of $V$ with a \textit{generic} linear space of dimension $n-d$. Here generic means that the linear space corresponds to a point outside of a Zariski--closed subset of the parameter space of linear spaces. This parameter space is the Grassmannian $\mathrm{Gr}(n-d,\PP^n)$, a generalisation of the object defined above.
\medskip

\noindent\textbf{Dual varieties.}
Several of our proofs use projective duality of algebraic varieties. Analogously to the dual three-space, for any $n\in\mathbb N$, the \emph{dual space} $(\PP^n)^*$ is defined as the set of hyperplanes in $\PP^n$. The elements of $(\PP^n)^*$ are likewise represented as $(n+1)$-tuples of projective coordinates.
For any subvariety $V$ of $\PP^n$, its \emph{dual variety} $V^\vee$ is the Zariski closure inside $(\PP^n)^\ast$ of the set of hyperplanes that are tangent at some smooth point of $V$.
\\ \indent
Over the complex numbers, we have $(V^\vee)^\vee = V$. This is called \emph{projective duality}. More concretely, we have the following over $\mathbb{C}$: For a smooth point $p$ of $V$ and a smooth point $H^\vee$ of the dual variety $V^\vee$, the hyperplane $H$ is tangent to $V$ at the point $p$ if and only if the hyperplane $p^\vee$ is tangent to $V^\vee$ at the point $H^\vee$.

\subsection{Derivation of \texorpdfstring{$\Lambda$}{Lambda} map} \label{ssec:Lambda}

The camera ray through an image point results from intersecting backprojected planes of the image lines $\rho(v:t)$ and $\nu(u:s)$: the map
\small
\[ 
\Lambda:\PP^1 \times \PP^1 \dashrightarrow \GrLine^*
\]
\normalsize
sends a point $((u:s),(v:t))$ to
\small
\begin{align} \begin{split}
&  [-t,0,v]\,R(v:t)\,[I_3 \mid -C(v:t)] \\
& {}\wedge{}   [0,-s,u]\,R(v:t)\,[I_3 \mid -C(v:t)]\\
& =  \mat{cc}{I_3\\-C(v:t)^\top} R(v:t)^\top
    \mat{rrr}{0~ & -s t & ut \\s t& 0~ & -sv\\ -ut & sv & 0~}\\
&   {}\cdot{}   R(v:t)[I_3 \mid -C(v:t)]\\
& = \mat{cc}{I_3\\-C(v:t)^\top} R(v:t)^\top
    \xx{\mat{r}{s v\\ u t\\ s t}}\\
&   {}\cdot{}    R(v:t)\, [I_3 \mid -C(v:t)]\\
& = \mat{cc}{I_3\\-C(v:t)^\top} 
    \xx{R(v:t)^\top\mat{r}{s v\\ u t\\s t}}\\
&   {}\cdot{}    [I_3 \mid -C(v:t)].
\end{split}
\end{align}
\normalsize
%

\subsection{Classifying \texorpdfstring{\rs}{RS1}cameras}

In this section, we prove Theorem \ref{thm:rs1cameras}.
We assume throughout this section that the map $\Lambda$ is rational.

\begin{lemma}    \label{lem:LambdaRationalImplications}
    If the map $\Lambda$ is rational, so are the maps $C$ and $\Sigma^\vee$.
\end{lemma}

\begin{proof}
    Consider $(v:t) \in \PP^1$ general.
    Then, the plane $\Sigma(v:t)^\vee $ is the equation of the Zariski closure of $ \bigcup_{(u:s) \in \PP^1} \Lambda((v:t), (u:s))$, which is the span of $\Lambda((v:t), (1:0))$ and $\Lambda((v:t), (0:1))$. This shows that $\Sigma^\vee$ is a rational map.
    Similarly, we see that $C$ is rational by observing that $C(v:t)$ is the intersection of $\Lambda((v:t), (1:0))$ and $\Lambda((v:t), (0:1))$.
\myqed
\end{proof}

We start with showing that our unions of pencils of lines are indeed congruences, i.e., surfaces the Grassmannian $\GrLine$.

\begin{lemma}\label{union-dimension-two}
Let $X\subseteq \PP^3 \times (\PP^3)^*$ be an irreducible curve. For $(C,\Sigma^\vee)\in X$, consider the pencil $\mathcal L(C,\Sigma)\coloneqq\{\PL\in\Gr(1,\PP^3)\mid C\in \PL\subseteq \Sigma\}$. Then the set $\mathcal L_X\coloneqq \bigcup_{(C,\Sigma^\vee)\in X} \mathcal L(C,\Sigma)$ has dimension two.
\end{lemma}
\begin{proof}
Let $\mathcal V\subseteq X\times \mathcal L_X$ be the incidence variety of elements $(C,\Sigma^\vee, \PL)$ such that $\PL\in\mathcal L(C,\Sigma)$. Since all fibers over $X$ of $\mathcal V$ are one-dimensional, the variety $\mathcal V$ is two-dimensional. It remains to show that the generic fiber of $\mathcal V$ over $\mathcal L_X$ is zero-dimensional. Pick a general $\PL\in\mathcal L_X$ and let $\mathcal C$ be the image of $X$ in $\PP^3$. Then the set $\{C\in \mathcal C\mid \PL\in\mathcal L(C,\Sigma) \text{ for some } \Sigma\} = \PL\cap \mathcal C$ is finite because $\PL$ is general. Dually, if $\mathcal S$ denotes the image of $X$ in $(\PP^3)^*$, then the set $\{\Sigma^\vee\in \mathcal S\mid \PL\in\mathcal L(C,\Sigma) \text{ for some } C\} = \PL^\vee\cap \mathcal S$ is also finite. Thus, the fiber of $\mathcal L_X$ over $\PL$ is finite. \myqed
\end{proof}

A rolling shutter camera that does not move but rotates has (almost) always the order-one congruence $\overline{\mathcal{L}}$ that consists of all lines passing through the fixed center. 
The only exception is when the rotation cancels out the movement of the rolling line such that all rolling planes are the same plane, say $\Sigma$. In that case, the camera rays do not form a congruence, but only the pencil of lines that are contained in $\Sigma$ and pass through the camera center. Such a camera can only take pictures of the points in the fixed plane $\Sigma$.

Thus, to classify all rolling shutter cameras with an order-one congruence, we may assume that the camera is actually moving  (and possibly rotating). 
In the following, we use the rolling image lines $r \in \mathcal{R}$ synonymously with our projective parameters $(v:t) \in \PP^1$, via the identification $\rho$ in \eqref{eq:rho}.

\begin{theorem} \label{thm:order1congruence}
Consider a rolling shutter camera whose center moves along a rational curve $\mathcal{C} \subseteq \PP^3$.
The associated congruence $\overline{\mathcal{L}}$ has order one 
if and only if
the intersection of all rolling planes $\Sigma(r)$ is a line $K$ that satisfies one of the following two conditions: \\
1) $K$ intersects the curve $\mathcal{C}$ in $\deg \mathcal{C}-1$ many points (counted with multiplicty), or\\
2) $K=\mathcal C$ and $\Sigma(r_1) = \Sigma(r_2)$ implies $C(r_1) = C(r_2)$ for all $r_1,r_2\in\mathcal R$.
\end{theorem}

In order to prove this theorem, we make use of Kummer's classication of order-one congruences according to their focal loci \cite{kummer-classification}. 
A \emph{focal point} of an order-one congruence $\overline{\mathcal{L}}$ is a space point that lies on infinitely many lines on $\overline{\mathcal{L}}$.
The \emph{focal locus} of the congruence is the Zariski closure of its set of focal points. 
The following version of Kummer's classification is due to De Poi \cite{de-poi-classification}.

\begin{theorem}\label{thm:kummer}
A congruence ${\mathcal{L}} \subset \GrLine$ has order one if and only if it is one of the following: 
\begin{enumerate}
    \item[i.] $\mathcal{L}$ is the set of all lines passing through a fixed point $C \in \PP^3$. Its only focal point is $C$.
    \item[ii.] $\mathcal{L}$ consists of all secant (and tangent) lines of a twisted cubic curve $\mathcal{C} \subset \PP^3$. Its focal locus is $\mathcal{C}$.
    \item[iii.] $\mathcal{L}$ consists of all lines that meet both a rational curve $\mathcal{C} \subset \PP^3$ and a line $K \subset \PP^3$ that intersects the curve $\mathcal{C}$ in $\deg \mathcal{C}-1$ many points (counted with multiplicity). Its focal locus is $\mathcal{C}\cup K$.
    \item[iv.] There is a line $K \subset \PP^3$  and a dominant morphism $\sigma: K^\vee \to K$  such that
    $\mathcal{L} = \bigcup_{\Sigma^\vee \in K^\vee} \{ \PL \in \GrLine \mid \sigma(\Sigma^\vee) \in \PL \subset \Sigma \} $.
    Its focal locus is $K$.
\end{enumerate}
\end{theorem}

\noindent
Order-one congruences of type i are associated with rolling shutter cameras that stand still (but possibly rotate).

\begin{lemma} \label{lem:type2}
    The congruence $\overline{\mathcal{L}}$ associated with a rolling shutter camera cannot be of type ii.
\end{lemma}

\begin{proof}
    The congruence is a union of pencils; in fact, ${\mathcal{L}} = \bigcup_{r \in \mathcal{R}} \mathcal{L}(r)$, where $\mathcal{L}(r) = \{ \PL \in \GrLine \mid C(r) \in \PL \subset \Sigma(r) \}$.
    In particular, its focal locus contains the curve $\mathcal{C}$ along which the camera center moves.
    If the congruence $\overline{\mathcal{L}}$ were of type ii, the curve $\mathcal{C}$ would be a twisted cubic.
    However, any rolling plane $\Sigma(r)$ intersects the curve $\mathcal{C}$ in the  camera center $C(r)$ and at most two other points, meaning that the generic line in $\mathcal{L}(r)$ is not a secant of $\mathcal{C}$. \myqed
\end{proof}

\begin{remark}\label{rem:order0}
    The congruence associated with a rolling shutter camera has order zero if and only if all rolling planes are the same plane. This is because the set of points seen by the camera is the union of all rolling planes, which is a connected surface that contains a plane.
    When the congruence has order zero and all rolling planes are the same, the congruence consists of all lines in that plane and the camera moves along a rational curve in that plane. 
\end{remark}

\begin{proof}[Proof of Theorem \ref{thm:order1congruence}]
    We start by observing that the set of camera rays $\mathcal{L}$ of a moving rolling shutter camera is always two-dimensional by Lemma \ref{union-dimension-two}.
    Moreover, the focal locus of the congruence $\overline{\mathcal{L}}$ contains the curve $\mathcal{C}$ that describes the movement of the camera center.
    In particular, $\mathcal L$ cannot be of type i.
    Hence, by Theorem \ref{thm:kummer} and Lemma \ref{lem:type2},
    the congruence $\overline{\mathcal{L}}$ has order one if and only if it is of type iii or iv.

    We start by assuming that the congruence is of type iii. By considering the focal locus, we find that the rational curve from iii coincides with the camera movement curve $\mathcal C$. Since $\mathcal L$ is the union of the pencils $\mathcal{L}(r)= \{ \PL \in \GrLine \mid C(r) \in \PL \subset \Sigma(r) \}$ and every of its lines has to meet a fixed line $K$, every rolling plane $\Sigma(r)$ has to contain $K$. 
    The rolling planes cannot all be the same, because otherwise the congruence $\overline{\mathcal{L}}$ would consist of all lines in that plane, which is a congruence of order zero.
    Therefore, the intersection of the rolling planes is exactly the line $K$, and we are in type 1) of Theorem \ref{thm:order1congruence}.
    Conversely, given a rolling shutter camera of type 1), the generic rolling plane $\Sigma(r)$ is the span of the center $C(r)$ with the line $K$, and thus $\overline{\mathcal{L}}$ is of type iii.

    If the congruence $\overline{\mathcal{L}}$ is of type iv, the camera moves along a line $\mathcal{C} = K$. This line is the focal locus of $\overline {\mathcal L}$ and thus coincides with the line given in iv.
    We claim that every rolling plane $\Sigma(r)$ contains that line $K$.
    To see that, we assume for contradiction that one of the rolling planes $\Sigma(r)$ intersects $K$ in a single point, namely $C(r)$.
    Now, consider an arbitrary plane $\Sigma$ containing $K$ (i.e.,  $\Sigma^\vee \in K^\vee$). 
    The intersection $\Sigma \cap \Sigma(r)$ is a line $\PL$ passing through the center $C(r)$.
    Thus, the line $\PL$ is on the pencil $\mathcal{L}(r) \subset \mathcal{L}$. 
    Since $\Sigma$ is the unique plane that contains both $\PL$ and $K$, the expression of $\mathcal{L}$ in type iv of Theorem \ref{thm:kummer} shows that $\sigma(\Sigma^\vee) \in \PL$, which implies that $\sigma(\Sigma^\vee) = C(r)$.
    In other words, the map $\sigma: K^\vee \to K$ is constant with image $C(r)$, which contradicts its dominance.
    Hence, we have shown that the rolling planes intersect in the line $K$ (as before, they cannot all be equal, since otherwise the congruence would have order zero).
    Moreover, we have for any rolling plane $\Sigma(r)$ that $\sigma(\Sigma(r)^\vee) = C(r)$, and we are in type 2) of Theorem \ref{thm:order1congruence}.
    Conversely, for a rolling shutter camera of type 2), we have that $K^\vee$ consists of all rolling planes and the map $\sigma: K^\vee \to K, \Sigma(r)^\vee \mapsto C(r)$ is dominant since the camera is moving.
    Thus, $\mathcal{L} = \bigcup_{r \in \mathcal{R}} \mathcal{L}(r)$ is of type~iv. \myqed 
    \end{proof}

\medskip
To classify all \rs cameras, it remains to analyze when the congruence parame\-trization map $\Lambda$ is birational.

\begin{lemma} \label{lem:LambdaBirational}
    The map $\Lambda$ is birational if and only if, for a generic $\PL \in \mathcal{L}$, there is a unique $r \in \mathcal{R}$ such that $\PL \in \mathcal{L}(r)$.
\end{lemma}
\begin{proof}
    $\Lambda$ is birational if and only if, for a generic $\PL \in \mathcal{L}$, there are unique parameters $((v:t),(u:s))$ whose image under $\Lambda$ is $\mathcal{L}$.
    Recall that the morphism $\rho$ in \eqref{eq:rho} identifies the parameters $(v:t) \in \PP^1$ with $r \in \mathcal{R}$.
    Hence, the birationality of $\Lambda$ implies for $\PL \in \mathcal{L}$ the existence of a unique $r \in \mathcal{R}$ with $\PL \in \mathcal{L}(r)$.
    Conversely, if a generic $\PL \in \mathcal{L}$ uniquely determines the pencil $\mathcal{L}(r)$ that contains it, then $\PL$ is uniquely picked in that pencil via the parameters $(u:s)$, i.e., $\Lambda$ is birational. \myqed
\end{proof}

\begin{proposition}\label{prop:birationalOrder0}
    Consider a rolling shutter camera whose associated congruence $\overline{\mathcal{L}}$ has order zero.
    The map $\Lambda$ is birational if and only if the camera moves on a line $\mathcal{C}$ and the center map $C: \mathcal{R} \dashrightarrow \mathcal{C}$ is birational.
\end{proposition}

\begin{proof}
    As in Remark \ref{rem:order0}, all rolling planes are the same plane, say $\Sigma$.
    The congruence is the family of all lines contained in $\Sigma$
    and the camera moves along a rational plane curve $\mathcal{C} \subset \Sigma$.
    By Lemma \ref{lem:LambdaBirational}, the map $\Lambda$ is birational if and only if, for a generic line $\PL \subset \Sigma$, there is a unique $r \in \mathcal{R}$ such that $C(r) \in \PL$. 
    The latter means that the generic line $\PL$ intersects the plane curve $\mathcal{C}$ in a single point (i.e., $\deg \mathcal{C}=1$) and that $C$ is a birational map. \myqed
\end{proof}

\begin{proposition}\label{prop:birationalNoMovement}
Consider a rolling shutter camera that does not move (but possibly rotates). The map $\Lambda$ is birational  if and only if the intersection of all rolling planes is a line $K$ and the rolling planes map $\Sigma$ is birational.
\end{proposition}

\begin{proof}
    The congruence is the family of lines passing through the fixed camera center $C \in \PP^3$.
    Hence, this situation is projectively dual to the setting in Proposition \ref{prop:birationalOrder0}.
    In particular, when considering the rolling planes $\Sigma(r)$ as points $\Sigma(r)^\vee$ in $(\PP^3)^*$, they form a curve in the plane $C^\vee \subset (\PP^3)^*$. 
    As in the proof of Proposition \ref{prop:birationalOrder0}, the birationality of $\Lambda$ means that that plane curve is a line (namely $K^\vee$) with a rational parametrization via the rolling planes map $\Sigma^\vee$. \myqed 
\end{proof}

\begin{theorem}\label{thm:birationalInGeneral}
Consider a moving rolling shutter camera with a congruence $\overline{\mathcal{L}}$ of positive order. The map $\Lambda$ is birational  if and only if the map $\mathcal{R} \dashrightarrow \PP^3 \times (\PP^3)^*, r \mapsto (C(r), \Sigma(r)^\vee)$ that parametrizes the pencils $\mathcal{L}(r)$ is birational.
In particular, the birationality of $C: \mathcal{R} \dashrightarrow \mathcal{C}$ implies that $\Lambda$ is birational. 
\end{theorem}

\begin{proof}
    The pairs $(C(r), \Sigma(r)^\vee)$ trace a curve in $\PP^3 \times (\PP^3)^*$, which we denote by~$\mathcal{D}$.    
    
Let $\Lambda$ be birational and let $X$ be the set of pairs $(C(r),\Sigma(r)^\vee)\in\mathcal D$ such that there is some $r'\neq r$ with $(C(r),\Sigma(r)^\vee)=(C(r'),\Sigma(r')^\vee)$. The set $X$ has either dimension zero or it has dimension one. In the former case, we are done since a generic element of $\mathcal D$ will have a unique pre-image $r$. In the latter case, the union $\bigcup_{(C(r),\Sigma(r)^\vee)\in X} \mathcal L(r) \subseteq \mathcal L$ is dense by Lemma~\ref{union-dimension-two}, contradicting Lemma~\ref{lem:LambdaBirational}.

    Now, we assume that the map $\mathcal{R} \dashrightarrow \mathcal{D}$ is birational. 
    Then, by Lemma \ref{lem:LambdaBirational}, it is enough to show that, for a generic $\PL \in \mathcal{L}$, there is a unique pair $(C,\Sigma^\vee) \in \mathcal{D}$ such that $C \in \PL \subset \Sigma$.
    To prove this, we consider a generic line $\PL \in \mathcal{L}$ on the congruence and assume for contradiction that there are two distinct pairs $(C_1,\Sigma_1^\vee), (C_2,\Sigma_2^\vee)\in \mathcal{D}$ such that $C_i \in \PL \subset \Sigma_i$ for $i=1,2$.
    We distinguish two cases.

    First, if $C_1 \neq C_2$, the generic line $\PL$ is a secant line of the curve $\mathcal{C}$ traced by the camera centers. 
    Hence, the congruence $\overline{\mathcal{L}}$ is the family of secant lines of $\mathcal{C}$. 
    In particular, all the lines $\PL'$ in the pencil $C_1 \in \PL' \subset \Sigma_1$ have to be secants of $\mathcal{C}$, which is only possible if $\mathcal{C}$ is a plane curve contained in $\Sigma_1$. 
    Since the same applies to the pencil given by $(C_2,\Sigma_2)$, we see in particular that $\Sigma_1 = \Sigma_2$ (otherwise, the curve $\mathcal{C}$ would be the unique line in their intersection, but then its secants would not form a line congruence).
    However, the congruence $\overline{\mathcal{L}}$ is now simply the family of lines contained in the plane $\Sigma_1$, which has order zero; a contradiction to the assumptions in Theorem \ref{thm:birationalInGeneral}.

    Second, if $C_1=C_2$, then we have necessarily that $\Sigma_1 \neq \Sigma_2$. 
    This case is projectively dual to the first case. 
    In particular, the rolling planes $\Sigma(r)$ trace a curve in the plane $C_1^\vee \subset (\PP^3)^*$ and the dual congruence $\overline{\mathcal{L}}^\vee \subset \mathrm{Gr}(1, (\PP^3)^*)$ is the family of secant lines of that plane curve. 
    Hence, $\overline{\mathcal{L}}^\vee$ consists of all lines in the plane $C_1^\vee$, and $\overline{\mathcal{L}}$ consists of all lines passing through the center $C_1$. 
    However, as observed at the beginning of the proof of Theorem \ref{thm:order1congruence}, the latter is not possible for a moving camera. \myqed
\end{proof}

We obtain the following corollary for order one cameras.

\begin{corollary}\label{cor:birationalOrder1congruence}
Consider a moving rolling shutter camera whose associated congruence $\overline{\mathcal{L}}$ has order one.
The map $\Lambda$ is birational  if and only if 
the rolling planes map $\Sigma^\vee$ is birational.
\end{corollary}

\begin{proof}
    Clearly, if the map $\Sigma^\vee$ is birational, then the map $r \mapsto (C(r), \Sigma(r)^\vee)$ is birational, and we can apply Theorem \ref{thm:birationalInGeneral} to see that $\Lambda$ is birational.
    For the converse direction, the same statement implies that it is enough to show that the birationality of  $r \mapsto (C(r), \Sigma(r)^\vee)$ implies that $\Sigma^\vee$ is birational. 
    To prove this, we consider the two cases of order-one congruences described in Theorem \ref{thm:order1congruence} separatly.
    In case 2), any rolling plane $\Sigma(r)$ determines the corresponding camera center $C(r)$, and then the birationality of $r \mapsto (C(r), \Sigma(r)^\vee)$ ensures that there is a unique parameter $r$. 

    In case 1), a generic rolling plane $\Sigma(r)$ intersects the curve $\mathcal{C}$ is $\deg \mathcal{C}$ many points (counted with multiplicity). 
    All except one of those points lie on the line $K$.
    The remaining point is the camera center $C(r)$, which is thus uniquely determined by $\Sigma(r)$. As above, the birationality of $r \mapsto (C(r), \Sigma(r)^\vee)$ ensures that the parameter $r$ is unique.
    As a side note, a generic camera center $C(r)$ determines uniquely the corresponding rolling plane: $\Sigma(r) = C(r) \vee K$. 
    Therefore, the birationality of the rolling planes map $\Sigma^\vee$ is equivalent to the birationality of the center map $C$. 
    This also proves Remark \ref{rem:typeISigmaDeterminesC}. \myqed
\end{proof}

\begin{proof}[Proof of Theorem \ref{thm:rs1cameras}]
    If the camera does not move, then $\mathcal{C}$ is a point. 
    The associated congruence consists of all lines passing through that point and has order one. 
    Hence, the camera has order one if and only if the map $\Lambda$ is birational.
    By Proposition \ref{prop:birationalNoMovement}, the latter is equivalent to the conditions in Theorem \ref{thm:rs1cameras}. 
    The point $\mathcal{C}$ has to lie on the line $K$ since each rolling plane has to contain both $\mathcal{C}$ and $K$ and there is a one-dimensional family of such planes. 

    If the camera moves, then $\mathcal{C}$ is a rational curve. 
    The camera has order one if and only if the conditions in Theorem \ref{thm:order1congruence} and  Corollary \ref{cor:birationalOrder1congruence} are satisfied. 
    Note that the birationality of the map $\Sigma^\vee$ implies that every rolling plane  $\Sigma(r)$ uniquely determines the parameter $r$ and thus also the corresponding camera center $C(r)$.
    Hence, the birationality of $\Sigma^\vee$ ensures that the cases 1) and 2) in Theorem \ref{thm:order1congruence} are equivalent to the types I and II in Theorem \ref{thm:rs1cameras}. \myqed
\end{proof}

\subsection{Camera Spaces}

In this section, we prove Propositions \ref{prop:typeIparameters} and \ref{prop:typeIIparameters}.

\begin{lemma}\label{lem:curvesWithD-1Secants}
    The dimension of $\mathcal{H}_d$ is $3d+5$.
    For every line, conic, or nondegenerate rational curve $\mathcal{C}$ of degree at most five, there is a line $K$ such that $(\mathcal{C},K) \in \mathcal{H}_{\deg \mathcal{C}}$.
    For $d \geq 6$, the locus of rational degree-$d$ curves $\mathcal{C}$ that admit a line $K$ with $(\mathcal{C},K) \in \mathcal{H}_d$ is a low-dimensional subset of the locus of all rational degree-$d$ curves.
\end{lemma}

\begin{proof}[Proof of Lemma \ref{lem:curvesWithD-1Secants}]
    For curves $\mathcal{C}$ of degree at most two, there are clearly many such lines $K$.
For curves $\mathcal{C}$ of degree at least three, the existence of such a line $K$ requires the curve $\mathcal{C}$ to be \emph{nondegenerate}, i.e., not contained in any plane (otherwise, every secant line would automatically intersect the curve in $\deg \mathcal{C}$ many points).
Every nondegenerate cubic curve has a two-dimensional family of secant lines $K$.
It is a classical fact that every nondegenerate rational quartic curve has a one-dimensional family of trisecant lines $K$ (see e.g. \cite{LEBARZ2005743}) and that every nondegenerate rational quintic curve has a quadrisecant line $K$ (in fact, typically a unique one; see \cite[Remark 2.16]{ellia2001gonality}).
Together with the fact that the locus of all rational degree-$d$ space curves has dimension $4d$,
this discussion proves the first two assertions of Lemma \ref{lem:curvesWithD-1Secants} for $d \leq 5$.

For $d > 5$,  the locus $\mathcal{H}'_d$ of  rational curves in $\PP^3$ of degree $d$ that intersect some line at $d-1$ or $d$ many points has dimension $3d+5 < 4d$ \cite[Lem.\ 2.13]{ellia2001gonality}.
Moreover, for the general curve in $\mathcal{H}'_d$ (with $d \geq 5$), there is a \emph{unique} line meeting the curve in $d-1$ points (and not $d$ points) \cite[Thm.\ 2.15 \& Rem.\ 2.16]{ellia2001gonality}.\myqed
\end{proof}

\medskip \noindent
\Cref{lem:curvesWithD-1Secants} implies \Cref{rem:cameraPerCurve}, i.e., that almost any rational curve of degree $d \leq 5$ can be the center locus of a \rs camera of type I, while only special curves are allowed when $d \geq 6$.
Next, we determine which birational maps $\Sigma^\vee: \PP^1 \dashrightarrow K^\vee$ are allowed. 
The following lemma applies to all three types of \rs cameras, and uses the notation introduced in the paragraph before Proposition \ref{prop:typeIparameters}.

\begin{lemma}\label{lem:sigma-restrictions}
Let $C:\PP^1\dashrightarrow\PP^3$ be a rational map defined over $\mathbb{R}$ and let $\mathcal{C}$ be the complex Zariski closure of its image.
Let $K\subseteq \PP^3$ be a line and let $\Sigma^\vee:\PP^1\dashrightarrow K^\vee$ be a birational map, both defined over $\mathbb{R}$, such that, for all $(v:t) \in \PP^1$, the plane $\Sigma(v:t)$ contains the point $C(v:t)$. Assume that $\mathcal{C}\not\subseteq H^\infty$ and that $\mathcal{C}$ and $K$ are related as in Type I, II, or III of Theorem \ref{thm:rs1cameras}.

Then, the following are equivalent:
\begin{enumerate}
\item[(a)] There exists an \rs camera of type I, II, or III with distinguished line $K$, center-movement map $C$, and rolling planes map $\Sigma^\vee$.
\item[(b)] $K^\infty$ is a point and on its dual line $(K^\infty)^\vee\subseteq (H^\infty)^*$, there are points $A,B\in(K^\infty)^\vee$ such that $A\cdot B = 0$, $A\cdot A = B\cdot B$, and the map $\Sigma^\vee_\infty:\PP^1\dashrightarrow(H^\infty)^*$ is $\Sigma^\vee_\infty(v:t)=Av+Bt$.
\end{enumerate}
\end{lemma}

\begin{proof}
$(a)\Rightarrow (b)$: We begin by assuming that $C$, $K$, and $\Sigma^\vee$ are part of an \rs camera.
In particular, the birational map $\Sigma^\vee$ is of the form \eqref{eq:SigmaMap}. Thus, $\Sigma_\infty^\vee(x) = (1:0:-x) \cdot R(x)$ for  $x \in \mathbb{R}$. The rotation matrix $R(x)$ preserves norms. So, after fixing a scaling for the rational map $\Sigma_\infty^\vee$, the former equation implies
\begin{align} \label{sigma-infty-rotation}
    \frac{\Sigma_\infty^\vee(x)}{\Vert \Sigma_\infty^\vee(x) \Vert} = \frac{1}{\sqrt{1+x^2}} (1,0,-x) \cdot R(x).
\end{align}
For general $x\in\mathbb R$, the projection matrix $P(x)$ maps the whole plane $\Sigma(x)$ onto the rolling line $\rho(x:1)$. 
In particular, exactly one camera ray $\PL(x)\subseteq \Sigma(x)$ passing through $C(x)$ will be mapped by $P(x)$ to  $(0:1:0)^\top$.
Since $(0:1:0) = \varphi((x:1),(1:0))$, the homogenization of the map $\Lambda$ becomes well-defined at $((x:1),(1:0))$ and we see that  $\PL(x) = \Lambda((x:1),(1:0))$. 
In particular, $\PL$ is a rational map.

Thus, $\omega(x)\coloneqq \PL_{\infty}(x)$ also defines a rational map.
Since $P(x)\omega(x) = (0:1:0)^\top$, after fixing a scaling for the rational map $\omega$, we have
\begin{align}\begin{split}
    \label{rotation-lambda-tilde}
    R(x) \cdot \frac{\tilde \omega(x)^\top}{\Vert \tilde \omega(x) \Vert} = (0,1,0)^\top,\\ \text{ where } \tilde{\omega}(x) := \left( [I_3 \mid -C(x)]\, \omega(x) \right)^\top.
\end{split}
\end{align}
Since $\omega(x)$ is of the form $(\omega_1:\omega_2:\omega_3:0)^\top$, the point $\tilde \omega(x)$ is simply $(\omega_1:\omega_2:\omega_3)\in\PP^2$.
Equations \eqref{sigma-infty-rotation} and \eqref{rotation-lambda-tilde} uniquely determine the rotation matrix $R(x)$.
Indeed,
\footnotesize
\begin{align}\begin{split}\label{rotation-explicit}
    R(x) &= \\ \left[ \begin{matrix}
        \dfrac{1}{\sqrt{1+x^2}} &0&\dfrac{x}{\sqrt{1+x^2}} \\[15pt]
        0&1&0 \\[10pt]
        -\dfrac{x}{\sqrt{1+x^2}}&0&\dfrac{1}{\sqrt{1+x^2}}
     \end{matrix} \right] &{}\cdot{}
    \left[ \begin{matrix}
       \dfrac{\Sigma_\infty^\vee(x)}{\Vert \Sigma_\infty^\vee(x) \Vert}  \\[15pt]  \dfrac{\tilde \omega(x)}{\Vert \tilde \omega(x) \Vert}  \\[15pt]
       \dfrac{\Sigma_\infty^\vee(x) \times \tilde \omega(x)}{\Vert \Sigma_\infty^\vee(x) \Vert \cdot \Vert \tilde \omega(x) \Vert} 
    \end{matrix} \right].
    \end{split}
\end{align}
\normalsize

For $x,y\in\mathbb R$, we may now compute $\Lambda_{\infty}(x,y)$, the intersection of the camera ray parametrized by $(x,y)$ with the plane $H^{\infty}$. The point $\Lambda_{\infty}(x,y)$ is the intersection of the lines represented by
\[
(1:0:-x) \cdot R(x)=\Sigma_\infty^\vee(x) \ \text{ and }\  (0:1:-y) \cdot R(x)
\]
in$(H^\infty)^*$.
We compute
{\footnotesize\begin{align}
 & \Lambda_\infty(x,y) = 
   \Sigma^\vee_\infty(x) \times \biggl(  (0:1:-y) \cdot R(x) \biggr)\nonumber \\
   &= \Sigma^\vee_\infty(x) \times
   \biggl( \frac{\tilde \omega(x)}{\Vert \tilde \omega(x) \Vert}  
   - y\nonumber \\ & {}\cdot{}  \biggl( -\frac{x}{\sqrt{1+x^2}} \cdot \frac{\Sigma^\vee_\infty(x)}{\Vert \Sigma^\vee_\infty(x) \Vert} + \frac{1}{\sqrt{1+x^2}} \cdot \frac{\Sigma^\vee_\infty(x) \times \tilde \omega(x)}{\Vert \Sigma^\vee_\infty(x) \Vert \cdot \Vert \tilde \omega(x) \Vert}
   \biggr) \biggr)\nonumber  \\ \label{lambda-infty}
   &= 
   \frac{\Sigma^\vee_\infty(x) \times \tilde \omega(x)}{\Vert \tilde \omega(x) \Vert} + \frac{y}{\sqrt{1+x^2}} \cdot \frac{\Vert \Sigma^\vee_\infty(x) \Vert}{\Vert \tilde \omega(x) \Vert} \cdot \tilde \omega(x).
\end{align}}
\noindent
The latter expression can be scaled to a rational function since $\Lambda$ is a rational map. 
Since the first term of $\Vert \tilde \omega(x) \Vert \cdot \Lambda_\infty(x,y)$ (which is $\Sigma^\vee_\infty(x) \times \tilde \omega(x)$) is already rational, the function $\Vert \tilde \omega(x) \Vert \cdot \Lambda_\infty(x,y)$ can only be scaled to a rational function if its second term ${y} \cdot \Vert \Sigma^\vee_\infty(x) \Vert \cdot \tilde \omega(x)/{\sqrt{1+x^2}}$ is rational as well. 
This means that ${\Vert\Sigma^\vee_\infty(x) \Vert}/{\sqrt{1+x^2}}$ is rational, which is equivalent to $\sqrt{1+x^2} \cdot \Vert\Sigma^\vee_\infty(x) \Vert$ being rational.
For the affine linear map $\Sigma^\vee_\infty$ this means that $(1+x^2) \cdot (\Sigma^\vee_\infty(x) \cdot \Sigma^\vee_\infty(x)) = Q(x)^2$ for some quadratic polynomial $Q$.
Since $K\subseteq \Sigma(x)$ for all $x$, we can write $\Sigma^\vee_\infty(x) = Ax + B $ for some $A,B \in (K^\infty)^\vee$, and so a direct computation (e.g., with \texttt{Macaulay2}) reveals that the existence of $Q$ is equivalent to the conditions $A \cdot B = 0$ and $A\cdot A = B \cdot B$.

Since the entries of $A$ and $B$ are real numbers, these two conditions imply that $A$ and $B$ are not scalar multiples of each other. Thus, the map $\Sigma^\vee_\infty$ is not constant, which implies that $K$ is not contained in $H^\infty$ (otherwise we would have $\Sigma^\vee_\infty(x) = K^\vee$ for all $x$).

$(b)\Rightarrow (a)$: We are given the data of $K$, $C$, and $\Sigma^\vee$, and need to find an \rs camera that conforms to these. First, the type of our camera (I, II, or III) is readable from $K$ and $\mathcal{C}$. Second, the rotation map $R$ must be determined. Third, the map $\Lambda$ defined in \eqref{eq:LambdaMap} has to be rational.

We define a map $\omega:\PP^1\dashrightarrow H^\infty$ such that the line $\PL(x)$ spanned by $\omega(x)$ and $C(x)$ will be sent to $(0:1:0)$ by the camera.
Fix any line $\ell$ at infinity and set $\omega(x)\coloneqq \Sigma(x)\wedge \ell$.
Equation~\eqref{rotation-explicit} now gives the definition of a rotation map $R$ that conforms to the given data.

Finally, since $\sqrt{1+x^2} \Vert \Sigma^\vee_\infty(x) \Vert = Q(x)$, we see from \eqref{lambda-infty} that $\Lambda_\infty$ is a rational map (after multiplying by $\Vert \tilde \omega(\cdot) \Vert$). Therefore, also $\Lambda(x,y) = C(x) \vee \Lambda_\infty(x,y)$ is rational.\myqed
\end{proof}

\medskip
Lemma~\ref{lem:sigma-restrictions} gives an existence statement for \rs cameras of all types given the data $C$, $K$, $\Sigma^\vee$. 
That data determines uniquely the associated congruence $\overline{\mathcal{L}}$,  
but the camera rotation map $R:\mathbb R\to \SO(3)$ and the congruence parametrization $\Lambda$ are not yet fixed.
In fact, there are many ways of defining a camera rotation map $R:\mathbb R\to \SO(3)$ that fits the data. For $x \in \R$, the orientation $R(x)$ has three degrees of freedom. The first two are accounted for, up to orientation, by the rolling plane $\Sigma(x)$. The third may be fixed, up to orientation, by the choice of which line $\PL(x)\subseteq \Sigma(x)$ passing through $C(x)$ is the camera ray that the projection matrix $P(x)$ maps to $(0:1:0)$, the intersection point of all rolling lines. 
Such a choice is expressed as a rational map $\PL:\mathbb P^1\dashrightarrow \Gr(1,\PP^3)$. 

\begin{lemma}    \label{lem:rs1CamerasAndLambdaMaps}
Let $(C,K,\Sigma^\vee)$ satisfy the conditions from Lemma \ref{lem:sigma-restrictions}.
The \rs cameras involving $C$, $K$, and $\Sigma^\vee$ are in 4-to-1 correspondence with rational maps $\PL: \mathbb{P}^1 \dashrightarrow \GrLine$ satisfying $C(v:t) \in \PL(v:t)\subset \Sigma(v:t)$ for all $(v:t)\in \PP^1$. 
The correspondence comes from the two choices of orientation, that on $\PL$ and that on $\Sigma$.
Changing the orientation on $\PL$ does not change the picture-taking map $\Phi$. Changing the orientation on $\Sigma$ does change $\Phi$, from $(\Phi_1:\Phi_2:\Phi_0)$ to $(\Phi_1:-\Phi_2:\Phi_0)$.
\end{lemma}
\begin{proof}
     Given $(C,K,\Sigma^\vee)$, we are interested in describing all possible orientation maps $R: \RR \to \mathrm{SO}(3)$ that would conform to an \rs camera. As in the proof of \Cref{lem:sigma-restrictions}, for every such camera, there is a rational map $\PL$ that assigns to every $x \in \R$ the camera ray that the projection matrix $P(x)$ maps to $(0:1:0)^\top$ (see paragraph under \eqref{sigma-infty-rotation}).
    Conversely, any such map defines an orientation map $R: \RR \to \mathrm{SO}(3)$ via \eqref{sigma-infty-rotation} and \eqref{rotation-lambda-tilde} (with $\omega := \PL_\infty$), \emph{after} fixing a scaling for the rational maps $\Sigma_\infty^\vee$ and $\omega$.
    That means that $R(x)$ is actually only determined up to the signs of $\Sigma_\infty^\vee(x)$ and $\omega(x)$.
    Changing the sign of $\Sigma^\vee_\infty(x)$ would mean to use the unit normal vector of the rolling plane $\Sigma(x)$ that points in the opposite direction, i.e., a rotation around the ray $\PL(x)$ by $180^\circ$.
    Changing the sign of $\omega(x)$ amounts to a rotation around the normal vector of the plane $\Sigma(x)$ by $180^\circ$. 
  
   Finally, we have to analyze how these sign changes affect the picture-taking map $\Phi$. Recall from \eqref{eq:PhiMap} that $\Phi$ is the composition of several maps. The map $\Gamma$ is not affected by the sign changes. Hence, it suffices to consider how $\Lambda \circ \varphi^{-1}$ is affected. In an affine chart, this map is spelled out in \eqref{lambda-infty}. 
    On the one hand, 
changing the sign of $\Sigma^\vee_\infty(x)$, changes $\Lambda_\infty(x,y)$ in \eqref{lambda-infty} to $-\Lambda_\infty(x,-y)$. 
On the other hand, changing the sign of $\omega(x)$, changes $\Lambda_\infty(x,y)$ in \eqref{lambda-infty} simply to $-\Lambda_\infty(x,y)$. Combining this with the definition of the map $\varphi$ from~\eqref{small-phi} concludes the proof.
    \myqed
\end{proof}


We are now ready to prove \Cref{prop:typeIparameters}.

\begin{proof}[Proof of~\Cref{prop:typeIparameters}]
With the lemmas established above, we can describe the parameter space of \rs cameras of type I, up to the above described choice of orientation.
The first parameter is a line $K$. It is not allowed to lie at infinity due to Lemma \ref{lem:sigma-restrictions}.
Second, we can choose any curve $\mathcal{C}$ of degree $d$, not at infinity, such that $(\mathcal{C},K) \in \mathcal{H}_d$.
The third parameter is a map $\Sigma^\vee_\infty$ as in Lemma \ref{lem:sigma-restrictions}. 
The rolling planes map $\Sigma^\vee$ can be read off from the the map $\Sigma^\vee_\infty$ since each rolling plane $\Sigma(v:t)$ is the span of the line $\Sigma_\infty(v:t)$ with $K$.
By Remark \ref{rem:typeISigmaDeterminesC}, the map $\Sigma$ determines uniquely the parametrization $C$ of the curve $\mathcal{C}$.
Finally, we need to specify the map $\PL$.
In Type I, for general $(v:t) \in \PP^1$, the line $\PL(v:t)$ intersects the line $K$ at a point other than $C(v:t)$, giving rise to a rational map $\lambda:\PP^1\dashrightarrow K$. Conversely, every such rational map $\lambda$ gives a camera-ray map $\PL \coloneqq \lambda \vee C$. Thus, the orientation map $R$ is specified, up to the $4:1$ relationship above, by $\lambda$. To summarize, the space of all type-I cameras is
$\mathcal{P}_{I,d,\delta}$ in \Cref{prop:typeIparameters}.
There is a one-dimensional family of maps $\Sigma^\vee_\infty$ of the form described in \Cref{lem:sigma-restrictions},
since choosing $A$ already determines $B$, up to sign. Hence, the camera space has dimension
$
\dim \mathcal{P}_{I,d,\delta} = (3d+5) + 1 +  (2 \delta + 1) = 
3d + 2 \delta + 7.
$ \myqed
\end{proof}

\medskip
Similarly, the parameters for an \rs camera of type II or III are the line $K$ (not at infinity), the center-movement map $C: \PP^1 \dashrightarrow K$ of degree $d \geq 0$, the map $\Sigma^\vee_\infty$ as in Lemma~\ref{lem:sigma-restrictions} which determines the rolling planes map $\Sigma^\vee$, and the map $\PL:\PP^1 \dashrightarrow \GrLine$ satisfying $C(v:t)\in \PL(v:t)$ and $\PL(v:t)\subseteq \Sigma(v:t)$ for all $(v:t) \in \PP^1$.
Once the data $(K,C,\Sigma^\vee)$ is fixed, the following lemma shows that the map $\PL$ is determined by a pair of homogeneous polynomials $(h,p)$ of degrees $(\delta,\delta+d+1)$ for some $\delta \geq 0$.

We work in the following setting: we rotate and translate $K$ until it becomes the $z$-axis.
Then, the center-movement map $C:\PP^1\dashrightarrow \PP^3$ takes the form $(0:0:C_3:C_0)$, where $C_3,C_0$ are homogeneous polynomials in two variables of degree $d\geq 0$. The rolling planes map $\Sigma^\vee:\PP^1\dashrightarrow (\PP^3)^*$ takes the form $(\Sigma_1: \Sigma_2: 0:0)$, where $\Sigma_1,\Sigma_2$ are linear homogeneous polynomials in two variables. 
We represent the map $\PL: \PP^1 \dashrightarrow \GrLine$ by Plücker coordinates $(p_{12}: p_{13}: p_{10}: p_{23}: p_{20}: p_{30})$, where the $p_{ij}$ are homogeneous polynomials in two variables of arbitrary degree.

\begin{lemma}
    \label{lem:cameraRayMapII}
Assume that $K$ is the $z$-axis and consider maps $C: \PP^1 \dashrightarrow K$ and $\Sigma^\vee: \PP^1 \dashrightarrow K^\vee$.
A map $\PL:\PP^1 \dashrightarrow \GrLine$ satisfies $C(v:t)\in \PL(v:t)$ and $\PL(v:t)\subseteq \Sigma(v:t)$ for all $(v:t) \in \PP^1$ if and only its Plücker coordinates are
\begin{align*}
& (p_{12}: p_{13}: p_{10}: p_{23}: p_{20}: p_{30}) \\ {}={}& (0:-h \Sigma_2 C_3: -h \Sigma_2 C_0: h \Sigma_1 C_3: h \Sigma_1 C_0: p_{30}),
\end{align*}
where $h$ is a homogeneous polynomial in two variables.
\end{lemma}
\begin{proof}
The degree-$d$ map $C$ is of the form $(0:0:C_3:C_0)$.
We can assume that $C_3$ and $C_0$ share no common factor, as otherwise we could make $d$ smaller. Since the birational map $\Sigma = (\Sigma_1:\Sigma_2:0:0)$ is not constant, we note the same for $\Sigma_1$ and $\Sigma_2$.

The condition $C(v:t)\in \PL(v:t)$ means that $\PL(v:t)$ is the span of $C(v:t)$ and some other point \[(a_1(v:t):a_2(v:t):a_3(v:t):a_0(v:t)).\]
Hence, its Plücker coordinates are 
\begin{align*}
p_{12} &= 0 & p_{30} &= a_3C_0-a_0C_3\\
p_{13} &= a_1C_3 & p_{20} &= a_2 C_0 \\
p_{10} &= a_1C_0 & p_{23} &= a_2C_3.
\end{align*}
Note that fixing the point $(C_3:C_0)$ restricts the coordinates $p_{13},p_{10},p_{23},p_{20}$, but that $p_{30}$ is arbitrary.
Dually, considering the line $\PL(v:t)$ as the intersection of two planes $\Sigma(v:t)$ and $(b_1(v:t):b_2(v:t):b_3(v:t):b_0(v:t)) \in (\PP^3)^\ast$, we obtain 
\begin{align*}
p_{12} &= 0 & p_{30} &= b_2\Sigma_1-b_1\Sigma_2\\
p_{13} &= -b_0\Sigma_2 & p_{20} &= -b_3\Sigma_1 \\
p_{10} &= b_3\Sigma_2 & p_{23} &= b_0\Sigma_1.
\end{align*}

 Now we show that $\Sigma_1\mid a_2$. 
Let $\Sigma_1^k$ be the highest power of $\Sigma_1$ that divides $C_3$. 
Then $\Sigma_1^k \mid p_{13}$, thus $\Sigma_1^k \mid b_0$, so $\Sigma_1^{k+1}\mid p_{23}$, hence $\Sigma_1\mid a_2$. 
By a similar argument we find that $\Sigma_2\mid a_1$. Write $a_2 = h\Sigma_1$ and $a_1 = h'\Sigma_2$. Then $h\Sigma_1\Sigma_2C_0 = a_2 \Sigma_2 C_0 = - b_3 \Sigma_1 \Sigma_2 = -a_1 \Sigma_1 C_0 = -h' \Sigma_1 \Sigma_2 C_0$. It follows that $h' = -h$ and that the Plücker coordinates of $\PL$ have the required form. \myqed
\end{proof}

\begin{proof}[Proof of \Cref{prop:typeIIparameters}]
We are now ready to describe the parameter space of $\mathrm{RS}_1$ cameras of type II. The proof is analogous to that of \Cref{prop:typeIparameters}. Working in the same set-up as \Cref{lem:cameraRayMapII} we see that once $(K,C,\Sigma^\vee)$ are fixed, the map $\PL$ is determined by the bivariate polynomials $h$ and $p_{30}$ from the Lemma. Let $\delta \geq 0$ be the degree of $h$. Then $p_{30}$ is necessarily of degree $\delta+d+1$. All in all, the parameter space of \rs cameras of types II and III is
$\mathcal{P}_{II,d,\delta}$ as claimed.

Note that \rs cameras of type III are precisely the special case when $d=0$ and the image of the constant map $C$ is not allowed to be at infinity.
The dimension of the camera  space is 
$
    \dim \mathcal{P}_{II,d,\delta} = 4 + (2d+1) + 1 + (\delta+1) + (\delta+d+2) -1 = 3d +  2\delta + 8.
$ \myqed
\end{proof}

\subsection{Constant Rotation}

This section proves Propositions \ref{prop:rs1NoRotation} and \ref{prop:constandSpeedConstantRotation}, and Remark \ref{rem:constantSpeedOrder2}.

\begin{lemma}
	\label{lem:dSecantImpliesPlanar}
	Let $\mathcal{C} \subset \PP^3$ be an irreducible curve of degree $d$.
	If there is a line $K$ that intersects $\mathcal{C}$ in $d$ points (counted with multiplicity), then $\mathcal{C}$ and $K$ are contained in a common plane.
\end{lemma}

\begin{proof}
	Take any point on $\mathcal{C}$ that does not lie on $K$. Then, that point and $K$ span a plane $\Sigma$ that contains at least $d+1$ points of $\mathcal{C}$. 
	Since $d = \deg \mathcal{C}$, the plane $\Sigma$ has to contain the whole curve $\mathcal{C}$. \myqed
\end{proof}

\begin{lemma}
    \label{lem:notAllRollingPlanesTangent}
    Let $\mathcal{C} \subset \PP^3$ be a curve and let $K \in \GrLine$.
    There are only finitely many planes in $\PP^3$ that contain $K$ and are tangent to $\mathcal{C}$ at one of the points in $\mathcal{C} \setminus K$.
\end{lemma}

\begin{proof}
    We may assume without loss of generality that $\mathcal{C}$ is irreducible. 
    Moreover, we assume that $\mathcal{C} \neq K$, as otherwise the assertion is clear.
    We pick a generic point on the line $K$ and denote by $\pi: \PP^3 \dashrightarrow \PP^2$ the projection from that point. 
    Then, $\pi(\mathcal{C})$ is a plane curve, and every plane $\Sigma$ containing $K$ that is tangent to $\mathcal{C} \setminus K$ projects to a line $\pi(\Sigma)$ that contains the point $\pi(K)$ and is tangent to $\pi(\mathcal{C})$ at another point.
    In the dual projective plane $(\PP^2)^\vee$, those lines $\pi(\Sigma)$ correspond to intersection points of $\pi(\mathcal{C})^\vee$ with the line $\pi(K)^\vee$.
    There are only finitely many such intersection points, since otherwise $\pi(\mathcal{C})^\vee = \pi(K)^\vee$, which would imply the contradiction that the curve $\pi(\mathcal{C})$ equals the point $\pi(K)$. \myqed
\end{proof}

\begin{proof}[Proof of Proposition \ref{prop:rs1NoRotation}]
    Consider a rolling shutter camera with constant rotation. Without loss of generality, we assume that $R(v:t) = I_3$.
    Then, the rolling plane $\Sigma^\vee (v:t)$ defined as an element of $(\PP^3)^*$ by
    \[
    \left(  -t C_0(v:t)\,:\, 0 \,:\, v C_0(v:t) \,:\, tC_1(v:t) - vC_3(v:t) \right)
    \]
    is spanned by $(0:1:0:0)$, $(v:0:t:0)$, and $C(v:t)$.
    In particular, all rolling planes meet in a common point.
    By Theorem \ref{thm:rs1cameras} it is enough to show the following:
    If the intersection of the rolling planes is a line $K$,
    then the other conditions in Theorem \ref{thm:rs1cameras} are automatically satisfied.

We start by proving that the map $\Sigma$ is birational in that case.
We observe that $\Sigma: (v:t) \mapsto K \vee (v:0:t:0)$. Thus, we obtain its inverse $\Sigma^{-1}$ by intersecting each rolling plane $\Sigma(v:t)$ with the plane $(0:0:0:1)^\vee$ at infinity. (In fact, the latter is the line $(1:0:0:0) \vee (v:0:t:0)$.)
    
Now it remains to show that the Zariski closure $\mathcal{C}$ of the set of camera centers is one of the three types in Theorem \ref{thm:rs1cameras}.
If $\mathcal{C}$ is a point, it has to lie in each rolling plane and thus on the line $K$, which is type III.
Otherwise, $\mathcal{C}$ is an irreducible, rational curve of degree $d$.
If $\#(K \cap \mathcal{C}) > d$, then $\mathcal{C} = K$ and we are in type II.
If $\#(K \cap \mathcal{C}) = d$, then $\mathcal{C}$ and $K$ lie in a common plane, say $\Sigma'$, by Lemma \ref{lem:dSecantImpliesPlanar}.
But since $\Sigma: (v:t) \mapsto K \vee C(v:t)$, each rolling plane has to be equal to that plane $\Sigma'$, which contradicts that the intersection of the rolling planes is a line.
Hence, the case $\#(K \cap \mathcal{C}) = d$ cannot happen.

We still have to consider the case $\#(K \cap \mathcal{C}) <d$.
If $d=1$, this means that $\mathcal{C}$ is a line that does not meet $K$, which is type I in Theorem \ref{thm:rs1cameras}.
Thus, in the following, we may assume that $d \geq 2$.
Our goal is to show  $\#(K \cap \mathcal{C}) = d-1$, and so we assume for contradiction that $\#(K \cap \mathcal{C}) \leq d-2$.
Then, a general rolling plane $\Sigma' \in K^\vee$ meets the curve $\mathcal{C}$ in at least two points outside of $K$.
Those points are distinct by Lemma \ref{lem:notAllRollingPlanesTangent}.
Hence, there are distinct $(v:t), (v':t') \in \PP^1$ such that $C(v:t), C(v':t') \in (\Sigma' \cap \mathcal{C}) \setminus K$.
This implies that $\Sigma' = K \vee C(v:t) = \Sigma(v:t)$ and $\Sigma' = K \vee C(v':t') = \Sigma(v':t')$, which contradicts the birationality of the map $\Sigma$. \myqed
\end{proof}

\begin{proof}[Proof of~\Cref{prop:constandSpeedConstantRotation} and~\Cref{rem:constantSpeedOrder2}]
   Consider a rolling shutter camera that moves with constant speed along a line $\mathcal{C}$ and does not rotate.
   If the camera is of order one, then we see directly from the spaces $\mathcal{P}^\mathrm{cs}_{I,1}$ and $\mathcal{P}^\mathrm{cs}_{II,1}$ in Section \ref{sec:linearCameras} that the line $\mathcal{C}$ is parallel to the projection plane $\Pi$.
   For the converse direction, we rotate and translate such that we can assume that the constant rotation is $R = I_3$ and that the movement starts at the origin, i.e., the birational map $C: \PP^1 \dashrightarrow \mathcal{C}$ satisfies $C(0:1) = (0:0:0:1)$.
   Then, the assumption that $\mathcal{C}$ is parallel to the projection plane $\Pi$ means that $C(1:0) = \mathcal{C}^\infty = (a:b:0:0)$.
   Thus, $C(v:t) = (av:bv:0:t)$. 
   By \eqref{eq:SigmaMap}, the rolling planes are 
   $\Sigma^\vee(v:t) = (-t:0:v:av)$.
   The intersection of all rolling planes is the line $K$ spanned by the points $(0:1:0:0)$ and $(0:0:-a:1)$.
   Therefore, the camera has order one by \Cref{prop:rs1NoRotation}.
   If $a=0$, the lines $\mathcal{C}$ and $K$ coincide and the \rs camera is of type II.
   Otherwise, the lines $\mathcal{C}$ and $K$ are skew and the camera has type I. This proves \Cref{prop:constandSpeedConstantRotation}.

   For \Cref{rem:constantSpeedOrder2}, we assume again that the constant rotation is $R = I_3$ and that the movement starts at the origin, but this time $C(v:t) = (av:bv:cv:t)$ for $c \neq 0$.
   By \eqref{eq:SigmaMap}, the rolling planes are $\Sigma^\vee(v:t) = (-t^2:0:vt:cv^2-avt)$, so they trace a conic in $(\PP^3)^\ast$.
   Hence, a general point $X \in \PP^3$ is contained in two rolling planes
   $\Sigma(v_1:t_1)$ and $\Sigma(v_2:t_2)$ with $(v_1:t_1) \neq (v_2:t_2)$. 
   Therefore, the point $X$ is observed twice on the image, namely as $P(v_1:t_1) X$ and $P(v_2:t_2)X$. 
   The associated congruence has order two (since the general point $X$ is contained in the two congruence lines $X \vee C(v_1:t_1)$ and $X \vee C(v_2:t_2)$), while the congruence parametrization $\Lambda$ is birational by \Cref{thm:birationalInGeneral}. \myqed
\end{proof}

\subsection{Image of lines}

In this section, we prove the statements in Section \ref{sec:Image-of-a-line}.

\begin{proof}[Proof of~\Cref{prop:highMultiplicityPoint}]
    Let $D$ be the degree of the image curve $\Phi(L)$.
    Since $L$ is general, every rolling plane meets $L$ in exactly one point.
    Hence, a generic rolling line $r$ on the image plane meets the curve $\Phi(L)$ in a unique point outside of $(0:1:0)$.
    The intersection multiplicity at that unique point is one. 
    (Otherwise the point $r^\vee$ would be on the dual curve $\Phi(L)^\vee$.
    The genericity of $r$ would imply that $\Phi(L)^\vee$ is the line $(0:1:0)^\vee$, but then $\Phi(L)$ would not be a curve; a contradiction).

    We can compute the degree $D$ of the curve $\Phi(L)$ by intersecting wih any line different from $\Phi(L)$ when counting intersection multiplicities.
    In particular, intersecting with the generic rolling line $r$ shows that the curve has multiplicity $D-1$ at the point $(0:1:0)$. \myqed
\end{proof}

\begin{proof}[Proof of Theorem~\ref{thm:degreePhi}]
    After rotating and translating, we may assume that the line $K$ is the $y$-axis, i.e., $K = (0:1:0:0) \vee (0:0:0:1)$.
    Then, the birational map $\Sigma^\vee: \PP^1 \dashrightarrow K^\vee$ takes the form $\Sigma\vee = (\Sigma_1:0:\Sigma_3:0)$, where $\Sigma_1$ and $\Sigma_3$ are linear with $\gcd(\Sigma_1,\Sigma_3)=1$.
    We consider the center-movement map $C = (C_1:C_2:C_3:C_0)$ of degree $d$.
    In type I, the coordinate functions $C_1$ and $C_3$ have $d-1$ common roots (counted with multiplicity), i.e., $C_i = \tilde C \cdot \ell_i$ for $i = 1,3$, where $\deg \tilde C = d-1$ and $\deg \ell_i=1$.
    Since $C(v:t) \in \Sigma(v:t)$ for all $(v:t)\in \PP^1$, we have moreover that $(\ell_1:\ell_3) = (-\Sigma_3:\Sigma_1)$.
    All in all, the center-movement map is of the form
    \begin{align}\label{eq:Cmap}
        C = (- \tilde C \cdot \Sigma_3 : C_2 : \tilde C \cdot \Sigma_1: C_0).
    \end{align}
    In type I, we have $\tilde C \neq 0$, 
    while the case $\tilde C = 0$ corresponds to types II and III.

    Now, let us consider a general point $X  = (X_1:X_2:X_3:X_0)\in \PP^3$. 
    Then, $X \vee K = (X_3:0:-X_1:0)^\vee$ is one of the rolling planes and 
    there is precisely one $(v_X:t_X) \in \PP^1$ such that
    $\Sigma(v_X:t_X) = X \vee K$, i.e.,
    \begin{align}\label{eq:computingVT}
        (\Sigma_1(v_X:t_X) : \Sigma_3(v_X:t_X)) = (X_3: -X_1).
    \end{align}
    From this, we see that $(v_X:t_X)$ depends linearly on $(X_1:X_3)$, and not on $X_2$ or $X_0$. 
    In the following, for any map $f(v:t)$ defined on $\PP^1$, we write \linebreak[4] $f_X (X_1:X_3) := f(v_X:t_X)$.
    For instance, $\Sigma_X = X \vee K$ and 
    \eqref{eq:Cmap} becomes 
    \begin{align}
        \label{eq:CXmap}
        C_X = (\tilde{C}_X \cdot X_1 : C_{2,X} : \tilde{C}_X \cdot X_3 : C_{0,X}).
    \end{align}
    The line on the congruence passing through the point $X$ is $\Gamma(X) = C_X \vee X$. 
    To compute the image of $X$ under the map $\Phi$ in \eqref{eq:PhiMap}, we need to find the unique $(u:s) \in \PP^1$ such that $\Lambda((v_X:t_X),(u:s)) = \Gamma(X)$. 
    To do that, it will be enough to consider the points at infinity, i.e., to solve
    $\Lambda_\infty((v_X:t_X),(u:s)) = \Gamma_\infty(X)$.
    From \eqref{eq:CXmap}, we compute
    \begin{align}
        \begin{split}\label{eq:GammaXInf}
        \Gamma_\infty(X) & = (C_X \vee X)^\infty \\
        &= (X_1 \cdot \alpha(X) : \beta(X) : X_3 \cdot \alpha(X)), \\ \text{ where }
        \alpha(X) &:= C_{0,X} - X_0 \tilde{C}_X,  \\ 
        \beta(X) &:= X_2 C_{0,X}  - X_0 C_{2,X}.
        \end{split}
    \end{align}

\noindent
\textbf{Case 1: Non-constant rotation.}
    We begin by computing  $\Lambda_\infty$ in the case of non-constant rotations. 
    We make use of
    \Cref{lem:rs1CamerasAndLambdaMaps}
    and parametrize the \rs cameras involving $C,K,\Sigma^\vee$ using maps $\PL: \PP^1 \dashrightarrow \GrLine$ such that $C(v:t) \in \PL(v:t) \subset \Sigma(v:t)$ for all $(v:t) \in \PP^1$.
    In particular, we have $\PL_\infty(v:t) \in \Sigma_\infty(v:t)$ for all $(v:t) \in \PP^1$, which means that the map $\PL_\infty$ is of the form
    $\PL_\infty = (-\tilde \PL \cdot \Sigma_3 : \PL_2 : \tilde \PL \cdot \Sigma_1)$.
    We will compute $\Lambda_\infty$ by homogenizing \eqref{lambda-infty}.
    For that, we recall from \Cref{lem:sigma-restrictions}
    that $\Sigma^\vee_\infty(v:t) = Av+Bt$ with $A \cdot B = 0$ and $\Vert A \Vert = \Vert B  \Vert$. After scaling, we may assume that the latter norm is $1$.
    Then, $\Vert \Sigma^\vee_\infty(v:t) \Vert = \sqrt{v^2+t^2}$ and 
    we obtain from \eqref{lambda-infty} that
\begin{align}
\begin{split}
 &\Lambda_\infty((v_X:t_X), (u:s))
= (\Sigma^\vee_{\infty,X} \times \PL_{\infty,X}) \cdot s \\
&+ \PL_{\infty,X} \cdot t_X \cdot u
\\&=\left(X_1 \cdot (a(X)s+b(X)u) : c(X)s+d(X)u :\right. \\
&\left.X_3 \cdot (a(X)s+b(X)u)  \right), 
\end{split}
\label{eq:LambdaInftyX}
\end{align}
    where $a(X) := \PL_{2,X} $, 
    $b(X) := t_X \tilde{\PL}_X$, 
    $c(X) := -\tilde{\PL}_X(X_1^2+X_3^2)$, and
    $d(X) := t_X \PL_{2,X}$.
    Hence, comparing \eqref{eq:GammaXInf} and \eqref{eq:LambdaInftyX}, we see that $\Lambda_\infty((v_X:t_X),(u:s)) = \Gamma_\infty(X)$ is equivalent to
    $(\alpha:\beta) = (s:u)\left[\begin{smallmatrix}
        a & c
        \\ b & d
    \end{smallmatrix}\right]$.
    The determinant of the latter $2 \times 2$ matrix is
    $t_X \cdot (\PL_{2,X}^2 + \tilde{\PL}_X^2(X_1^2+X_3^2))$, 
    which is not the zero polynomial since $(X_1^2+X_3^2)$ is not a square.
    Thus, for general $X$, the $2 \times 2$ matrix is invertible and the unique solution $(u:s) \in \PP^1$ is given by
\begin{align*}
    u_X &= a(X) \cdot \beta(X) - c(X) \cdot \alpha(X)
    \\&= \PL_{2,X} \cdot (X_2 C_{0,X}  - X_0 C_{2,X}) \\
    &+ \tilde{\PL}_X(X_1^2+X_3^2) \cdot     (C_{0,X} - X_0 \tilde{C}_X), \\
    s_X &= d(X) \cdot \alpha(X) - b(X) \cdot \beta(X) \\
    &=t_X \left( \PL_{2,X} \cdot (C_{0,X} - X_0 \tilde{C}_X)\right. \\
    &\left.-\tilde{\PL}_X \cdot (X_2 C_{0,X}  - X_0 C_{2,X})
    \right).
\end{align*}
In particular, we see that $t_X$ divides $s_X$ and therefore
$\Phi(X) = (\tfrac{s_X}{t_X} v_X : u_X : s_X)$.
To determine the degree of the map $\Phi$, it is now sufficient to show that $u_X$ is irreducible and to compute its degree. For that, we rewrite $u_X$ as follows:
\begin{align*}
    u_X = f(X) - X_0 \cdot f_0(X) + X_2 \cdot f_2(X), 
\end{align*}
where $f(X) := \tilde{\PL}_X(X_1^2+X_3^2) \cdot C_{0,X}$, 
$f_0(X) := \PL_{2,X} C_{2,X} + \tilde{\PL}_X(X_1^2+X_3^2)  \cdot\tilde{C}_X$, and 
$f_2(X):= \PL_{2,X} C_{0,X}$.
Note that $f, f_0$ and $f_2$ only depend on $(X_1:X_3)$, and not on $X_0$ or $X_2$.
Hence, the only way that $u_X$ can be reducible is when $f$, $f_0$ and $f_2$ have a common factor. We now consider \rs cameras of type I and II separately to show that this is not possible.

\noindent
\textbf{Case 1a: Type I.}
In type I, the map $\PL$ is given via a map $\lambda: \PP^1 \dashrightarrow K$
such that $\PL(v:t) = C(v:t) \vee \lambda(v:t)$ for general  $(v:t) \in \PP^1$.
Writing $\lambda = (0:\lambda_2: 0 : \lambda_0)$, this means that 
\begin{align*}
    \PL_\infty = ( - \lambda_0 \tilde C \cdot\Sigma_3 : \lambda_0 C_2 - \lambda_2 C_0 : \lambda_0 \tilde C \cdot \Sigma_1),
\end{align*}
i.e.,
$\tilde \PL = \lambda_0 \cdot \tilde C$ and $\PL_2 = \lambda_0 C_2 - \lambda_2 C_0$.
Then, $f_0(X) = \lambda_{0,X} (C_{2,X}^2
+  \tilde{C}_X^2 (X_1^2+X_3^2))
 - \lambda_{2,X}C_{0,X}C_{2,X}$.
Hence, for every choice of $C_0$ and for sufficiently general $\lambda_0, \lambda_2, \tilde C, C_2$, we have that 
$\gcd(C_{0,X},f_0)=1$ and $\gcd(\PL_{2,X}, f) = 1$.
Thus, $f$, $f_0$ and $f_2$ cannot have a common factor for a general \rs camera in $\mathcal{P}_{I,d,\delta}$.  
Therefore, $u_X$ is generally irreducible and we conclude that $\deg \Phi = \deg u_X = 2d + \delta + 1$, where 
 $ d = \deg C $ and $\delta = \deg \lambda$.
 The irreducibility argument worked for \emph{every} choice of $C_0$, in particular in the case $d=1$ and $C_0 = t$, which corresponds to the camera center moving on a line with constant speed. Hence, we obtain $\deg \Phi = \delta +3$ for a general \rs camera in $\mathcal{P}^\mathrm{cs}_{I,1,\delta}$.

\noindent
 \textbf{Case 1b: Type II and III.}
 In this case, $\tilde C = 0$. 
 By \Cref{lem:cameraRayMapII}, 
 the Plücker coordinates of the map $\PL$ are
 $(-h \Sigma_3C_2 : 0 : -h \Sigma_3 C_0 : -h \Sigma_1 C_2 : p : h \Sigma_1 C_0)$, for some polynomials $h$ and $p$.
 Hence, $\tilde \PL = h C_0$ and $\PL_2 = p$. 
 Thus, we obtain that $f(X) = h_X C_{0,X}^2 (X_1^2+X_3^2)$ and $f_0(X) = p_X C_{2,X}$, and so 
$\gcd(f,f_0)=1$ and $u_X$ is irreducible
for every choice of $C_0$ and for sufficiently general $h, p,  C_2$.
Therefore, for a general \rs camera in $\mathcal{P}_{II,d,\delta}$, we find that $\deg \Phi = \deg u_X = 2d+\delta+2$, where $d = \deg C$ and $\delta = \deg h$.
Also, a general \rs camera in $\mathcal{P}_{II,1,\delta}^\mathrm{cs}$ (where $C_0 = t$) satisfies $\deg \Phi = \delta + 4$. 

\noindent
\textbf{Case 2: Constant rotation.}
The second row of the constant rotation matrix $R$ has to be $K^\infty = (0:1:0)$.
Hence, we may assume that $R = I_3$.
Then, $\Sigma^\vee = (-t:0:v:0)$ and we see from \eqref{eq:computingVT} that $(v_X:t_X) = (X_1:X_3)$.
We compute $\Lambda_\infty((v_X:t_X),(u:s))$ as the intersection of the lines that are dual to the points $\Sigma^\vee_{\infty,X} = (-X_3:0:X_1)$ and $((0:-s:u) P(v_X:t_X))^\infty = (0:-s:u)$:
\begin{align*}
    \Lambda_\infty((v_X:t_X),(u:s)) = (sX_1 :uX_3 : sX_3).
\end{align*}
Comparing this with \eqref{eq:GammaXInf}, we see that the unique solution $(u:s) \in \PP^1$ to $\Lambda_\infty((v_X:t_X),(u:s)) = \Gamma_\infty(X)$ is
$(u_X:s_X) = (\beta(X):X_3\cdot  \alpha(X))$.
Therefore, $\Phi(X) = (X_1 \cdot \alpha(X) : \beta(X) : X_3 \cdot \alpha(X))$.

We will now show that $\beta(X) = X_2 C_{0,X}  - X_0 C_{2,X}$ is generally irreducible.
For every choice of $C_0$ and sufficiently general $C_2$, we have $\gcd(C_{0,X},C_{2,X}) = 1$.
The latter means that $\beta(X)$ is irreducible since $C_{0,X}$ and $C_{2,X}$ only depend on $(X_1:X_3)$.
This argument does not involve $\tilde C$.
Hence, a general \rs camera in either $\mathcal{P}_{I,d}$ or $\mathcal{P}_{II,d}$ satisfies $\deg \Phi = \deg \beta = d+1$.
Since $C_0$ was arbitrary, this also holds for cameras moving along a line with constant speed. \myqed
\end{proof}

\subsection{Minimal problems for several linear \texorpdfstring{\rs}{RS1}cameras}
This section and the next provide proofs for all claims in Section \ref{sec:Min-Prob}.
The following lemma is an extended version of \Cref{lemma-p-infty-short}.
\begin{lemma}\label{lemma-p-infty}
Let $(R,\mathcal{C},C)\in \mathcal{P}_1^{cs}$ be of type I and let $(a:b:0)$ be such that $\mathcal{C}^\infty = (a:b:0) \cdot R$.
Consider the associated picture-taking map $\Phi\colon\mathbb{P}^3\dashrightarrow\mathbb{P}^2$ and a line $L\subset\mathbb{P}^3$. Then, we have:
\begin{itemize}
	\item either $\overline{\Phi(L)}$ is a conic through the points $(0:1:0)$ and $(a:b:0)$,
	\item or $\overline{\Phi(L)}$ is the line through the points $(0:1:0)$ and $(a:b:0)$.
\end{itemize}
Conversely, for a generic conic $c$ through $(0:1:0)$ and $(a:b:0)$, there exists a one-dimensional family of  lines  $L\subseteq \PP^3$ with $c=\overline{\Phi(L)}$. These lines rule a smooth quadric surface in $\PP^3$.
\end{lemma}
\begin{proof}
Up to a change of coordinates, as in the proof of \cref{prop:constandSpeedConstantRotation}, we may assume that $R = I_3$
and $\mathcal{C}$ is the line parametrised by
\begin{equation*}
C\colon\mathbb{P}^1\dashrightarrow \mathbb{P}^3,\,(v:t)\mapsto (av:bv:0:t)
\end{equation*}
for some real parameters $a,b$ with $(a:b)\neq0$. 
Note that this is indeed  an element of $\mathcal{P}_1^{cs}$ because $\mathcal{C}\neq\mathcal{C}^\infty=(a:b:0:0)=C(1:0)$. By \eqref{eq:SigmaMap}, we obtain $\Sigma\vee(v:t)=(-t:0:v:av)$.
The proof of this lemma proceeds by first observing that in this setting the picture-taking map has a particularly elegant description. More precisely, the projection matrix is given by
\begin{equation*}
P(v:t)=\begin{pmatrix}
t & 0 & 0 & -av\\
0 & t & 0 & -bv\\
0 & 0 & t & 0
\end{pmatrix}.
\end{equation*}
Let $X=(X_1:X_2:X_3:X_0)\in\mathbb{P}^3$. We aim to determine the time at which the camera sees $X$. For this, we need to check which rolling plane contains $X$, i.e., we need to solve
\begin{equation*}
-tX_1+vX_3+avX_0=0.
\end{equation*}
However, we see immediately that $t=v\cdot \frac{X_3+aX_0}{X_1}$ for $X_1\neq0$. Thus, we obtain for the picture-taking map $\Phi\colon\mathbb{P}^3\dashrightarrow\mathbb{P}^2,$
{\small \begin{align*}
&(X_1\!\!:X_2\!:\!X_3\!:\!X_0)\mapsto (v\frac{X_3+aX_0}{X_1}X_1-avX_0:\\
&v\frac{X_3+aX_0}{X_1}X_2-bvX_0:v\frac{X_3+aX_0}{X_1}X_3).
\end{align*}}
We observe that $v$ is a common factor of all entries, i.e. we obtain
{\small \begin{align} \label{eq:PhiLinearRS}\begin{split}
\Phi\colon 
X\mapsto ((X_3+aX_0)X_1-aX_1X_0:\\
(X_3+aX_0)X_2-bX_1X_0:(X_3+aX_0)X_3).
\end{split}
\end{align}}
We now fix two points $\alpha,\beta\in\mathbb{P}^3$ with $\alpha=(\alpha_1:\alpha_2:\alpha_3:\alpha_0)$ and $\beta=(\beta_1:\beta_2:\beta_3:\beta_0)$, and consider the line $L=\alpha\vee\beta$ they span.
The following \textsc{Macaulay2} computation
\begin{verbatim}
QQ[x0,x1,x2,x3,y0,y1,y2,
a0,a1,a2,a3,b0,b1,b2,b3,s,t,a,b]
I=ideal(y1-((x3+a*x0)*x1-a*x1*x0),
y2-((x3+a*x0)*x2-b*x1*x0),
    y0-(x3+a*x0)*x3,x0-(s*a0+t*b0),
    x1-(s*a1+t*b1),x2-(s*a2+t*b2),
    x3-(s*a3+t*b3))
J=eliminate(I,{x0,x1,x2,x3,s,t})
\end{verbatim}
shows that $\Phi(L)$ is cut out by the quadratic polynomial
\begin{align*}
    &b(\alpha_0\beta_3-\alpha_3\beta_0)y_1^2+(\alpha_2\beta_1-\alpha_1\beta_2)y_0^2\\
    &+a(\alpha_3\beta_0-\alpha_0\beta_3)y_1y_2+(a\alpha_0\beta_2-a\alpha_2\beta_0-b\alpha_0\beta_1\\
    &+b\alpha_1\beta_0+\alpha_3\beta_2-\alpha_2\beta_3)y_1y_0+(\alpha_1\beta_3-\alpha_3\beta_1)y_2y_0.
\end{align*}
For the rest of the discussion, we denote the coefficient of $y_iy_j$ by $c_{ij}$.
We first note that $c_{22}=0$, thus any conic that is an image of a line passes through the point $(0:1:0)$.
This has already been observed in \Cref{prop:highMultiplicityPoint}.
Moreover, we observe that $ac_{11}+bc_{12}=0$. Therefore any such conic passes through the point $(a:b:0)$. The following continuation of our previous \textsc{Macaulay2} calculation shows that these are the only two conditions satisfied by a generic image conic:
\begin{verbatim}
M=first entries gens J
f=M_(00)
T=QQ[a0,a1,a2,a3,b0,b1,b2,b3,
a,b][y0,y1,y2]
g=sub(f,T)
N=last coefficients(g,Monomials=>
{y1^2,y2^2,y0^2,y1*y2,y1*y0,y2*y0})
T1=QQ[a0,a1,a2,a3,b0,b1,b2,b3,a,b,
z0,z1,z2,z3,z4,z5]
Nnew=sub(N,T1)
Check=ideal(z0-Nnew_(0,0),z1-Nnew_(1,0),
z2-Nnew_(2,0),z3-Nnew_(3,0),
z4-Nnew_(4,0),z5-Nnew_(5,0))
N1=eliminate(Check,{a0,a1,a2,a3,
b0,b1,b2,b3})
\end{verbatim}
Thus, for a generic conic $c$ through the two points $(0:1:0)$ and $(a:b:0)$, there is a line $L$ such that $\overline{\Phi(L)} = c$. 
Since the space of such conics is three-dimensional and $\dim \GrLine = 4$, there has to be in fact a one-dimensional family of such lines. It can be obtained as follows:
Fix a generic line $L$ with $\overline{\Phi(L)} = c$.
Consider the set of all lines that meet $\mathcal{C}$, $K$, and $L$. These are all camera rays that $\Phi$ contracts to points on the conic $c$. These camera rays rule a smooth quadric surface since the lines $\mathcal{C}$, $K$, and $L$ are generically pairwise skew. The other ruling of that quadric contains $K$, $\mathcal{C}$, and $L$; and therefore, any line $L'$ in that second ruling will also satisfy $\overline{\Phi(L')} = c$. \myqed
\end{proof}

In terms of the joint-camera map, minimality means that $\Phi^{(m)}$ in \eqref{eq:jcm} is dominant and has generically finite fibers. 
A necessary condition for a reconstruction problem to be minimal is for it to be
\emph{balanced}.
\begin{definition}\label{def:balanced}
    The Reconstruction Problem \ref{reconstruction-problem} is called \emph{balanced} if domain and codomain of its associated joint-camera map $\Phi^{(m)}$ in \eqref{eq:jcm} have the same dimensions.  
\end{definition}
To determine the dimensions of the varieties $\mathfrak X$ and $\mathcal Y$, we count the elements of a general arrangement $X$ in $\mathfrak X$ in the following way:
\begin{enumerate}
\item Count all \emph{free points}: in any collection of points, some may be \emph{dependent} on others, defined as being collinear with two other points. Each minimal set of independent points (which we call free points) has the same cardinality, which we denote  by $p_0$.
\item Count all \emph{third collinear points}: whenever more than two points are collinear, we pick three of them (always including the free points if present) and we count the dependent points among the three as third collinear points. Write their amount as $p_3$.

\item Count all \emph{further collinear points}: These are the points that are collinear with at least three already-counted points. Write their amount as $p_4$. 
Note that $p_0 + p_3 + p_4 = p$.
\item Count all \emph{free} lines, i.e., the lines that do not contain any points. Write their amount as $\ell_0$.
\item Count all \emph{lines incident to} exactly \emph{one point}. Write their amount as $\ell_1$.
\item Count all \emph{lines incident to} exactly \emph{two points}. Write their amount as $\ell_2$.
\end{enumerate}

In this way, we counted all elements of $X$ and associated the following combinatorial data to it:

\begin{center}\begin{tabular}{cl}
$p_0$ & free points\\
$p_3$ & third collinear points\\
$p_4$ & further collinear points\\
$\ell_0$ & free lines\\
$\ell_1$ & lines incident to one point\\
$\ell_2$ & lines incident to two points.
\end{tabular}\end{center}

Recall that in addition $p_\infty \in \{ 0,1\}$ indicates whether we know the point $(a:b:0)$ from Lemma \ref{lemma-p-infty}.
We observe the following implication:
\begin{align} \label{eq:pInftyImplications}
    p_\infty = 0 \quad \Rightarrow \quad \ell_0 = \ell_1 = \ell_2 = p_4 = 0.
\end{align}
Indeed, if $\ell_i > 0$, then there is an image conic that determines the point $(a:b:0)$. Also, if $p_4 > 0$, then there are four collinear points in space. Their image points, together with $(0:1:0)$, determine a unique image conic, which gives us again the point $(a:b:0)$. 

Now, we are able to compute the dimension of the source and target spaces of the joint-camera map $\Phi^{(m)}$:
 \begin{lemma} \label{lem:dimensionSpaces}
     We have $\dim \mathcal{P} = 8$, $\dim G = 7$, and 
     \begin{align*}
\dim \mathfrak X &= 3p_0 + p_3 + p_4 + 4\ell_0 + 2\ell_1,\\
\dim \mathcal Y &= p_\infty+2p_0+2p_3+p_4+3\ell_0+2\ell_1+\ell_2.
\end{align*}
 \end{lemma}
\begin{proof}
    For the dimensions of the camera space and the group, see \Cref{sec:linearCameras} and \Cref{rem:groupAction}.
    In $\PP^3$, once a line is fixed by two points, the further points on this line have only one degree of freedom. In $\PP^2$, lines become conics, and these are only fixed after their third image point is counted. This explains the differentiation between $p_3$ and $p_4$. Apart from these considerations, computing $\dim \mathfrak X$ and $\dim \mathcal Y$ is straightforward. \myqed
\end{proof}


\begin{corollary} \label{cor:balanced}
The Reconstruction Problem \ref{reconstruction-problem} is balanced if and only if
\begin{align}\label{big-den}\begin{split}
&3p_0 + p_3 + p_4 + 4\ell_0 + 2\ell_1 - 7 \\
&= m(p_\infty+ 2p_0 + 2p_3+p_4+3\ell_0+2\ell_1+\ell_2-8).
\end{split}
\end{align}
\end{corollary}
\begin{proof}
    This follows immediately from \eqref{eq:jcm} and \Cref{lem:dimensionSpaces}. \myqed
\end{proof}


Now, we classify the minimal problems for at least two cameras. In Section \ref{ssec:singleCam}, we show that there are no minimal problems for a single camera.

\begin{proposition}
There are finitely many balanced problems for {$m > 1$}.
\end{proposition}
\begin{proof}
All variables in~\eqref{big-den} are non-negative integers. These become bounded from above if we were to require both sides of~\eqref{big-den} to be zero or negative. Thus, there are only finitely many sequences $I=(p_0,p_3,p_4,\ell_0,\ell_1,\ell_2,p_\infty)$ where both sides of~\eqref{big-den} can be zero or negative. But a computer search reveals that this cannot happen. Thus, we may assume that the right-hand side of~\eqref{big-den} is positive and write
\begin{equation*}\label{diff-den}
m - 1 = \frac{p_0 - p_3 + \ell_0 - \ell_2 + 1 - p_\infty}{2p_0+2p_3+p_4+3\ell_0+2\ell_1+\ell_2 - 8 + p_\infty}.
\end{equation*}
Since the denominator in the right-hand side is positive,  a necessary condition for $m$ to be an integer is that this denominator is smaller than or equal to the numerator. This condition is equivalent to
\[
p_0 + 3p_3 + p_4 + 2\ell_0 + 2\ell_1 + 2\ell_2 + 2p_{\infty}\leq 9.
\]
Thus, there are only finitely many choices for the combinatorial datum $I = (p_0,p_3,p_4,\ell_0,\ell_1,\ell_2,p_\infty)$.
For each such sequence $I$, there are only finitely many reconstruction problems. \myqed
\end{proof}

A computer search reveals the full list of sequences $I \!=\! (p_0,p_3,p_4,\ell_0,\ell_1,\ell_2,p_\infty)$ satisfying \eqref{big-den} for some $m>1$. 
Here, we notice that some sequences give rise to two balanced problems due to combinatorial ambiguity.
For instance, the two lines of the configuration $I = (3,0,0, 0,1,1, 1)$ could intersect at one of the three points, or not. To account for this, we introduce a further combinatorial datum $\gamma\in\{0,1,\dotsc,\ell_1+\ell_2\}$ that counts the maximal number of lines going through 
any
free point in the configuration. As it turns out, this resolves all ambiguities in our list of balanced problems.

\begin{theorem} \label{thm:minimalProblems}
Let $m>1$.
The Reconstruction Problem \ref{reconstruction-problem} is balanced if and only if its combinatorial signature $I = (p_0, p_3, p_4, \ell_0, \ell_1, \ell_2, \gamma, p_\infty)$ 
is contained in the following list:
\footnotesize
\[
\begin{tabular}{CCC}
m = 2 &	m = 2 &	m = 2 \\
\textcolor{gray}{(1, 0, 0, 0, 3, 0, 3, 1)} &	(3, 0, 0, 0, 1, 1, 1, 1) &	(4, 1, 0, 0, 0, 0, 0, 1)\\	
(1, 0, 0, 1, 2, 0, 2, 1) &	(3, 0, 0, 0, 1, 1, 2, 1) &	(5, 0, 0, 0, 0, 1, 1, 1)\\	
(1, 0, 0, 2, 1, 0, 1, 1) &	(3, 0, 0, 0, 2, 0, 1, 1) &	(5, 0, 0, 0, 1, 0, 1, 1)\\
(1, 0, 0, 3, 0, 0, 0, 1) &	(3, 0, 0, 0, 2, 0, 2, 1) &	(5, 0, 0, 1, 0, 0, 0, 1)\\
(2, 1, 0, 0, 1, 0, 1, 1) &	(3, 0, 0, 1, 0, 1, 1, 1) &	(7, 0, 0, 0, 0, 0, 0, 1)\\
(2, 1, 0, 1, 0, 0, 0, 1) &	(3, 0, 0, 1, 1, 0, 1, 1) &	(3, 2, 0, 0, 0, 0, 0, 0)\\
\textcolor{gray}{(2, 1, 2, 0, 0, 0, 0, 1)} &	(3, 0, 0, 2, 0, 0, 0, 1) &	(6, 1, 0, 0, 0, 0, 0, 0)\\
(3, 0, 0, 0, 0, 2, 2, 1) &	\textcolor{gray}{(3, 1, 1, 0, 0, 0, 0, 1) }&  (9, 0, 0, 0, 0, 0, 0, 0)\\\\
m = 3 &	m = 4 & m = 5 \\
(2, 0, 0, 0, 2, 0, 1, 1) &	(1, 0, 0, 2, 0, 0, 0, 1) & (4, 0, 0, 0, 0, 0, 0, 1)\\
(2, 0, 0, 0, 2, 0, 2, 1) &	(3, 0, 0, 0, 1, 0, 1, 1)\\
(2, 0, 0, 1, 0, 1, 1, 1) &	(5, 0, 0, 0, 0, 0, 0, 0)\\
(3, 0, 0, 1, 0, 0, 0, 1)\\
(3, 1, 0, 0, 0, 0, 0, 1)\\
(4, 1, 0, 0, 0, 0, 0, 0)\\

\end{tabular}
\]
\normalsize
The three problems printed in gray, that is, $(1,0,0,0,3,0,3,1)$, $(2,1,2,0,0,0,0,1)$, and $(3,1,1,0,0,0,0,1)$, are non-minimal. All other $31$ problems are minimal.
\end{theorem}
\begin{proof}
This is the list of all balanced problems, obtained from augmenting $I$ by the additional combinatorial datum $\gamma$ as described. To show that the problems printed in black are minimal, we carefully construct the joint camera-map map $\Phi^{(m)}$ for each balanced problem in the computer algebra system \texttt{Macaulay2}~\cite{Macaulay2} and test that its Jacobian at a randomly chosen point is invertible. 
This shows that the derivative of $\Phi^{(m)}$ is surjective almost everywhere, and thus that $\Phi^{(m)}$ is dominant. 
The balancedness condition ensures that the fibers of such a dominant $\Phi^{(m)}$ are generically finite.
The non-minimality of the three problems printed in gray is shown in \Cref{lem:lineHomography,lem:hedhehog} below. \myqed
\end{proof}

\begin{lemma}
    \label{lem:lineHomography}
    There is no minimal problem with $m>1$ and $p_4>0$ (i.e., four or more collinear points in space). 
\end{lemma}
\begin{proof}
    We show that the joint-camera map $\Phi^{(m)}$ in \eqref{eq:jcm} cannot be dominant whenever $m>1$ and $\mathfrak X$ contains four collinear points. 
    We assume for contradiction that there exists such a dominant map. 
    Hence, for generic $Y_1, \ldots, Y_m \in \mathcal{Y}$, there are $m$ linear \rs cameras and an arrangement $X$ in $\mathfrak X$ such that $Y_i$ is the picture of $X$ under the $i$-th camera.
    The arrangement $X$ contains four collinear points. 
    We denote their image points on the $i$-th image $Y_i$ by $x_{i,1}, x_{i,2}, x_{i,3}, x_{i,4}$.
    Together with the point $(0:1:0)$, they form a conic $c_i$ in the $i$-th image plane. 
    That conic is the image of the line $L \subset \PP^3$ that is spanned by the four collinear points in space. 
    Restricting the $i$-th camera map $\Phi_i$ (see \eqref{eq:PhiMap}) to the line $L$, yields a birational map from $L$ to the image conic~$c_i$.
    \\ \indent
    Now we consider the birational projection $\pi_i$ of the conic $c_i$ away from the point $(0:1:0)$ onto the line $l_i$ that is spanned by $(1:0:0)$ and $(0:0:1)$. 
    We denote the image of the point $x_{i,j}$ under this projection by $y_{i,j} \in l_i$.
    The composed map $\pi_i \circ \Phi_i |_L$ is an isomorphism from the line $L$ onto the line $l_i$.
    Therefore, $\pi_2 \circ \Phi_2 |_L \circ (\pi_1 \circ \Phi_1 |_L)^{-1}$ is an isomorphism between the lines $l_1$ and $l_2$ that maps the four points $y_{1,1}, \ldots, y_{1,4}$ onto the four points $y_{2,1}, \ldots, y_{2,4}$.
    Hence, the cross-ratio of the four points $y_{1,1}, \ldots, y_{1,4}$ has to equal the cross-ratio of  $y_{2,1}, \ldots, y_{2,4}$.
    This contradicts that the pictures $Y_1$ and $Y_2$ were chosen generically. \myqed
\end{proof}
\begin{lemma}
    \label{lem:hedhehog}
    There is no minimal problem with $m>1$, a single point, and without free lines.
\end{lemma}
\begin{proof}
    An arrangement $X$ in $\mathfrak X$ consists of a single point and possibly some lines passing through that point. 
    For such an arrangement $X$ and an arbitrary $m$-tuple of linear \rs cameras, we show that~-- even after modding out the group $G$~-- there are infinitely many $m$-tuples of cameras that yield the same $m$ pictures.
    This implies that the generic fiber of $\Phi^{(m)}$ cannot be finite, and so the reconstruction problem cannot be minimal.
    \\ \indent
    We begin by modding out the group $G$. To fix the translation, we assume that the single point in $\PP^3$ is at the origin in $\RR^3$, i.e., $(0:0:0:1)$.
    To fix the global scaling, we assume that the Euclidean distance in $\RR^3$ between the origin and the starting position $C_2(0)$ of the second camera is one.
    Finally, we fix the rotation of the first camera to be $R_1 = I_3$.
    \\ \indent
    In particular, we have not assumed anything about the linear constant-speed movement $C_1: \RR^1 \to \RR^3$ of the first camera. 
    Since $X$ consists of a single point with some adjacent lines, at any given time $x \in \RR^1$, the arrangement $X$ looks the same from any camera center on the ray starting from the origin and passing through $C_1(x)$. 
    In other words, when positioned at $C_1(x)$, one can move back and forth along the line spanned by $C_1(x)$ and the origin without changing the picture of $X$ at time $x$.
    Therefore, we can modify the first camera by multiplying the map $C_1$ with any non-zero scalar without changing the image of $X$ under the first camera. \myqed
\end{proof}

All minimal problems listed in Theorem \ref{thm:minimalProblems} are depicted in Figure \ref{fig:minmtwo}.
To conclude the proof of Theorem \ref{thm:minimalMain}, we need to exclude the possibility of minimal problems for a single camera. We will explain this in Section \ref{ssec:singleCam}. 
Beforehand, we explain our degree computations reported  in Figure \ref{fig:minmtwo}.

\begin{remark}\label{deg-comp}
For each of the finitely many minimal problems, we used the numerical algebraic geometry package \texttt{HomotopyContinuation.jl}~\cite{breiding2018homotopycontinuationjl} to compute their respective degrees, i.e. the number of solutions of a general instance of a reconstruction problem with that combinatorial type. More precisely, we used the monodromy method provided in that package. This method can strictly speaking only provide lower bounds for the degree, but multiple runs can be tried and if the results stays the same across runs, this indicates that the number computed is the actual degree.
In our computations, small degrees ($<$2000) remained constant as described above throughout six monodromy runs, while larger degrees tended to differ run-by-run. In any case, we report in~\Cref{fig:minmtwo} the largest computed value, indicating the larger degrees with an asterisk. A plus indicates that the computation had to be interrupted after running more than 5 days. We used interval certification to guarantee that for each computed value there are indeed at least that many solutions. This holds even for the interrupted computations.
\end{remark}

\subsection{No minimal problem for a single linear \texorpdfstring{\rs}{RS1}camera}
\label{ssec:singleCam}

Our next aim is to prove that no minimal problems arise from a single camera. Throughout this section, we assume that $m=1$. 
Then, by \Cref{cor:balanced}, the Reconstruction Problem \ref{reconstruction-problem} is balanced if and only if


\begin{equation*}
p_0+ \ell_0 +1 =p_\infty+p_3+\ell_2.
\end{equation*}


\begin{lemma} \label{lem:singleCamNoFreeLine}
    For $m=1$, there is no minimal problem with $\ell_0$ positive.
\end{lemma}

\begin{proof}
If $p_\infty=0$, the assertion follows from \eqref{eq:pInftyImplications}.
Otherwise, \Cref{lemma-p-infty} tells us that there exists a one-dimensional family of lines in $\mathbb{P}^3$ mapping to a generic image conic in $\mathbb{P}^2$ via $\Phi$, i.e., a free image conic has infinitely many solutions. \myqed
\end{proof}

Therefore, we may assume that $\ell_0=0$ and obtain

\begin{equation} \label{eq:balancedSingleCamKnowInf}
    p_0+1=p_\infty + p_3+\ell_2
\end{equation}
as numerical condition for our balanced problems.
We observe that this equation only involves free points, third collinear points,  and lines incident to two points. 
If $p_\infty=0$, the remaining features do not appear by~\eqref{eq:pInftyImplications}.
If $p_\infty=1$, they do not affect minimality:
\begin{lemma}
\label{lem-minproblemreduced}
    For any minimal problem with $m=1$ and $p_\infty=1$, one can add or remove further collinear points (corresponding to $p_4$) and lines incident to one point (corresponding to $\ell_1$) without changing minimality.
\end{lemma}

\begin{proof}
    Since every minimal problem is balanced and neither $p_4$ nor $\ell_1$ appear in the balanced equation \eqref{eq:balancedSingleCamKnowInf}, it is sufficient to show that the existence of solutions is not affected by adding or removing further collinear points or lines incident to one point.
    
    The existence of solutions is not affected by removing data, thus this part is clear. In order to show that adding such points and lines does not affect minimality, we consider the following set-up:
    
    Let $\widetilde{\mathfrak{X}}$ be the set of arrangements obtained from arrangements in $\mathfrak{X}$ after forgetting all further collinear points and all lines incident to exactly one point. Similarly, we consider $\widetilde{\mathcal{Y}}$ obtained from $\mathcal{Y}$.
Thus, we obtain forgetful maps
\begin{equation*}
    \alpha\colon \mathcal{Y}\to\widetilde{\mathcal{Y}}\quad\textrm{and}\quad \beta\colon\mathfrak{X}\to\widetilde{\mathfrak{X}}.
\end{equation*}
Analogously to the joint camera map $\Phi^{(1)}$ in \eqref{eq:jcm}, we now obtain a map

\begin{equation}       \tilde{\Phi}^{(1)}\colon(\mathcal{P}\times\widetilde{\mathfrak{X}})/G \dashrightarrow\widetilde{\mathcal{Y}}.
\end{equation}
We need to show that for generic $\tilde{y}\in \widetilde{\cal Y}$ and generic $y\in{\cal Y}$ with $\alpha(y)=\tilde{y}$, any solution $\tilde{S}\in (\mathcal{P}\times \widetilde{\mathfrak X})/G$ with $ \tilde{\Phi}^{(1)}(\tilde{S})=\tilde{y}$ can be uniquely extended to a solution $S \in (\mathcal{P}\times {\mathfrak X})/G$ with $ {\Phi}^{(1)}(S)=y$. 
By the same arguments as in the proof of \Cref{lemma-p-infty}, we see that, for a generic line $L\subset\mathbb{P}^3$ with image conic $c$, the map $\Phi_{|L}\colon L\to c$ is birational. Therefore, when $c$ is the conic passing through three points that are the images of three collinear points in $\PP^3$ and the two known points at infinity from Lemma \ref{lemma-p-infty} and $L$ is the solution of $c$ in $\tilde{S}$, then any further point on $c$ corresponds to a unique point on $L$. Thus, adding further collinear points does not affect minimality.

In order to add a line incident to exactly one point, we consider a generic point $x\in\mathbb{P}^2$ with fixed pre-image $X\in\mathbb{P}^3$ and fix a generic image conic $c$ passing through $x$. We claim there is a unique line $L\subset\Phi^{-1}(c)$ passing through $X$ with $\Phi(L)=c$. Recall from \Cref{lemma-p-infty} that the lines mapping to $c$ rule a smooth quadric surface $Q_c$. Any smooth quadric surface is doubly ruled, i.e., there are two lines passing through each point. The pre-image of any point on $c$ is a line in $\mathbb{P}^3$ that is contracted, i.e, one of the lines in the quadric $Q_c$ passing through $X$ is the pre-image of $x$. The other line through $X$ is not contracted and therefore maps to $c$. This is the desired unique line. \myqed
\end{proof}

Due to \Cref{lem:singleCamNoFreeLine,lem-minproblemreduced}, we may from now on assume that $\ell_0 = \ell_1 = p_4 = 0$.
This means that the generic arrangement in $\mathfrak{X}$ only contains free points, third collinear points, and lines incident to exactly two points.
We encode these incidences in a graph:


\begin{definition}
    Let $X$ be a generic arrangement in ${\mathfrak{X}}$. We consider the points and lines in $X$ in $\mathbb{R}^3$. 
    This point-line arrangement induces a graph $\mathfrak{G} = \mathfrak{G}(p,\ell,\mathcal{I})$ whose vertices are the points and whose edges are the line segments connecting points.
\end{definition}

For a balanced problem with $\ell_0 = \ell_1 = p_4 =0$, this graph has $p_0 + p_3$ many vertices and   its number of edges is
$\ell_2 + 2 p_3 = p_0+p_3+1-p_\infty$, where the latter equality follows from \eqref{eq:balancedSingleCamKnowInf}. Thus, if $p_\infty=1$, the graph  has the same amount of edges as vertices;
if $p_\infty=0$, it has one more edge.
For every connected component $\mathfrak{C}$ of the graph $\mathfrak{G}$, we denote by $\mathfrak{X}_\mathfrak{C}$ the set of arrangements obtained from arrangements in $\mathfrak{X}$ after forgetting all point and lines, except those in the component $\mathfrak{C}$.

\begin{proposition}
    \label{prop:infinityReconstructionsPerComponent}
    Let $\mathfrak{C}$ be a connected component of the graph $\mathfrak{G}$.
    For a generic $X \in \mathfrak{X}_\mathfrak{C}$, there is a one-dimensional family of arrangements $X'$ in $\mathfrak{X}_\mathfrak{C}$ such that $\Phi(X) = \Phi(X')$.
\end{proposition}

The statement above is the key insight for why there are no minimal problems for a single camera.
Before we can prove the proposition, we need the following helpful tool.
Recall that  the congruence associated with a linear \rs camera consists of all lines meeting both $\mathcal{C}$ and $K$.
Therefore, for a generic image point $x \in \PP^2$, the line $\overline{{\Phi}^{-1}(x)}$ meets both lines $\mathcal{C}$ and $K$.

\begin{lemma}
\label{lem-graphiso}
    Let $c$ be the image conic of a generic line in $\PP^3$ under $\Phi$.
    Moreover, let $x,x'\in c$ be generic points, and $L:=\overline{{\Phi}^{-1}(x)}, L':=\overline{{\Phi}^{-1}(x')}$. Then, there is an isomorphism $\gamma\colon L\to L'$ such that $\overline{\Phi(X \vee \gamma(X))} = c$ for almost all $X \in L$.
    Furthermore, we have that 
    $\gamma(L \cap \mathcal{C}) = L' \cap \mathcal{C}$ and 
    $\gamma(L \cap K) = L' \cap K$.
\end{lemma}


\begin{proof}
    Recall that the lines that $\Phi$ maps onto the conic $c$ rule a smooth quadric surface $Q_c$, which is therefore doubly ruled. 
    A generic point $X\in L$ is met by two lines on the quadric $Q_c$.
    One of the lines is $L$ itself.
    Let $\tilde{L}$ be the other line.
    Then $\tilde L$ intersects $L'$ in a single point $X'$.
    We define $\gamma(X):=X'$.
    This is clearly a well-defined isomorphism from $L$ to $L'$.
    For those $X$ outside the base locus of $\Phi$, the line $\tilde L$ is mapped by $\Phi$ onto the conic $c$. This proves the first assertion.

    For the second assertion, we start by recalling that the quadric $Q_c$ contains both lines $\mathcal{C}$ and $K$. Note that both $\mathcal{C}$ and $K$ are in $\Phi$'s base locus.
    If $X \in L$ is the point of intersection between $L$ and $\mathcal{C}$, then the two lines on the quadric $Q_c$ that pass through  $X$ are $L$ and $\mathcal{C}$.
    By our definition of $\gamma$, we have that $\gamma(X)$ is the point of intersection between $L'$ and $\mathcal{C}$. The same reasoning applies to the point of intersection between $L$ and $K$. \myqed
\end{proof}

\begin{proof}[Proposition~\ref{prop:infinityReconstructionsPerComponent}]
    Every vertex $V$ in the graph $\mathfrak{C}$ corresponds to a point $V$ in the arrangement $X$, and thus gives rise to a line $L_V := \overline{\Phi^{-1}(\Phi(V))}$.
    Every edge $E$ between two vertices $V_1$ and $V_2$ corresponds to an image conic $\overline{\Phi(V_1 \vee V_2)}$, and thus gives rise to an isomorphism $\gamma_E: L_{V_1} \to L_{V_2}$ as in \Cref{lem-graphiso}. By construction of $\gamma_E$ as seen in the proof of \Cref{lem-graphiso}, we see that $\gamma_E$ sends $V_1$ to $V_2$.

    We start by assuming that the component $\mathfrak{C}$ is a spanning tree.
    In that case, we can start at any vertex $V_1$ and pick a generic point $V_1'$ on the line $L_{V_1}$.
    For any neighboring vertex $V_2$ connected to $V_1$ by the edge $E$, we define $V_2' := \gamma_E(V_1')$. By  \Cref{lem-graphiso}, the lines $V_1 \vee V_2$ and $V_1' \vee V_2'$ have the same image under $\Phi$.
    Continuing like this throughout the whole component $\mathfrak{C}$, we construct an arrangement $X'$ with $\Phi(X) = \Phi(X')$.
    By varying $V_1'$ on the line $L_{V_1}$, we see that this family of arrangements $X'$ is one-dimensional.

    If $\mathfrak{C}$ contains a cycle, then we consider any vertex $V_1$ on the cycle and, by composing all the edge maps $\gamma_E$ along the cycle, we obtain an automorphism $\gamma_{V_1}: L_{V_1} \to L_{V_1}$. 
    Since $L_{V_1}$ is a projective line, every such automorphism has either exactly two fixpoints or is the identity.
    We already know from Lemma \ref{lem-graphiso} that $\gamma_{V_1}$ has the two fixpoints $L_{V_1} \cap \mathcal{C}$ and $L_{V_1} \cap K$.
    Since $V_1$ is a point in the original arrangement $X$, the point $V_1$ must also be a fixpoint of $\gamma_{V_1}$, which shows that $\gamma_{V_1}$ must be the identity. 
    Since this holds for every cycle in the connected graph $\mathfrak{C}$, we can~-- as in the case of spanning trees~-- start at any vertex $V_1$ and with any point $V_1'$ on the line $L_{V_1}$ and then use the graph maps $\gamma_E$ to define an arrangement $X'$
    with the same image under $\Phi$ as the original arrangement~$X$.\myqed
\end{proof}

\begin{corollary}
   For $m=1$, the Reconstruction Problem \ref{reconstruction-problem} is never minimal. 
\end{corollary}
\begin{proof}
    We can understand the domain $(\mathcal{P} \times \mathfrak{X})/G$ of the joint-camera map $\Phi^{(1)}$ by modding out the group $G$ from the camera space $\mathcal{P}$. 
    Indeed, modulo the group $G$, every equivalence class of cameras in \eqref{eq:paramSpaceLinearRScam} is represented by a tuple 
    $(I_3, (0:0:0:1)\vee(a:b:0:0), C: (v:t) \mapsto (av:bv:0:t))$, where $a^2+b^2=1$.
    We consider a generic such camera representative and a generic arrangement $X \in \mathfrak{X}$.
    Let $\Phi$ be the picture-taking map of that camera.
    By \Cref{lem:singleCamNoFreeLine,lem-minproblemreduced}, we may assume that $\ell_0 = \ell_1 = p_4 = 0$.
    Then, by \Cref{prop:infinityReconstructionsPerComponent}, there is an $\eta$-dimensional family of arrangements $X' \in \mathfrak{X}$ with $\Phi(X')=\Phi(X)$, where $\eta$ is the number of  connected components of $\mathfrak{G}$.
    Hence, the dimension of the fiber of $ \Phi(X)$ under the joint-camera map $\Phi^{(1)}$ is at least $\eta > 0$, and so the reconstruction problem cannot be minimal. \myqed
\end{proof}

This concludes the proof of Theorem \ref{thm:minimalMain}.

\subsection{Straight-Cayley cameras}

In this section, we prove the statements from Section \ref{sec:Straight-Cayley}.

\begin{proof}[Proposition~\ref{prop:twistedCubicGeneral}]
To find the curve on which the camera center in \eqref{eq:C-curve} moves, we first substitute $a$ for $u - u_0$ and then eliminate $a$ from \eqref{eq:C-curve}. This is what the following \Macaulay{}~\cite{Macaulay2} computation does: 
\begin{small}
\begin{verbatim}
restart
C2R = c -> matrix {
{c_(0,0)^2-c_(1,0)^2-c_(2,0)^2+1,
2*c_(0,0)*c_(1,0)+2*c_(2,0),
    2*c_(0,0)*c_(2,0)-2*c_(1,0)},
{2*c_(0,0)*c_(1,0)-2*c_(2,0),
-c_(0,0)^2+c_(1,0)^2-c_(2,0)^2+1,
    2*c_(1,0)*c_(2,0)+2*c_(0,0)},
{2*c_(0,0)*c_(2,0)+2*c_(1,0),
2*c_(1,0)*c_(2,0)-2*c_(0,0),
    -c_(0,0)^2-c_(1,0)^2+c_(2,0)^2+1}

Rng = QQ[a,c_1..c_3,o_1..o_3,t_1..t_3,
v_1..v_3]
C = genericMatrix(Rng,c_1,3,1)
O = genericMatrix(Rng,o_1,3,1)
T = genericMatrix(Rng,t_1,3,1)
V = genericMatrix(Rng,v_1,3,1)
Rd = C2R(a*O) 
-- Non-normalized Cayley matrix of rotations
d = 1+a^2*o_1^2+a^2*o_2^2+a^2*o_3^2 
-- The denominator to normalize R
-- Equations for the camera center C 
-- in the world coordinate system
e = flatten entries(d*C + 
transpose(Rd)*(a*V+T)) 
-- Avoid the components for d=0
eMinusD = saturate(ideal(e),ideal(d))
-- Eliminate a
Rne = QQ[flatten entries (vars Rng)_{1..12}]
curve = sub(eliminate({a},eMinusD),Rne)
\end{verbatim}
\end{small}
The last line computes the ideal of the curve $\mathcal{C}$ along which the camera center moves. By plugging in random values for the parameters $o_i, t_j, v_k$, we see that the ideal describes a twisted cubic curve in the variables $c_1,c_2,c_3$.

The projection matrix $P(a)$ and the rolling planes map $\Sigma^\vee(a)$ can be computed as follows:
\begin{verbatim}
P = Rd|(d*(a*V+T))
Sigma = matrix{{1,0,-a}}*P
\end{verbatim}
Again, by plugging in random values for the parameters $o_i, t_j, v_k$, we see that the rolling planes map $\Sigma^\vee$ is of degree four in $a$.
This means that for a generic space point $X$, there are four solutions in $a$ such that $X$ is contained in the rolling plane $\Sigma(a)$.
Each such solution $a$ gives rise to an image point $P(a)\cdot X$. \myqed
\end{proof}

\begin{proof}[Proof of Theorem \ref{thm:twistedCubicOrder1}]
    We continue to work with the \Macaulay{} code from above.
    First, we can check for every choice of parameters that either $\Sigma(0) \cap \Sigma(1)$ or  $\Sigma(0) \cap \Sigma(2)$ is a line (and not a plane):
    \begin{verbatim}
M1small = sub(Sigma, a => 0) 
|| sub(Sigma, a => 1)
I1small = sub(minors(2,M1small), Rne)
M2small = sub(Sigma, a => 0) 
|| sub(Sigma, a => 2)
I2small = sub(minors(2,M2small), Rne)
radical(I1small+I2small) == ideal 1_Rne
    \end{verbatim}
    \vspace{-1em}
    Hence, if all rolling planes are supposed to intersect in a line $K$, then $K$ has to be one of those two lines.
    This means that we have to express the condition that $\Sigma(0) \cap \Sigma(i)$ (for $i \in \{ 1,2 \}$) being a line implies that it is contained in all rolling planes $\Sigma(a)$.
    These conditions are encoded in the following ideals \texttt{I1} and \texttt{I2}:
    \begin{verbatim}
M1 = M1small || Sigma
M2 = M2small || Sigma
I1 = minors(3,M1)
I2 = minors(3,M2)
    \end{verbatim}
    \vspace{-1em}
    Since the conditions are supposed to hold for \emph{all} choices of $a$, when viewing the elements in \texttt{I1+I2} as polynomials in $a$, all of their coefficients have to be zero.
    We collect all those coefficients in an ideal $\texttt{I}$ and decompose it into prime ideals:
\begin{small}    
\begin{verbatim}
Rng2 = Rne[a]
L = flatten entries gens sub(I1+I2, Rng2)
I = sub(ideal flatten apply(L, eq ->
     flatten entries last 
     coefficients eq), Rne)
dec = decompose I
\end{verbatim}
\end{small}
This yields 5 prime ideals:
    \begin{align*} 
        I_1 &:= \langle v_3, o_3, o_2+1 \rangle,\\
        I_2 &:= \langle v_3, o_1, o_2^2+o_3^2+o_2 \rangle, \\
        I_3 &:= \langle v_3,o_3t_1-o_1t_3+o_1v_1,2o_2^2+2o_3^2+3o_2+1, \\ &\quad\quad\quad2o_1^2+o_2+1,2o_2t_1^2+t_1^2-t_3^2+2t_3v_1-v_1^2,\\ &\quad\quad\quad 2o_1o_2t_1+o_1t_1-o_3t_3+o_3v_1 \rangle,\\
        I_4 &:= \langle t_3-v_1,o_3,2o_2+1,4o_1^2+1 \rangle, \\ I_5 &:= \langle v_3,t_3-v_1,t_1,o_3,o_1 \rangle.
    \end{align*}
    The ideals $I_1, I_2, I_5$ appear in \Cref{thm:twistedCubicOrder1}.
    The ideal $I_4$ does not have any real solutions due to its last quadratic generator. Finally, all real solutions of $I_3$ are already contained in the zero sets of $I_1$ and $I_2$.
    We can see the latter claim by using the generator $2 o_1^2 + o_2 + 1 \in I_3$:
    After substituting $o_2 \mapsto - 2 o_1^2 - 1$,
 the generator $2o_2^2+2o_3^2+3o_2+1$ becomes $8 o_1^4+2o_1^2+2o_3^2$, whose only real solutions are $o_1=0=o_3$.
 In that case, $o_2$ becomes $-1$, and all real solutions of $I_3$ are already solutions of $I_1$ and $I_2$. This concludes the proof of the first part of \cref{thm:twistedCubicOrder1}.

 To see that the camera moves along a twisted cubic curve for a generic choice of parameters in each of the three components described by $I_1,I_2,I_5$, it is sufficient to verify this at a single example for each component.
 This is done in \cref{ex:twistedCubicCameras1,ex:twistedCubicCameras2,ex:twistedCubicCameras3}.

When imposing the conditions in ideal $I_1$ on the rolling planes map $\Sigma^\vee$, we see as follows that it becomes linear:
\begin{verbatim}
Sigma1 = sub(Sigma, {v_3 => 0, o_3 => 0,
o_2 => -1})
factor1 = gcd flatten entries Sigma1
1/factor1*Sigma1
\end{verbatim}
This  reveals for parameters in the zero set of $I_1$ that $\Sigma^\vee$ is the birational map 
$a \mapsto (1,0,a,-a t_3 + av_1 + t_1)$.
Since $\Sigma^\vee$ is linear in $a$, such a RS camera has order one.

Similarly, for the ideal $I_2$, the code
\begin{verbatim}
use Rng
Sigma2 = sub(Sigma, {v_3 => 0, 
o_1 => 0})
factor2 = first flatten entries 
sub(Sigma2, Rng/ideal(o_2^2+o_3^2+o_2))
1/factor2*Sigma2
\end{verbatim}
shows that $\Sigma^\vee$, when restricted to parameters in the zero set of $I_2$, is the linear map
$a \mapsto (1, 2 o_3 a, -2 o_2 a - a, -a t_3 + a v_1 + t_1)$, and the RS camera has order one.

For parameters in the zero set of $I_5$, an analogous computation as for $I_1$ shows that $\Sigma^\vee$ becomes the cubic map
\begin{align}\label{eq:SigmaCubic}
    a \mapsto (1-a^2o_2(o_2+2),  \, 0, \,  a (a^2 o_2^2 - 2 o_2 - 1), \, 0).
\end{align}
There are two ways how this map can become linear:
One possibility is $o_2 =0 $, but then the parameters would be contained in the solution set of $I_2$. 
If $o_2 \neq 0$, the map \eqref{eq:SigmaCubic} can only become linear if its quadratic entry divides its cubic entry.
This would mean that $1-a^2o_2(o_2+2)$ and $a^2 o_2^2 - 2 o_2 - 1$ are equal up to scaling by a constant that is allowed to depend on $o_2$ but not on $a$. This constant would have to be $- 2 o_2 - 1$, so we would get
\[
- 2 o_2 - 1  + a^2o_2(o_2+2)(2 o_2 + 1) = a^2 o_2^2 - 2 o_2 - 1
\]
for all $a$, i.e.,
$(o_2+2)(2 o_2 + 1) = o_2$.
The latter is equivalent to $(o_2+1)^2 = 0$, i.e., $o_2 = -1$, 
but in that case the parameters would be already contained in the zero locus of $I_1$.
Hence, for generic parameters in the zero set of $I_5$, the rolling planes map $\Sigma^\vee$ is cubic, and it is linear if and only if the parameters are also contained in the zero locus of either $I_1$ or $I_2$.\myqed
\end{proof}

\begin{proof}[Proof of Proposition~\ref{prop:degreePhiTwistedCubic}]
    For parameters in the zero locus of $I_1$, we can compute $\Phi$ as follows:
    Given a space point $X$, we begin by finding the unique $a$ such that $X$ is contained in the plane $\Sigma(a)$:
    \begin{verbatim}
RX = QQ[gens Rng,x_1..x_3,x_0]
X = matrix{toList(x_1..x_3)|{x_0}}
linX = (sub(Sigma1,RX)*transpose(X))_(0,0)
aX0 = - sub(linX, a => 0)
aX1 = (linX + aX0)/a
aX = aX0 / aX1
    \end{verbatim}
    \vspace{-2em}
    Then, we can plug that parameter $a$ into the projection matrix $P$ and use it to take a picture of $X$:
    \begin{verbatim}
PX = sub(sub(P,RX), {v_3 => 0, o_3 => 0,
o_2 => -1, a => aX})
Phi = (aX1)^3*PX*transpose(X)
    \end{verbatim}
    The resulting map has entries of degree four in $X$.
    It is enough to observe in an example that those entries do not have any common factor; for that, see \cref{ex:twistedCubicCameras1}.
    We can proceed analogously for the ideal $I_2$. \myqed
\end{proof}

\section{Diagram of the relationships between parameter spaces of rolling shutter cameras} \label{sec:paramSpaceOverview}

The following diagram shows the relationships of the parameter spaces of RS cameras we discussed in this paper, as well as the mathematical objects that go into their respective definitions. 

\[\begin{tikzcd}
	{K, \mathcal C, \Sigma_{\infty}^\vee,\lambda} & {\mathcal P_{I, d, \delta}} & {\mathcal P_{I, 1, \delta}^{\mathrm{cs}}} \\
	{R, K,\mathcal C} & {\mathcal P_{I, d}} & {\mathcal P_{I,1}^{\mathrm{cs}}} \\
	{R,K,C} & {\mathcal P_{II, d}} & {\mathcal P_{II, 1}^{\mathrm{cs}}} & {\mathcal P_1^{\mathrm{cs}}} \\
	{K, C,\Sigma_{\infty}^\vee, h,p} & {\mathcal P_{II, d, \delta}} & {\mathcal P_{II, 1, \delta}^{\mathrm{cs}}}
	\arrow[dashed, from=1-1, to=1-2]
	\arrow[from=1-2, to=1-3]
	\arrow[from=1-2, to=2-2]
	\arrow[from=1-3, to=2-3]
	\arrow[dashed, from=2-1, to=2-2]
	\arrow[from=2-2, to=2-3]
	\arrow[hook', from=2-3, to=3-4]
	\arrow[dashed, from=3-1, to=3-2]
	\arrow[from=3-2, to=3-3]
	\arrow[hook, from=3-3, to=3-4]
	\arrow[dashed, from=4-1, to=4-2]
	\arrow[from=4-2, to=3-2]
	\arrow[from=4-2, to=4-3]
	\arrow[from=4-3, to=3-3]
\end{tikzcd}\]

Dashed arrows indicate using building blocks to construct parameter spaces. Normal arrows indicate passing to a special case. Hooked arrows are set inclusions.

\section{Details for Figure \ref{fig:twistedCubicCameras}}\label{sec:details-twistedCubicCameras}
\begin{figure}[ht]\label{fig:twistedCubicCameras1}
        \centering
        \begin{tabular}{ccc}
        \includegraphics[height=3.5cm]{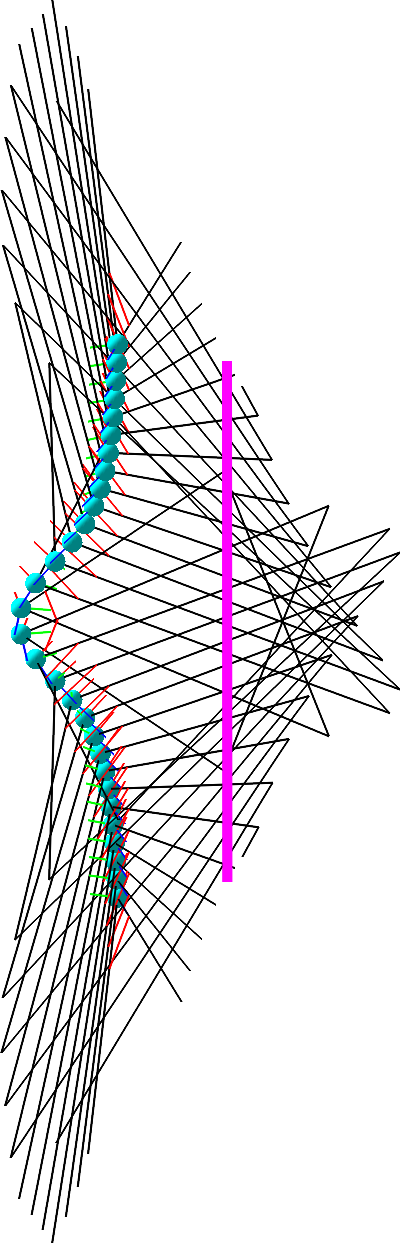}&
        \includegraphics[height=3.5cm]{Figs/simpl-model-figs/21-2.png}&
        \includegraphics[height=3.5cm]{Figs/simpl-model-figs/21-3.png}
        \end{tabular}
\caption{
\Cref{ex:twistedCubicCameras1}. The camera center ${\textcolor{cyan}C}$ (cyan) moves along a twisted cubic $\mathcal{C}$. All rolling planes $\Sigma$ (black) intersect in the line ${\textcolor{magenta} K = (0:X_2:X_0:X_0)}$ (magenta). $\mathcal{C}$ intersects ${\textcolor{magenta} K}$ in the two non-real conjugate points $(0:i:1:1)$ and $(0:-i:1:1)$.
}
\end{figure}

\begin{figure}[ht]
        \centering
        \begin{tabular}{c}
        \includegraphics[height=3.5cm]{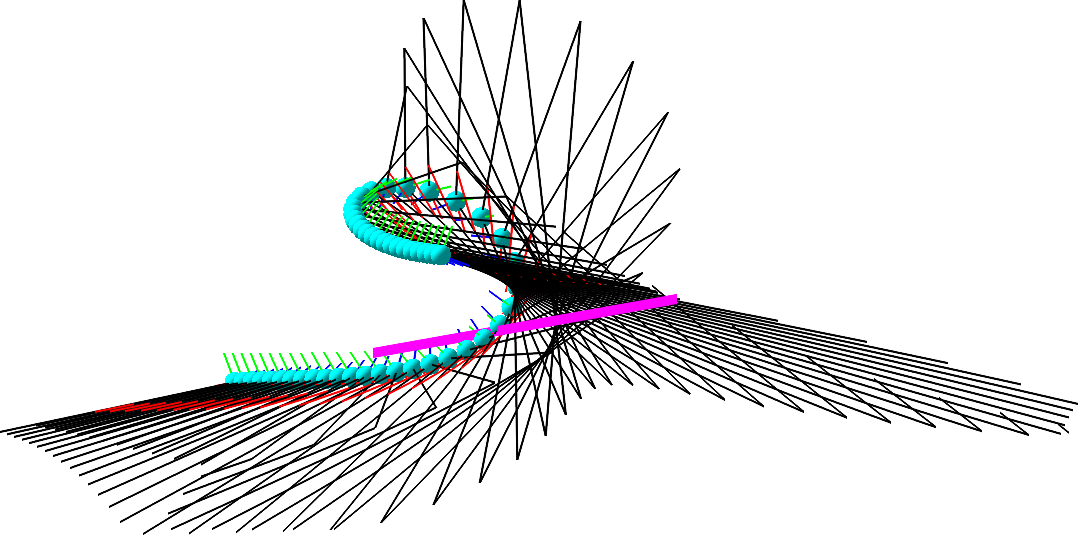}\\
        \includegraphics[height=3.5cm]{Figs/simpl-model-figs/22-2.png}\\
        \includegraphics[height=3.5cm]{Figs/simpl-model-figs/22-3.png}
        \end{tabular}
        \caption{\Cref{ex:twistedCubicCameras2}. The camera center ${\textcolor{cyan}C}$ (cyan) moves along a twisted cubic $\mathcal{C}$. All rolling planes $\Sigma$ (black)  intersect in the line ${\textcolor{magenta}K = (0:X_0:X_3:X_0)}$ (magenta). $\mathcal{C}$ intersects ${\textcolor{magenta} K}$ in $(0:0:1:0)$ and $(0:1:-1:1)$.}
        \label{fig:twistedCubicCameras2}
\end{figure}\ 
\begin{figure}[ht]
        \centering
        \begin{tabular}{c}
        \includegraphics[height=3.5cm]{Figs/simpl-model-figs/23-1.png}\\
        \includegraphics[height=3.5cm]{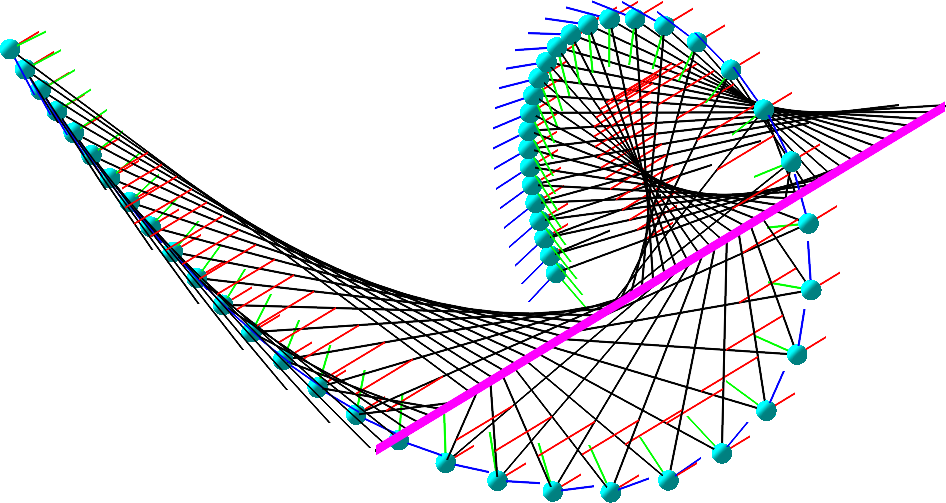}\\
        \includegraphics[height=3.5cm]{Figs/simpl-model-figs/23-3.png}
        \end{tabular}
        \caption{\Cref{ex:twistedCubicCameras3}.  The camera center ${\textcolor{cyan}C}$ (cyan) moves along a twisted cubic $\mathcal{C}$. All rolling planes $\Sigma$ (black)  intersect in the line ${\textcolor{magenta}K = (0:X_2:0:X_0)}$ (magenta). $\mathcal{C}$ does not intersect ${\textcolor{magenta} K}$.}
        \label{fig:twistedCubicCameras3}
\end{figure}

\end{document}